\documentclass[twoside]{article}

\usepackage[accepted]{aistats2022}
\usepackage[T1]{fontenc}
\usepackage[round]{natbib}

\usepackage{algorithm} % 
\usepackage{algorithmic}
\usepackage{amsthm}
\usepackage{amsmath}
\usepackage{amssymb}
\usepackage{nccmath}
\usepackage{autobreak}
\usepackage{mathtools}
\usepackage{xcolor}
\usepackage{graphicx}
\usepackage{caption}
\usepackage{dsfont}
\usepackage{makecell}
\usepackage{xspace}
\usepackage{hyperref}       % hyperlinks
\usepackage{wrapfig}

\usepackage{enumitem}
\usepackage{autobreak}
\newcommand{\OFU}{{{OFU}}\xspace}

\DeclareMathOperator{\Tr}{Tr}
\newcommand{\ttt}{\tilde{\Theta}}
\newcommand{\tth}{\hat{\Theta}}
\newcommand{\tts}{\Theta_*}
\newcommand{\Tw}{T_{w}}
\newcommand{\fil}{\mathcal{F}}
\newcommand{\ta}{\tilde{A}}
\newcommand{\tb}{\tilde{B}}

\newcommand{\tp}{\tilde{P}}

\newcommand{\temp}{\sigma^2_1}
\newcommand{\tempsquare}{\sigma_1}

\newcommand{\tempp}{\sigma^2_2}

\newcommand{\tppp}{\sigma^2_3}
\newcommand{\tpppsquare}{\sigma_3}
\newcommand{\alg}{{{StabL}}\xspace}
\newcommand{\OFULQ}{{{OFULQ}}\xspace}
\newcommand{\RL}{{{RL}}\xspace}
\newcommand{\LQR}{{{LQR}}\xspace}
\newcommand{\OO}{\Tilde{\mathcal{O}}}

\newcommand{\Tbase}{{T_{base}}}

\newcommand{\Prob}{\mathbb P}
\newcommand{\I}{\mathds 1}
\newcommand{\N}{\mathcal N}

\newtheorem{lemma}{Lemma}[section]
\newtheorem{assumption}{Assumption}[section]

\newtheorem{definition}{Definition}[section]
\newtheorem{theorem}{Theorem}[section]
\begin{document}

% If your paper is accepted and the title of your paper is very long,
% the style will print as headings an error message. Use the following
% command to supply a shorter title of your paper so that it can be
% used as headings.
%
\runningtitle{Reinforcement Learning with Fast Stabilization in Linear Dynamical Systems}

% If your paper is accepted and the number of authors is large, the
% style will print as headings an error message. Use the following
% command to supply a shorter version of the authors names so that
% they can be used as headings (for example, use only the surnames)
%
%\runningauthor{Surname 1, Surname 2, Surname 3, ...., Surname n}

\twocolumn[

\aistatstitle{Reinforcement Learning with Fast Stabilization \\ in Linear Dynamical Systems}

\aistatsauthor{ Sahin Lale \And Kamyar Azizzadenesheli \And Babak Hassibi \And Anima Anandkumar }

\aistatsaddress{ Caltech \And Purdue University \And Caltech \And Caltech } ]

\begin{abstract}
  In this work, we study model-based reinforcement learning (\RL) in unknown stabilizable linear dynamical systems. When learning a dynamical system, one needs to stabilize the unknown dynamics in order to avoid system blow-ups. We propose an algorithm that certifies fast stabilization of the underlying system by effectively exploring the environment with an improved exploration strategy. We show that the proposed algorithm attains $\OO(\sqrt{T})$ regret after $T$ time steps of agent-environment interaction. We also show that the regret of the proposed algorithm has only a polynomial dependence in the problem dimensions, which gives an exponential improvement over the prior methods. Our improved exploration method is simple, yet efficient, and it combines a sophisticated exploration policy in \RL with an isotropic exploration strategy to achieve fast stabilization and improved regret. We empirically demonstrate that the proposed algorithm outperforms other popular methods in several adaptive control tasks. 
\end{abstract}

\section{INTRODUCTION}

We study the problem of reinforcement learning (\RL) in linear dynamical systems, in particular in linear quadratic regulators (\LQR). \LQR is the canonical setting for linear dynamical systems with quadratic regulatory costs and observable state evolution. For a known \LQR model, the optimal control policy is given by a stabilizing linear state feedback controller \citep{bertsekas1995dynamic}. When the underlying model is unknown, the learning agent needs to learn the dynamics in order to $(1)$ stabilize the system and $(2)$ find the optimal control policy. This online control task is one of the core challenges in \RL and control theory. 

\textbf{Learning \LQR models from scratch:} \enskip The ultimate goal in online control is to design learning agents that can autonomously adapt to the unknown environment with minimal information and also enjoy finite-time stability and performance guarantees. This problem has sparked a flurry of research interest in the control and \RL communities. However, there are only a few approaches that provide a complete treatment of the problem and strive for learning from scratch with no initial model estimates \citep{abbasi2011lqr,abeille2018improved,chen2020black}. Other than these, the prior works focus either on the problem of finding a stabilizing policy while ignoring the control costs \citep{faradonbeh2018finite}, or on achieving low control costs while assuming access to an initial stabilizing controller \citep{abeille2020efficient,simchowitz2020naive}. 

\textbf{Lack of stabilization and its consequences:} \enskip The existing works \citep{abbasi2011lqr,abeille2017thompson,abeille2018improved} that learn from scratch in \LQR{}s aim to minimize the regret, which is the additional cumulative control cost of an agent compared to the expected cumulative cost of the optimal policy. These algorithms suffer from regret that has an exponential dependence in the \LQR dimensions since they do not assume access to an initial stabilizing policy. They also face system blow-ups due to unstable system dynamics. Besides poor regret performance, the uncontrolled dynamics prevent the deployment of these learning algorithms in practice.

\begin{table*}
\centering
\captionsetup{justification=centering}
\caption{Comparison with the prior works.}
\label{table:1}
 \begin{tabular}{l l l l}

 \textbf{Work} &  \textbf{Regret} &  \textbf{Setting} &   \textbf{Stabilizing Controller}  \\
 \hline 
 \citet{dean2018regret} & $\text{poly}(n,d)T^{2/3}$ & Controllable & Required \\
 
%  \citet{faradonbeh2018input} & $\sqrt{T}~^\star$ & Stabilizable & Requires\\
 \citet{mania2019certainty} & $\text{poly}(n,d)\sqrt{T}$ & Controllable & Required \\
%  \citet{cohen2019learning} & $\text{poly}(n,d)\sqrt{T}$ & Controllable & Required\\
%  \citet{faradonbeh2017optimism} & $\sqrt{T}~^\star$ & Stabilizable & Required \\ 
  \citet{simchowitz2020naive} & $\text{poly}(n,d)\sqrt{T}$ & Stabilizable &Required \\
  
 \citet{abbasi2011lqr} & $(n+d)^{n+d} \sqrt{T}$ & Controllable & Not required \\ 
 
\citet{chen2020black} & $\text{poly}(n,d)\sqrt{T}$ & Controllable & Not required \\

% \textbf{Theorem \ref{reg:sta}$^\dagger$}& $(n+d)^{n+d}\sqrt{T}$ & Stabilizable & Not required \\
 \textbf{This work} & $\text{poly}(n,d)\sqrt{T}$ & Stabilizable & Not required 
\end{tabular}
% \vspace{-1em}
\end{table*}

\textbf{Joint goals of fast stabilization and low regret:} \enskip In this paper, we design an \RL agent for online \LQR{}s that achieves low regret and fast stabilization. To design stabilizing policies without prior knowledge, the agent needs to effectively explore the environment and estimate the system dynamics. However, in order to achieve low regret, the agent should also strategically exploit the gathered knowledge. Thus, the agent requires to balance exploration and exploitation such that it designs stabilizing policies to avoid dire consequences of unstable dynamics and minimize the regret. 

\textbf{Optimism in the face of uncertainty (\OFU) principle:} \enskip One of the most prominent methods to effectively balance exploration and exploitation is the \OFU principle \citep{lai1985asymptotically}. An agent that follows the \OFU principle deploys the optimal policy of the model with the lowest optimal cost within the set of plausible models. This guarantees the asymptotic convergence to the optimal policy for the \LQR \citep{bittanti2006adaptive}. 

\textbf{Failure of \OFU to achieve stabilization:} \enskip Using the \OFU principle, the learning algorithm of \citep{abbasi2011lqr} attains order-optimal $\OO(\sqrt{T})$ regret after $T$ time steps, but the regret upper bound suffers from an \emph{exponential} dependence in the \LQR model dimensions. 
This is due to the fact that the \OFU principle relies heavily on the confidence-set constructions. An agent following the \OFU principle mostly explores parts of state-space with the lowest expected cost and with higher uncertainty. When the agent does not have reliable model estimates, this may cause a lack of exploration in certain parts of the state-space that are important in designing stabilizing policies.
This problem becomes more evident in the early stages of agent-environment interactions due to lack of reliable knowledge about the system. This highlights the need for an improved exploration in the early stages. Note that this issue is unique to control problems and not as common in other \RL settings, e.g. bandits and gameplay. 

\textbf{The restricted \LQR settings in the prior works:} \enskip In designing our learning agent for the online \LQR problem, we consider the stabilizable \LQR setting. Stabilizability is the necessary and sufficient condition to have a well-defined online \LQR problem, \textit{i.e.} it guarantees the existence of a policy that stabilizes the system \citep{kailath2000linear}. In contrast, the prior works that learn from scratch in \LQR{}s only guarantee low regret in the controllable or contractive \LQR settings \citep{abbasi2011lqr,abeille2017thompson,abeille2018improved,chen2020black}, which form a narrow subclass of stabilizable \LQR problems. These conditions significantly simplify the identification and regulation of the unknown dynamics. However, they are violated in many practical systems, e.g., physical systems with non-minimal representation due to complex dynamics \citep{friedland2012control}. In contrast, most of the real-world control systems are stabilizable. 

\subsubsection*{Contributions:} 

Based on the above observations and shortcomings, we propose a novel \textbf{Stab}ilizing \textbf{L}earning algorithm, \alg, for the online \LQR problem and study its performance both theoretically and empirically. 

\textbf{1) } We carefully prescribe an early exploration strategy and a policy update rule in the design of \alg. We show that \alg quickly stabilizes the underlying system, and henceforth certifies the stability of the dynamics with high probability in the stabilizable \LQR{}s. 

\textbf{2) } We show that \alg attains $\OO(\text{poly}(n,d)\sqrt{T})$ regret in the online control of unknown stabilizable \LQR{}s. Here $\OO(\cdot)$ presents the order up to logarithmic terms, $n$ is the state and $d$ is the input dimensions respectively. This makes \alg the first \RL algorithm to achieve order-optimal regret in all stabilizable \LQR{}s without a given initial stabilizing policy. This result completes an important part of the picture in designing autonomous learning agents for the online \LQR problem (See Table \ref{table:1}).

\textbf{3) } We empirically study the performance of \alg in various adaptive control tasks. We show that \alg achieves fast stabilization and consequently enjoys orders of magnitude improvement in regret compared to the existing certainty equivalent and optimism-based learning from scratch methods.  Further, we study the statistics of the control inputs and highlight the effect of strategic exploration in achieving this improved performance. 

The design of \alg is motivated by the importance of stabilizing the unknown dynamics and the need for exploration in the early stages of agent-environment interactions. \alg deploys the \OFU principle to balance exploration vs. exploitation trade-off. Due to lack of reliable estimates in the early stages of learning, an optimistic controller, guided by \OFU, neither provides sufficient exploration required to achieve stabilizing controllers, nor achieves sub-linear regret. Therefore, \alg uses isotropic exploration along with the optimistic controller in the early stages to achieve an improved exploration strategy. This allows \alg to excite all dimensions of the system uniformly as well as the dimensions that have more promising impact on the control performance. By carefully adjusting the early improved exploration, we guarantee that the inputs of \alg are persistently exciting the system under the sub-Gaussian process noise. We show that using this improved exploration quickly results in stabilizing policies with high probability, therefore a much smaller regret in the long term. 

We conduct extensive experiments to verify the theoretical claims about \alg. In particular, we empirically show that the improved exploration strategy of \alg persistently excites the system in the early stages and achieves effective system identification required for stabilization. In contrast, we observe that the optimism-based learning algorithm of \citet{abbasi2011lqr} fails to achieve effective exploration in the early stages and suffers from unstable dynamics and high regret. We also demonstrate that, once \alg obtains reliable model estimates for stabilization, the balanced strategy prescribed by the \OFU principle effectively guides \alg to regret minimizing policies, resulting in a significant improved regret performance in all settings. 

\section{PRELIMINARIES}

\textbf{Notation:} We denote the Euclidean norm of a vector $x$ as $\|x\|$. For a given matrix $A$, $\| A \|$ denotes the spectral norm, $\| A\|_F$ denotes the Frobenius norm, $A^\top$ is the transpose, $\Tr(A)$ gives the trace of matrix $A$ and $\rho(A)$ denotes the spectral radius of $A$, \textit{i.e.} largest absolute value of $A$'s eigenvalues. The maximum and minimum singular values of $A$ are denoted as $\sigma_{\max}(A)$ and $\sigma_{\min}(A)$ respectively.
% is $\sigma_j(A)$, where $\sigma_{\max}(A):=\sigma_1(A) \ldots \geq \sigma_{\min}(A):=\sigma_n(A)$.

Consider a discrete time linear time-invariant system,
\begin{align}
    x_{t+1}& = A_* x_t + B_* u_t + w_t \label{output},
\end{align}
where $x_t \in \mathbb{R}^{n}$ is the state of the system, $u_t \in \mathbb{R}^{d}$ is the control input, $w_{t} \in \mathbb{R}^{n}$ is the process noise at time $t$. We consider the systems with sub-Gaussian noise.

\begin{assumption}[Sub-Gaussian Noise]\label{general_noise}
The process noise $w_t$ is a martingale difference sequence with respect to the filtration $\left(\mathcal{F}_{t-1}\right)$. Moreover, it is component-wise conditionally $\sigma_w^2$-sub-Gaussian and isotropic such that for any $s \in \mathbb{R}$, $\mathbb{E}\left[\exp \left(s w_{t, j}\right) | \mathcal{F}_{t-1}\right] \leq \exp \left(s^{2} \sigma_w^2 / 2\right)$ and $\mathbb{E}\left[w_{t} w_{t}^{\top} | \mathcal{F}_{t-1}\right] = \bar{\sigma}_w^2 I$ for some $\bar{\sigma}_w^2 > 0$.
\end{assumption}

Note that the results of this paper only require the conditional covariance matrix $W = \mathbb{E}[w_tw_t^\top | \mathcal{F}_{t-1}]$ to be full rank. The isotropic noise assumption is chosen to ease the presentation and similar results can be obtained with upper and lower bounds on $W$, \textit{i.e.}, $W_{up} > \sigma_{\max}(W) \geq \sigma_{\min}(W) > W_{low} > 0.$

At each time step $t$, the system is at state $x_t$. After observing $x_t$, the agent applies a control input $u_t$ and the system evolves to $x_{t+1}$ at time $t + 1$. At each time step $t$, the agent pays a cost $c_t = x_t^\top Q x_t + u_t^\top R u_t$, where $Q \in \mathbb{R}^{n \times n}$ and $R \in \mathbb{R}^{d \times d}$ are positive definite matrices such that $\|Q\|, \| R\| < \overline{\alpha}$ and $\sigma_{\min}(Q), \sigma_{\min}(R) > \underline{\alpha}$. The problem is to design control inputs based on past observations in order to minimize the average expected cost $J_*$.
% \begin{equation} \label{opt control}
% J_{\star}=\lim _{T \rightarrow \infty} \min _{u=\left[u_{1}, \ldots, u_{T}\right]} \frac{1}{T} \mathbb{E}\Big[\sum\nolimits_{t=1}^{T} x_{t}^{\top} Q x_{t}+u_{t}^{\top} R u_{t}\Big]
% \end{equation}
This problem is the canonical example for the control of linear dynamical systems and termed as linear quadratic regulator (\LQR). The system (\ref{output}) can be represented as $x_{t+1} = \Theta_*^\top z_t + w_t$, where $\tts^\top = [A_* \enskip B_*]$ and $z_t = [x_t^\top \enskip u_t^\top]^\top$. Knowing $\tts$, the optimal control policy, is a linear state feedback control $u_t = K(\tts)x_t$ with $K(\tts) = -(R+B_*^{\top} P_* B_*)^{-1} B_*^{\top} P_* A_*$, where $P_*$ is the unique solution to the discrete-time algebraic Riccati equation (DARE)~\citep{bertsekas1995dynamic}:
\begin{equation} \label{ARE}
P_* = A_*^\top P_* A_* + Q -  A_*^\top P_* B_* ( R + B_*^\top P_* B_* )^{-1} B_*^\top P_* A_*.
\end{equation}
The optimal cost for $\tts$ is denoted as $J_* = \Tr(\bar{\sigma}_w^2 P_*)$. When the model parameters, $A_*$ and $B_*$, are unknown, the learning agent interacts with the environment to learn these parameters and aims to minimize the cumulative cost $\sum_{t=1}^T c_t$. Note that the cost matrices $Q$ and $R$ are the designer's choice and given. After $T$ time steps, we evaluate the regret of the learning agent as $\text{R}(T) = \sum\nolimits_{t=0}^T (c_t - J_*)$,
% \begin{equation*}
% ,
% \end{equation*}
which is the difference between the performance of the agent and the expected performance of the optimal controller. In this work, unlike the controllable \LQR setting of the prior adaptive control algorithms without a stabilizing controller \citep{abbasi2011lqr,chen2020black}, we study the online \LQR problem in the general setting of \textit{stabilizable} \LQR.

\begin{definition}[Stabilizability vs. Controllability]
\label{def:Stabilizability}
The linear dynamical system $\tts$ is stabilizable if there exists $K$ such that $\rho(A_* + B_*K) < 1$. On the other hand, the linear dynamical system $\tts$ is controllable if the controllability matrix $[B_* \enskip A_*B_* \enskip A_*^2B_* ~\ldots~ A_*^{n-1}B_*]$ has full row rank. 
\end{definition}

% \begin{definition}[Controllability] \label{def:Controllability}
% The linear dynamical system $\tts$ is controllable if the controllability matrix $[B_* \enskip A_*B_* \enskip A_*^2B_* ~\ldots~ A_*^{n-1}B_*]$ has full row rank. 
% \end{definition}

Note that the stabilizability condition is the minimum requirement to define the optimal control problem. It is \textit{strictly weaker than controllability}, \textit{i.e.}, all controllable systems are stabilizable but the converse is not true~\citep{bertsekas1995dynamic}. Similar to \citet{cohen2019learning}, we quantify the stabilizability of $\tts$ for the finite-time analysis.

\begin{definition}[$(\kappa,\gamma)$-Stabilizability]
\label{def:StrongStabilizability}
The linear dynamical system $\tts$ is $(\kappa, \gamma)$-stabilizable for ($\kappa \geq 1$ and $0 < \gamma \leq 1$) if $\|K(\tts)\| \leq \kappa$ and there exists $L$ and $H\succ 0$ such that $A_* + B_*K(\tts) = HLH^{-1}$, with $\|L\| \leq 1-\gamma$ and $\|H\| \|H^{-1} \| \leq \kappa $.
\end{definition}

Note that this is merely a quantification of stabilizability. In other words, any stabilizable system is also $(\kappa,\gamma)$-stabilizable for some $\kappa$ and $\gamma$ and conversely $(\kappa,\gamma)$-stabilizability implies stabilizability (See Appendix \ref{apx:neighborhood}). Thus, we consider $(\kappa,\gamma)$-stabilizable \LQR{}s.

% \begin{assumption}[Sub-Gaussian Noise]\label{general_noise}
% The noise process noise $w_t$ is a martingale difference sequence with respect to the filtration $\left(\mathcal{F}_{t-1}\right)$. Moreover, it is component-wise conditionally $\sigma_w^2$-sub-Gaussian and isotropic such that 
% for any $s \in \mathbb{R}$, $\mathbb{E}\left[\exp \left(s w_{t, j}\right) | \mathcal{F}_{t-1}\right] \leq \exp \left(s^{2} \sigma_w^2 / 2\right)$ and $\mathbb{E}\left[w_{t} w_{t}^{\top} | \mathcal{F}_{t-1}\right] = \bar{\sigma}_w^2 I$ for some $\bar{\sigma}_w^2 > 0$.
% \end{assumption}

% Note that the assumption of having isotropic $w_t$ is only used to provide cleaner analysis and the analysis works without that assumption similar to \cite{abbasi2011lqr}.

\begin{assumption}[Stabilizable Linear Dynamical System]\label{parameterassump_stabilizability}
The unknown parameter $\tts$ is a member of the set $\mathcal{S}$ such that $\mathcal{S} = \big\{ \Theta'=[A', B'] ~\big|~ \Theta' \text{ is } (\kappa, \gamma)\text{-stabilizable, } \|\Theta'\|_F \leq S\big\}$
% \begin{align*}
% \end{align*}
\end{assumption}

Notice that $\mathcal{S}$ denotes the set of all bounded systems that are $(\kappa, \gamma)$-stabilizable, where $\tts$ is an element of, and the membership to $\mathcal{S}$ can be easily verified. Moreover, the proposed algorithm in this work only requires the upper bounds on these relevant control-theoretic quantities $\kappa, \gamma$, and $S$, which are also standard in prior works, e.g. \citep{abbasi2011lqr,cohen2019learning}. In practice, when there is a total lack of knowledge about the system, one can start with conservative upper bounds and adjust these based on the behavior of the system, \textit{e.g.}, the growth of the state. 

From $(\kappa,\gamma)$-stabilizability, we have that $\rho(A'\!+\!B'K(\Theta')) \leq 1-\gamma$, and $\sup \{\|K(\Theta')\| ~|~ \Theta' \in \mathcal{S} \}\leq \kappa $. The following lemma shows that for any $(\kappa,\gamma)$-stabilizable system the solution of \eqref{ARE} is bounded. 
\begin{lemma} [Bounded DARE Solution] \label{bounded_P}
For any $\Theta$ that is $(\kappa,\gamma)$-stabilizable and has bounded regulatory cost matrices, \textit{i.e.}, $\|Q \|, \| R\|< \overline{\alpha}$, the solution of \eqref{ARE}, $P$, is bounded as $\|P\| \leq D \coloneqq \overline{\alpha} \gamma^{-1}\kappa^2(1+\kappa^2)$
\end{lemma}

\section{STABL}

In this section, we present \alg, a sample efficient stabilizing \RL algorithm for the online stabilizable LQR problem. The algorithmic outline is provided in Algorithm \ref{algo_exact}. \alg only requires the minimal information about the stabilizability of the underlying system and \emph{does not} need a stabilizing controller. Therefore, along the ultimate goal of minimizing the regret, \alg puts its primary focus on achieving stabilizing controllers for the unknown system dynamics.

\subsection{Adaptive Control with Improved Exploration} \label{sec:first_phase}

In order to quickly design stabilizing controllers, \alg needs to explore the system dynamics effectively. To this end, \alg  solves $ \min_\Theta \sum\nolimits_{s=0}^{t-1} \|x_{s+1} - \Theta^{\top} z_{s} \|^2 + \lambda \|\Theta\|_F^2$, using the past state-input pairs to estimate the system dynamics as $\hat{\Theta}_t$. Using this estimate, \alg constructs a high probability confidence set $\mathcal{C}_t(\delta)$ that contains the underlying parameter $\Theta_*$ with high probability. In particular, for $\delta \in (0,1)$, at time step $t$, it forms $\mathcal{C}_t(\delta) = \{\Theta : \|\Theta\!-\!\hat{\Theta}_{t}\|_{V_{t}} \leq \beta_{t}(\delta)\}$, for $\beta_t(\delta) =\sigma_w \sqrt{2 n \log (\delta^{-1} \sqrt{\operatorname{det}\left(V_{t}\right)/  \operatorname{det}(\lambda I)} )} + \sqrt{\lambda}S$ and $V_t = \lambda I + \sum_{i=0}^{t-1} z_i z_i^\top$ such that $\tts \in \mathcal{C}_{t}(\delta)$ with probability at least $1-\delta$ for all time steps $t$. Note that this estimation method and the learning guarantee is standard in learning linear dynamical systems since \citet{abbasi2011lqr}. 
% Instead of solving \eqref{leastsquares} from scratch, the model estimate updates can be done via batch or online updates using the standard linear regression techniques.

\begin{algorithm}[t] 
	\caption{\alg}%, Explore and Commit Optimistically}
	\begin{algorithmic}[1]
		\STATE \textbf{Input:}  $\kappa$, $\gamma$, $Q$, $R$, $\sigma_w^2$  $\bar\sigma_w^2$, $V_0 = \lambda I$, $\tth_{0} = 0$, $\tau = 0$

		\FOR{$t = 0, \ldots, T$} 
		\IF {$(\det(V_t) > 2 \det(V_0)) ~ \textbf{and} ~ \left(t-\tau >\! H_0\right)$} 
			\STATE Estimate $\tth_t$ \& find optimistic $\ttt_t \in \mathcal{C}_{t}(\delta) \cap \mathcal{S}$ 
% 			\STATE Find $\ttt_t$ such that $J(\tilde{\Theta}_{t}) \leq \inf _{\Theta \in \mathcal{C}_{t}(\delta) \cap \mathcal{S}} J(\Theta)$
			\STATE Set $V_0 = V_t$ and $\tau = t$.
		\ELSE
			\STATE $\ttt_t = \ttt_{t-1}$
		\ENDIF
		\vspace{0.3em}
		\IF {$t \leq \Tw$ }
	        \STATE $u_t \!=\! K(\ttt_{t-1}) x_t \!+\! \nu_t~~$ \hfill  \textsc{\small{Improved Exploration}}  %One step delay
	    \ELSE
	        \STATE $u_t \!=\! K(\ttt_{t-1}) x_t $ \hfill  \textsc{\small{Stabilizing Control}}
	    \ENDIF
	   \vspace{0.2em}
	   \STATE Pay cost $c_t$ \& Observe $x_{t+1}$ 
	   \STATE Update $V_{t+1} \!=\! V_t \!+\! z_t z_t^\top$ for $z_t = [x_t^\top \enskip u_t^\top]^\top$
		\ENDFOR 
% 	- \hfill  \textsc{\small{Stabilizing Adaptive Control in }} -------- \\
% 	\sahin{fix the update scheme and the proof as well as state boundedness}
% 	\FOR{$i = 0, 1, \ldots$}
%     \STATE Estimate $\tth_i$ \& find optimistic $\ttt_i \in \mathcal{C}_{t}(\delta)$ 
%     \FOR{$t = \Tw + 2^i \Tbase, \ldots, \Tw + 2^{i+1}\Tbase - 1$ }
%         \STATE Deploy  $u_t \!=\! K(\ttt_{i}) x_t $ 
%     \ENDFOR
    % \ENDFOR	
	\end{algorithmic}
	\label{algo_exact} 
\end{algorithm}

The confidence set above provides a self-normalized bound on the model parameter estimates via design matrix $V_t$. \alg uses the \OFU principle in this confidence set to design a policy. In particular, it chooses an optimistic parameter $\ttt_t$ from $\mathcal{C}_t \cap \mathcal{S}$, which has the lowest expected optimal cost, and constructs the optimal linear controller $K(\ttt_t)$ for $\ttt_t$, \textit{i.e.} the optimistic controller. At time $t$, \alg uses the optimistic controller $K(\ttt_{t-1})$. This choice is for technical reasons to guarantee persistence of excitation (Appendix \ref{apx:smallest_eigen}). 

The optimistic controllers allow \alg to adaptively balance exploration and exploitation. They guide the exploration towards the region of state-space with the lowest expected cost. The key idea in this design is that as the confidence set shrinks, the performance of \alg improves over time \citep{bittanti2006adaptive}. 

Due to lack of an initial stabilizing policy, \alg aims to rapidly stabilize the system to avoid the consequences of unstable dynamics. To stabilize an unknown \LQR, one requires sufficient exploration in all directions of the state-space (Lemma \ref{stabilityofoptimisticcontroller}). Unfortunately, due to lack of reliable estimates in the early stages, the optimistic policies come short to guarantee such an effective exploration. 

Therefore, \alg deploys an adaptive control policy with an improved exploration in the early stages of interactions with the system. In particular, for the first $\Tw$ time-steps \alg uses isotropic perturbations along with the optimistic controller. For $t\leq \Tw$, it injects an i.i.d. Gaussian vector $\nu_t \!\sim\! \mathcal{N}(0,\sigma_{\nu}^2I)$ to the system besides the optimistic policy $K(\ttt_{t-1})x_t$, where $\sigma_{\nu}^2 = 2 \kappa^2 \bar\sigma_w^2 $.

\alg effectively excites and explores all dimensions of the system via this improved exploration strategy (Theorem \ref{lem:smallest_eigen}). The duration of the adaptive control with improved exploration phase is chosen such that \alg quickly finds a stabilizing controller. In particular, after $\Tw \coloneqq poly(\sigma_w,\sigma_\nu,n,d,\gamma^{-1}\!, \kappa, \overline{\alpha},\log(1/\delta))$
% \begin{equation}\label{explorationtimes}
%     \sahin{update this with new} \Tw \coloneqq poly(\sigma_w^2, \sigma_\nu^2, n, \log(1/\delta)) poly(\sigma_w,\sigma_\nu^{-1},n,d,\gamma^{-1}\!, \kappa, \overline{\alpha}).
% \end{equation}
time steps, \alg has the guarantee that the linear controllers $K(\ttt_{t-1})$ stabilize $\tts$ for all $t\geq \Tw$ with high probability (Lemma \ref{lem:2norm_bound} \& \ref{stabilityofoptimisticcontroller}). 

Moreover, \alg avoids frequent updates in the system estimates and the controller. It uses the same controller at least for a fixed time period of $H_0 \!=\! O(\gamma^{-1}\log(\kappa))$ and also waits for a significant improvement in the estimates.
The latter is achieved by updating the controller if the determinant of the design matrix $V_t$ is doubled since the last update. This update rule is chosen such that policy changes do not cause unstable dynamics for the stabilizable \LQR. The effect of this update rule on maintaining bounded state for \alg are studied in detail in Section \ref{stabilizable_extension}.   

\subsection{Stabilizing Adaptive Control}

After guaranteeing the stabilizing policy design, \alg starts the adaptive control that stabilizes the underlying system. In this phase, \alg stops injecting isotropic perturbations and relies on the balanced exploration and exploitation via the optimistic controller design. The stabilizing optimistic controllers further guide the exploration to adapt the structure of the problem and fine-tune the learning process to achieve optimal performance. However, note that the frequent policy changes can still cause unbounded growth of the state even though the policies are stabilizing. Therefore, \alg continues the same policy update rule in this phase to maintain bounded state. 

% in  with doubling length. It sets the base period to $\Tbase$ and for each  $i$, it runs for $2^{i-1}\Tbase$ time steps. Note that the frequent policy changes can still cause unbounded growth of the states. Therefore, $\Tbase$ is chosen to avoid frequent changes as in previous section (See Section \ref{stabilizable_extension}). 

% At the beginning of each  $i$, \alg uses all the past experiences up to  $i$ and estimates the model up its confidence interval $\mathcal{C}_t(\delta)$. In the policy design, \alg searches for the optimistic parameters only within $\mathcal{C}_t(\delta)$, since all the optimal controllers of the systems in this set are stabilizing with high probability. Finally, \alg deploys the optimistically chosen policy until the end of. The optimistic controller further guides the exploration to adapt the structure of the problem and fine-tunes the learning process to achieve optimal performance.

Unlike the prior works that constitute two distinct phases, \alg has a very subtle two-phase structure. In particular, the same subroutine (optimism) is applied continuously with the aim of balancing exploration and exploitation. An additional isotropic perturbation is only deployed for an improved exploration in the early stages to achieve stable learning for the autonomous agent.

\section{THEORETICAL ANALYSIS} \label{sec:theory}

In this section, we study the main theoretical contributions of this work. In Section \ref{stabilizable_extension}, we discuss the challenges that the stabilizability setting brings compared to the setting of the prior learning algorithms for the online \LQR. We then introduce our approaches to overcome these challenges in the design of \alg. In Section \ref{sec:early_improved_exploration}, we provide the formal statements for the theoretical guarantees of \alg and, finally, we give the regret upper bound of \alg in Section \ref{sec:regret_bound}. 

\subsection{Challenges in the Online Stabilizable \LQR Problem} \label{stabilizable_extension}
 
The main challenge for learning algorithms in control problems is to achieve input-to-state stability (ISS), which requires having well-bounded state in future time steps via using bounded inputs. Achieving this becomes significantly more challenging in the setting of stabilizable \LQR compared to their controllable counterpart considered in many recent works \citep{abbasi2011lqr,mania2019certainty,chen2020black}. A controllable system can be brought to $x_t = 0$ in finite time steps. Furthermore, some of these works assume that the underlying system to be closed-loop contractible, \textit{i.e.} $\|A_* - B_* K(\Theta_*) \| < 1 $. These facts significantly simplify the overall stabilization problem. Moreover, recalling Definition \ref{def:Stabilizability}, for controllable systems the controllability matrix is full row rank. In prior works, this has been a prominent factor in guaranteeing the persistence of excitation (PE) of the inputs, identifying the system and deriving regret bounds, e.g. \citep{hazan2019nonstochastic,chen2020black}. 

Unfortunately, we do not have these properties in the general stabilizable \LQR setting. Recall Assumption \ref{parameterassump_stabilizability} that states the system is $(\kappa, \gamma)$-stabilizable, which yields $\rho( A_* + B_* K(\tts)) \leq 1-\gamma$ for the optimal policy $K(\tts) \leq \kappa$. Therefore, even if the optimal policy of the underlying system is chosen by the learning algorithm, it may not produce contractive closed-loop system, \textit{i.e.}, we can have $\rho( A_* + B_* K(\tts)) < 1 < \|A_* + B_* K(\tts)\| $ since for any matrix $M$, $\rho(M) \leq \|M\|$. 

Moreover, from the definition of stabilizability in Definitions \ref{def:Stabilizability} and \ref{def:StrongStabilizability}, we know that for any stabilizing controller $K'$, there exists a similarity transformation $H'\succ 0$ such that it makes the closed loop system contractive, \textit{i.e.} $A_* + B_*K' = H'LH'^{-1}$, with $\|L\| < 1$. However, even if all the policies that \alg execute stabilize the underlying system, these different similarity transformations of different policies can further cause an explosion of state during the policy changes. If policy changes happen frequently, this may even lead to linear growth of the state over time. 

In order to resolve these problems, \alg carefully designs the timing of the policy updates and applies all the policies long enough, so that the state stays well controlled, \textit{i.e.}, ISS is achieved. To this end, \alg applies the same policy at least for $H_0 = 2\gamma^{-1} \log(2\kappa\sqrt{2})$ time steps. This particular choice prevents state blow-ups due to policy changes in the optimistic controllers in the stabilizable \LQR setting (see Appendix \ref{apx:boundedness}). 

To achieve PE and consistent model estimates under the stabilizability condition, we leverage the early improved exploration strategy which does not require controllability. Using the isotropic exploration in the early stages, we derive a novel lower bound for the smallest eigenvalue of the design matrix $V_t$ in the stabilizable \LQR with sub-Gaussian noise setting. Moreover, we derive our regret results using the fast stabilization and the optimistic policy design of \alg. The results only depend on the stabilizability and other trivial model properties such as the \LQR dimensions.

\subsection{Benefits of Early Improved Exploration}
\label{sec:early_improved_exploration}

To achieve effective exploration in the early stages, \alg deploys isotropic perturbations along with the optimistic policy for $t \leq \Tw$. Define $\sigma_\star>0$ where $\sigma_\star$ is a problem and in particular $\bar{\sigma}_w,\sigma_w, \sigma_\nu$-dependent constant (See Appendix \ref{apx:smallest_eigen} for exact definition). The following shows that for a long enough improved exploration, the inputs are persistently exciting the system.  

\begin{theorem}[Persistence of Excitation During the Improved Exploration] \label{lem:smallest_eigen}
If \alg follows the early improved exploration strategy for $T \geq poly(\sigma_w^2, \sigma_\nu^2, n, \log(1/\delta))$ time steps, then with probability at least $1-\delta$, \alg has $ \sigma_{\min}(V_{T}) \geq \sigma_{\star}^2 T$.
\end{theorem}

This theorem shows that having isotropic perturbations along with the optimistic controllers provides persistence excitation of the inputs, \textit{i.e.} linear scaling of the smallest eigenvalue of the design matrix $V_t$. This result is quite technical and its proof is given in Appendix \ref{apx:smallest_eigen}. At a high-level, we show that isotropic perturbations allow the covariates to have a Gaussian-like tail lower bound even in the stabilizable \LQR with sub-Gaussian process noise setting. Using the standard covering arguments, we prove the statement of the theorem. This result guarantees that the inputs excite all dimensions of the state-space and allows \alg to obtain uniformly improving estimates at a faster rate.

\begin{lemma}[Parameter estimation error]\label{lem:2norm_bound} 
Suppose Assumptions \ref{general_noise} and \ref{parameterassump_stabilizability} hold. For $T \!\geq\! poly(\sigma_w^2, \sigma_\nu^2, n, \log(1/\delta))$ time steps of adaptive control with improved exploration, with probability at least $1\!-\!2\delta$, \alg achieves 
$\|\tth_T-\tts\|_2 \leq \beta_t(\delta) / (\sigma_\star \sqrt{T})$.
\end{lemma}

This lemma shows that early improved exploration strategy using $\nu_t \!\sim\! \mathcal{N}(0,\sigma_\nu^2)$ for $\sigma_\nu^2 \!= \! 2 \kappa^2 \bar\sigma_w^2$ enables to guarantee the consistency of the parameter estimation. The proof is in Appendix \ref{apx:2norm}, where we combine the confidence set construction in Section \ref{sec:first_phase} with Theorem \ref{lem:smallest_eigen}. This bound is utilized to guarantee stabilizing controllers after early improved exploration. However, first we have the following lemma, which shows that there is a stabilizing neighborhood around $\tts$, such that $K(\Theta')$ stabilizes $\tts$ for any $\Theta'$ in this region.

\begin{lemma}[Strongly Stabilizable Neighborhood] \label{stabilityofoptimisticcontroller}  
For $D \!=\! \overline{\alpha} \gamma^{-1}\kappa^2(1+\kappa^2)$, let $C_{0} = 142D^8$ and  $\epsilon = 1/(54D^5)$. For any $(\kappa,\gamma)$-stabilizable system $\tts$ and for any $\varepsilon \leq \min\{ \sqrt{\bar{\sigma}_w^2nD/C_0},\epsilon\}$, such that $\|\Theta' - \tts \| \leq \varepsilon$, $K(\Theta')$ produces $(\kappa',\gamma')$-stable closed-loop dynamics on $\tts$ where $\kappa' = \kappa \sqrt{2} $ and $\gamma' = \gamma/2$.
\end{lemma}

 The proof is given in Appendix \ref{apx:neighborhood}. This lemma shows that to guarantee the stabilization of the unknown dynamics a learning agent should have uniformly sufficient exploration in all directions of the state-space. By the choice of $\Tw$ (precise expression given in Appendix \ref{apx:boundedness}) and using Lemma \ref{lem:2norm_bound}, \alg guarantees to quickly find this stabilizing neighborhood with high probability due to the adaptive control with improved exploration phase of $\Tw$ time steps.
 
 For the remaining time steps, $t \geq \Tw$, \alg starts redressing the possible state explosion due to unstable controllers and the perturbations in the early stages. Define $\Tbase$ and $T_{r}$ such that $\Tbase \!=\! (n+d)\log(n+d)H_0$ and $T_{r} = \Tw + \Tbase$. Recall that $H_0$ is the minimum duration for a controller such that the state is well-controlled despite the policy changes. The following shows that the stabilizing controllers are applied long enough that the state stays bounded for $T\!>\!T_{r}$.

\begin{lemma}[Bounded states]\label{lem:bounded_state} 
Suppose Assumption \ref{general_noise} \& \ref{parameterassump_stabilizability} hold. For given $\Tw$ and $\Tbase$, \alg controls the state such that $\|x_t\| = O((n+d)^{n+d})$ for $t\leq T_{r}$, with probability at least $1-2\delta$ and $\|x_t\| \! \leq\! (12\kappa^2\!+\! 2\kappa\sqrt{2})\gamma^{-1}\sigma_w \sqrt{2n\log(n(t\!-\!\Tw)/\delta)}$ for $T \!\geq\! t \!>\! T_{r}$, with probability at least $1-4\delta$.  
\end{lemma}

In the proof (Appendix \ref{apx:boundedness}), we show the policies seldom change via determinant doubling condition or the lower bound of $H_0$ for the adaptive control with improved exploration phase to keep the state bounded. For the stabilizing adaptive control, we show that deploying stabilizing policies for at least $H_0$ time-steps provides an exponential decay on the state and after $\Tbase$ time-steps brings the state to an equilibrium. 

\subsection{Regret Upper Bound of \alg} \label{sec:regret_bound}
After showing the effect of fast stabilization, we can finally present the regret upper bound of \alg.

\begin{theorem}[Regret of \alg]  \label{reg:exp_s}
Suppose Assumptions \ref{general_noise} and \ref{parameterassump_stabilizability} hold. For the given choices of $\Tw$ and $\Tbase$, with probability at least $1-4\delta$, \alg achieves regret of $\OO\big(\text{poly}(n,d) \sqrt{T\log(1/\delta)}\big)$, for long enough $T$.
\end{theorem}

The proofs and the exact expressions are presented in Appendix \ref{apx:regret_analysis}. Here, we provide a proof sketch. The regret decomposition leverages the optimistic controller design. Recall that for the early improved exploration, \alg applies independent perturbations through the controller yet still deploys the optimistic policy. Thus, we consider this external perturbation as a part of the underlying system and study the regret obtained by the improved exploration strategy separately.

In particular, denote the system evolution noise at time $t$ as $\zeta_t$. For $t\leq \Tw$, system evolution noise can be considered as $\zeta_t =B_*\nu_t + w_t$ and for $t> \Tw$, $\zeta_t = w_t$. We denote the optimal average cost of system $\ttt$ under $\zeta_t$ as $J_*(\ttt, \zeta_t)$. Using the Bellman optimality equation for \LQR \citep{bertsekas1995dynamic}, we consider the system evolution of the optimistic system $\ttt_t$ using the optimistic controller $K(\ttt_t)$ in parallel with the true system evolution of $\tts$ under $K(\ttt_t)$ such that they share the same process noise (See details in Appendix \ref{apx:regret_analysis}). Using the confidence set construction, optimistic policy, Lemma \ref{lem:bounded_state}, Assumption \ref{parameterassump_stabilizability} and Lemma \ref{bounded_P}, we get a regret decomposition and bound each term separately.
% in Appendix \ref{apx:regret_analysis}. 

At a high-level, the exact regret expression has a constant regret term due to early additional exploration for $T_w$ time-steps with exponential dimension dependency and a term that scales with square root of the duration of stabilizing adaptive control with polynomial dimension dependency, \textit{i.e.} $(n+d)^{n+d} T_w  \!+\! \text{poly}(n,d) \sqrt{T-T_w}$. Note that $\Tw$ is a problem dependent expression. Thus, for large enough $T$, the polynomial dependence dominates, giving Theorem \ref{reg:exp_s}.

\section{EXPERIMENTS} \label{sec:Exp}

\begin{table}
% \hspace{-2.2em}
\centering
\captionsetup{justification=centering}
\caption{Regret Performance After 200 Time Steps in Marginally Unstable Laplacian System. \alg outperfoms other algorithms by a significant margin}
\label{table:2}
\resizebox{0.48\textwidth}{!}{
\begin{tabular}{l c c c c}
 \textbf{Algo.} \!\!\!\!\!\! & \!\! \makecell{ \textbf{Avg.} \\ \textbf{Regret} } \!\! & \!\!   \makecell{ \textbf{Top} \\ $\mathbf{90\%}$ }  \!\! & \!\!  \makecell{ \textbf{Top} \\ $\mathbf{75\%}$ } \!\! & \!\!   \makecell{ \textbf{Top} \\ $\mathbf{50\%}$ } \!\!  \\
  \hline
 \alg \!\!\!\!\!\! & \!\! $\bf{1.5 \!\! \times \!\! 10^4}$   \!\! & \!\! $\bf{1.3 \!\! \times \!\! 10^4}$ \!\! & \!\! $\bf{1.1 \!\! \times \!\! 10^4 }$ \!\! & \!\! $\bf{8.9  \!\! \times \!\! 10^3} $\!\! \\

 OFULQ  \!\!\!\!\!\! & \!\! $6.2 \!\! \times \!\! 10^{10}$  \!\! & \!\!  $4.0 \!\! \times \!\! 10^6$ \!\! & \!\! $3.5 \!\! \times \!\! 10^5$ \!\! & \!\! $4.7  \!\! \times \!\! 10^4 $ \!\!\\

CEC-Fix \!\!\!\!\!\! & \! $3.7 \!\! \times \!\! 10^{10}$ \!\! & \!\! $2.1 \!\! \times \!\! 10^{4}$  \!\! & \!\! $1.9 \!\! \times \!\! 10^{4}$  \!\! & \!\!   $1.7 \!\! \times \!\! 10^{4}$\!\!  \\

  CEC-Dec \!\!\!\!\!\! & \!\! $4.6 \!\! \times \!\! 10^{4}$  \!\! & \!\! $4.0 \!\! \times \!\! 10^{4}$ \!\! & \!\! $3.5 \!\! \times \!\! 10^4$ \!\! & \!\! $2.8 \!\! \times \!\! 10^{4}$\!\! \\
  
\end{tabular}}
% \vspace{-1em}
\end{table}

In this section, we evaluate the performance of \alg in four adaptive control tasks: \textbf{(1)} a marginally unstable Laplacian system~\citep{dean2018regret}, \textbf{(2)} the longitudinal flight control of Boeing 747 with linearized dynamics~\citep{747model}, \textbf{(3)} unmanned aerial vehicle (UAV) that operates in a 2-D plane~\citep{zhao2021infinite},  and \textbf{(4)} a stabilizable but not controllable linear dynamical system. For each task, we compare \alg with three \RL algorithms: (i) \OFULQ of \citet{abbasi2011lqr}; (ii) certainty equivalent controller with fixed isotropic perturbations (CEC-Fix), which is the standard baseline in control theory; and (iii) certainty equivalent controller with decaying isotropic perturbations (CEC-Dec), which is shown to \textit{achieve optimal regret with a given initial stabilizing policy} \citep{simchowitz2020naive, dean2018regret,mania2019certainty}. In the implementation of CEC-Fix and CEC-Dec, the optimal control policies of the estimated model are deployed. Furthermore, in finding the optimistic parameters for \alg and \OFULQ, we use projected gradient descent within the confidence sets. We perform $200$ independent runs for each algorithm for $200$ time steps starting from $x_0 = 0$. We present the performance of best parameter choices for each algorithm. For further details and the experimental results please refer to Appendix \ref{apx:experiment_detail}.

Before discussing the experimental results, we would like to highlight the baselines choices. Unfortunately, there are only a few works in literature that consider \RL in \LQR{}s without a stabilizing controller. These works are \OFULQ of \citep{abbasi2011lqr}, \citep{abeille2018improved}, and \citep{chen2020black}. Among these, \citep{chen2020black} considers \LQR{}s with adversarial noise setting and deploys \textit{impractically large inputs}, e.g. $10^{28}$ for task \textbf{(1)}, whereas the algorithm of \citep{abeille2018improved} only works in scalar setting. These prohibit meaningful regret and stability comparisons, thus, we compare \alg against the only relevant comparison of \OFULQ among these. Moreover, there are only a few and limited experimental studies in the literature of \RL in \LQR{}s. Among these, \citep{dean2018regret,faradonbeh2018input,faradonbeh2020adaptive} highlight the superior performance of CEC-Dec. Therefore, we compare \alg against CEC-Dec with the best-performing parameter choice, as well as the standard control baseline of CEC-Fix.

\begin{figure}[t]
\centering
  \includegraphics[width=0.99\linewidth]{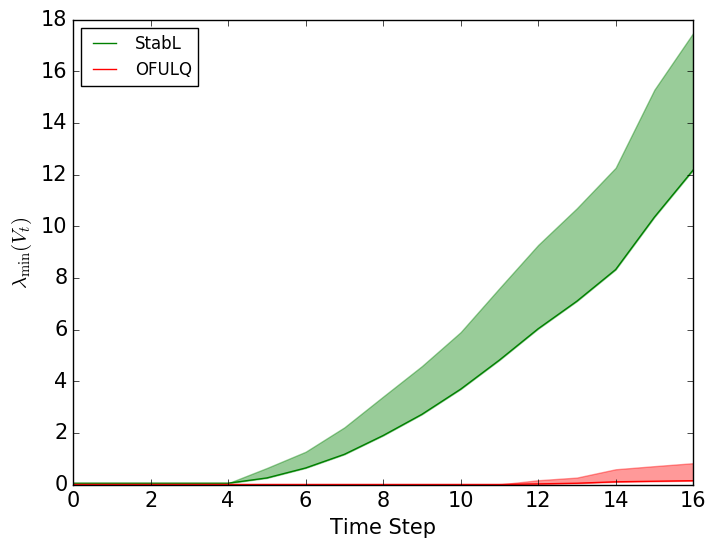}
  \caption{Evolution of the smallest eigenvalue of the design matrix for \alg and \OFULQ in Laplacian system. The solid line is the mean and the shaded region is one standard deviation. \alg attains linear scaling whereas \OFULQ suffers from lack of early exploration.}
  \label{fig:compare_main}
\end{figure}

\begin{table}[t]
% \hspace{-2.2em}
\centering
\captionsetup{justification=centering}
\caption{Maximum State Norm in the Laplacian System. \alg keeps the state smallest}
% \label{table:4}
\resizebox{0.48\textwidth}{!}{
\begin{tabular}{l c c c c}
\!\! \textbf{Algo.} \!\!\!\!\!\! & \!\! \makecell{ \textbf{Avg.} \\ max$ \|x\|_2$}
%  \textbf{Avg. $\max\|x\|_2$}
  \!\! & \!\! \makecell{ \textbf{Worst} \\ $\mathbf{5\%}$ }  \!\! & \!\!  \makecell{ \textbf{Worst} \\ $\mathbf{10\%}$ } \!\! & \!\!  \makecell{ \textbf{Worst} \\ $\mathbf{25\%}$ }\\
  \hline
 \!\! \alg \!\!\!\!\!\! & \!\! $\bf{1.3 \!\!\times\!\! 10^1}$  \!\! & \!\! $\bf{2.2 \!\!\times\!\! 10^1}$ \!\! & \!\! $\bf{2.1 \!\!\times\!\! 10^1}$ \!\! & \!\! $\bf{1.9 \!\!\times\!\! 10^1}$  \\
 \!\! \OFULQ \!\!\!\!\!\! &\!\! $9.6 \!\!\times\!\! 10^{3}$ \!\!& \!\!$1.8 \!\!\times\!\! 10^5$ \!\! &\!\!  $9.0\!\! \times \!\! 10^4$\!\! & \!\!$3.8 \!\!\times\!\! 10^4$  \\
\!\! CEC-Fix \!\!\!\!\!\! & \!\! $3.3 \!\!\times\!\! 10^{3}$ \!\!&\!\! $6.6 \!\!\times\!\! 10^{4}$ \!\!  & \!\! $3.3 \!\!\times\!\! 10^{4}$ \!\! & \!\!  $1.3 \!\!\times\!\! 10^{4}$  \\
 \!\! CEC-Dec \!\!\!\!\!\! &\!\! $2.0 \!\!\times\!\! 10^{1}$ \!\! &\!\! $3.5 \!\!\times\!\! 10^{1}$ \!\!& \!\! $3.3 \!\!\times\!\! 10^1$ \!\!&\!\! $2.9 \times 10^{1}$
\end{tabular}}
\label{table:max_laplace_main}
\end{table}

\textbf{(1) Laplacian system (Appendix \ref{apx:I1}).}
% In this system $Q = 10I$ and  $R = I$. 
Table \ref{table:2} provides the regret performance for the average, top $90\%$, top $75\%$ and top $50\%$ of the runs of the algorithms. We observe that \alg attains at least an order of magnitude improvement in regret over \OFULQ and CEC{}s. This setting combined with the unstable dynamics is challenging for the \textit{solely} optimism-based learning algorithms. Our empirical study indicates that, at the early stages of learning, the smallest eigenvalue of the design matrix $V_t$ for \OFULQ is much smaller than that of \alg as shown in Figure~\ref{fig:compare_main}. The early improved exploration strategy helps \alg achieve linear scaling in $\lambda_{\min}(V_t)$, thus persistence of excitation and identification of stabilizing controllers. In contrast, the only \OFU-based controllers of \OFULQ fail to achieve persistence of excitation and accurate estimate of the model parameters. Therefore, due to lack of reliable estimates and the skewed cost, \OFULQ cannot design effective strategies to learn model dynamics and results in unstable dynamics (see Table~\ref{table:max_laplace_main}). Table \ref{table:max_laplace_main} displays the stabilization capabilities of the deployed \RL algorithms. In particular, it provides the averages of the maximum norms of the states for all runs, the worst $5\%$, $10\%$ and $25\%$ runs. Of all algorithms, \alg keeps the state smallest.

\textbf{(2) Boeing 747 (Appendix \ref{apx:I2}).} In practice, nonlinear systems, like Boeing 747, are modeled via local linearizations which hold as long as the states are within a certain region. Thus, to maintain the validity of such linearizations, the state of the underlying system must be well-controlled, \textit{i.e.}, stabilized. Table \ref{table:3} provides the regret performances and Table \ref{table:4} displays the stabilization capabilities of the deployed \RL algorithms similar to \textbf{(1)}. Once more, among all algorithms, \alg maintains the maximum norm of the state smallest and operates within the smallest radius around the linearization point of origin. This observation is consistent among tasks \textbf{(3)} and \textbf{(4)}, which shows that \alg maintains tightly bounded state with high probability. The specifics of the maximum state results on \textbf{(3)} and \textbf{(4)} are given in the Appendix \ref{apx:I3} and \ref{apx:I4} respectively.

\begin{table}[t]
% \hspace{-2.2em}
\centering
\captionsetup{justification=centering}
\caption{Regret Performance After 200 Time Steps in Boeing 747 Flight Control. \alg outperfoms others.}
\label{table:3}
\resizebox{0.48\textwidth}{!}{
\begin{tabular}{l c c c c}
 \textbf{Algo.} \!\!\!\!\!\! & \!\! \makecell{ \textbf{Avg.} \\ \textbf{Regret} } \!\! & \!\!   \makecell{ \textbf{Top} \\ $\mathbf{90\%}$ }  \!\! & \!\!  \makecell{ \textbf{Top} \\ $\mathbf{75\%}$ } \!\! & \!\!   \makecell{ \textbf{Top} \\ $\mathbf{50\%}$ } \!\!  \\
  \hline
 \alg \!\!\!\!\!\! & \!\! $\bf{1.3 \!\! \times \!\! 10^4}$   \!\! & \!\! $\bf{9.6 \!\! \times \!\! 10^3}$ \!\! & \!\! $\bf{7.6 \!\! \times \!\! 10^3 }$ \!\! & \!\! $\bf{5.3  \!\! \times \!\! 10^3} $\!\! \\

 OFULQ  \!\!\!\!\!\! & \!\! $1.5 \!\! \times \!\! 10^{8}$  \!\! & \!\!  $9.9 \!\! \times \!\! 10^5$ \!\! & \!\! $5.6 \!\! \times \!\! 10^4$ \!\! & \!\! $8.9  \!\! \times \!\! 10^3 $ \!\!\\

CEC-Fix \!\!\!\!\!\! & \! $4.8 \!\! \times \!\! 10^{4}$ \!\! & \!\! $4.5 \!\! \times \!\! 10^{4}$  \!\! & \!\! $4.3 \!\! \times \!\! 10^{4}$  \!\! & \!\!   $3.9 \!\! \times \!\! 10^{4}$\!\!  \\

  CEC-Dec \!\!\!\!\!\! & \!\! $2.9 \!\! \times \!\! 10^{4}$  \!\! & \!\! $2.5 \!\! \times \!\! 10^{4}$ \!\! & \!\! $2.2 \!\! \times \!\! 10^4$ \!\! & \!\! $1.9 \!\! \times \!\! 10^{4}$\!\! \\
  
\end{tabular}}

% \vspace{-1em}
\end{table}

\begin{table}[t]
% \hspace{-2.2em}
\centering
\captionsetup{justification=centering}
\caption{Maximum State Norm in Boeing 747 Control. \alg keeps the state smallest.}
% \label{table:4}
\resizebox{0.48\textwidth}{!}{
\begin{tabular}{l c c c c}
 \!\! \textbf{Algo.} \!\!\!\!\!\! & \!\! \makecell{ \textbf{Avg.} \\ max$ \|x\|_2$}
%  \textbf{Avg. $\max\|x\|_2$}
  \!\! & \!\! \makecell{ \textbf{Worst} \\ $\mathbf{5\%}$ }  \!\! & \!\!  \makecell{ \textbf{Worst} \\ $\mathbf{10\%}$ } \!\! & \!\!  \makecell{ \textbf{Worst} \\ $\mathbf{25\%}$ }\\
  \hline
 \!\! \alg \!\!\!\!\!\! & \!\! $\bf{3.4 \!\!\times\!\! 10^1}$ \!\!  & \!\! $\bf{7.5 \!\!\times\!\! 10^1}$ \!\! & \!\!  $\bf{7.0 \!\!\times\!\! 10^1}$ \!\! & \!\!  $\bf{5.2 \!\! \times \!\! 10^1 }$  \\

 \!\! OFULQ  \!\!\!\!\!\!  &   $1.6 \!\! \times \!\! 10^{3}$ \!\! & \!\!  $2.2 \!\! \times\!\! 10^4$   \!\! & \!\!   $1.4 \!\! \times \!\! 10^4$  \!\! & \!\!  $6.3 \!\!\times\!\! 10^3$  \\

 \!\! CEC-Fix  \!\!\!\!\!\! &   $5.0 \!\!\times\!\! 10^{1}$  \!\! & \!\!  $7.8 \!\!\times\!\! 10^{1}$   \!\! & \!\!  $7.3 \!\!\times\!\! 10^{1}$  \!\! & \!\!   $6.5 \!\!\times\!\! 10^{1}$  \\

  \!\! CEC-Dec \!\!\!\!\!\!  &   $4.6 \!\!\times\!\! 10^{1}$   \!\! &  \!\! $8.0 \!\!\times\!\! 10^{1}$ \!\!  & \!\!  $7.3 \!\!\times\!\! 10^1$  \!\! & \!\!  $6.3 \!\!\times\!\! 10^{1}$ 
  
\end{tabular}}
\label{table:4}
\end{table}

\textbf{(4) Stabilizable but not controllable system (Appendix \ref{apx:I4}).} Besides \alg, which is tailored for the general stabilizable setting, other algorithms perform poorly in this challenging setting. In particular, CEC-Fix drastically blows up the state due to significantly unstable dynamics for the uncontrollable part of the system. Therefore, the regret performances of only \alg, \OFULQ and CEC-Dec are presented in Figure \ref{fig:regret_stabil_main}. Figure \ref{fig:regret_stabil_main} is in semi-log scale and \alg provides an order of magnitude improved regret compared to the best performing state-of-art baseline CEC-Dec.

\begin{figure}[t]
\centering
  \includegraphics[width=0.99\linewidth]{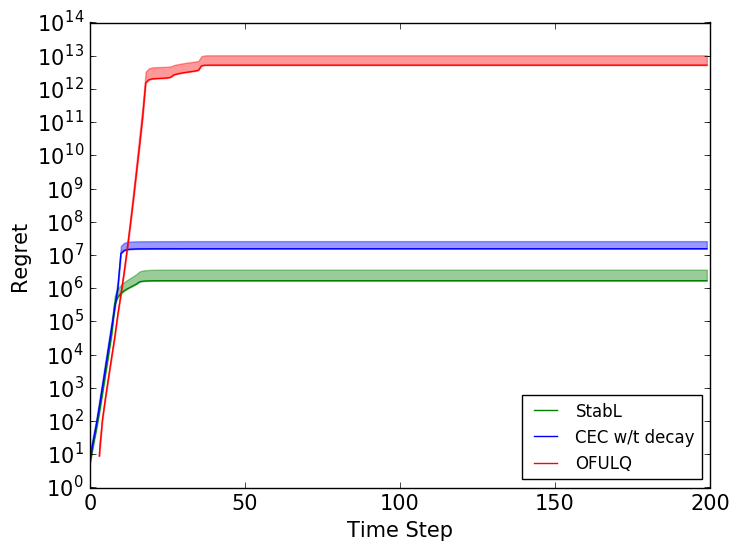}
  \caption{Regret Comparison of three algorithms in controlling a stabilizable but not controllable system. The solid lines are the average regrets and the shaded regions are the quarter standard deviations.}
  \label{fig:regret_stabil_main}
\end{figure}

\section{RELATED WORK}
\textbf{Finite-time regret guarantees:} Prior works study the problem of regret minimization in \LQR{}s and achieve sublinear regret using CEC{}s \citep{mania2019certainty, faradonbeh2018input,faradonbeh2020adaptive}, robust controllers \citep{dean2018regret}, the \OFU principle \citep{abeille2020efficient}, Thompson sampling 
\citep{abeille2018improved} and an SDP relaxation \citep{cohen2019learning} with a lower bound provided in \citet{simchowitz2020naive}. These works all assume that an initial stabilizing policy is given and do not design autonomous learning agents which is the main focus of this paper. Among these, \citet{simchowitz2020naive} provide the tight regret guarantee for the setting with known initial stabilizing policy. Their proposed algorithm follows the given non-adaptive initial stabilizing policy for a \textit{long} period of time with isotropic perturbations. 
% The goal is to obtain significantly accurate model estimates, such that the sub-optimality gap in the optimal costs is quadratic in estimation error. 
Thus, they provide an order-optimal theoretical regret upper bound with an additional large constant regret. However, in many applications, e.g. medical, such constant regret, and non-adaptive controllers are not tolerable. \alg aims to address these challenges and provide an adaptive algorithm that can be deployed in practice. Moreover, \alg achieves significantly improved performance over the prior baseline \RL algorithms in various adaptive control tasks (Section \ref{sec:Exp}). 

\textbf{Finding a stabilizing controller:} Similar to the regret minimization, there has been a growing interest in finite-time stabilization of linear dynamical systems \citep{dean2019sample,faradonbeh2018finite,faradonbeh2019randomized}. Among these works, \citet{faradonbeh2018finite} is the closest to our work. However, there are significant differences in the methods and the span of the results. In \citet{faradonbeh2018finite}, random linear controllers are used solely for finding a stabilizing set without a control goal. This results in the explosion of state, presumably exponentially in time, leading to a regret that scales exponentially in time. The proposed method provides many insightful aspects for finding a stabilizing set in finite-time, yet a cost analysis of this process or an adaptive control policy are not provided. Moreover, the stabilizing set in \citet{faradonbeh2018finite} relates to the minimum value that satisfies a specific condition for the roots of a polynomial.
This results in a somewhat implicit sample complexity for constructing such a set. On the other hand, in this work, we provide a complete study of an autonomous learning algorithm for the online \LQR problem. Among our results, we give an explicit formulation of the stabilizing set and a sample complexity that only relates to the minimal stabilizability information of the system.

% \textbf{\OFU based works:} For control problems, \OFU principle was first used by \citet{campi1998adaptive}. Besides adaptive control of \LQR, \OFU principle is widely used in a variety of decision making paradigm, such as multi-arm bandit~\citep{auer2002using}, linear bandit~\citep{abbasi2011improved}, Markov Decision Processes~\citep{jaksch2010near,azizzadenesheli2016reinforcement}, and adaptive control of partially observable linear quadratic control systems~\citep{lale2020regret,lale2020root}. Recently, \cite{abeille2020efficient} provide computationally feasible optimism in the face uncertainty approach for large dimensional \LQR{}s. However, they assume access to a stabilizing set in the beginning of the algorithm where the regret of obtaining such set or its effect on the state is not analyzed.

\textbf{Generalized \LQR setting:} 
% Besides the standard setting of online \LQR problem considered in this paper, there are various works that consider more general settings. One line of research considers the partially observable counterpart of \LQR termed as \LQG. It is the setting where instead of exact state vector, a noisy linear combination of the state vector is observed. In this setting \citet{mania2019certainty, simchowitz2020improper, lale2020root} achieve $\sqrt{T}$ regret and \citet{lale2020logarithmic} introduces the algorithm that achieves first polylogarithmic regret in this setting. 
Another line of research considers the generalizations of the online \LQR problem under partial observability~\citep{lale2020regret,lale2020root,lale2020logarithmic,mania2019certainty,simchowitz2020improper} or adversarial disturbances~\citep{hazan2019nonstochastic,chen2020black}. These works either assume a given stabilizing controller or open-loop stable system dynamics, except \citet{chen2020black}. Independently and concurrently, the recent work by \citet{chen2020black} designs an autonomous learning algorithm and regret guarantees that are similar to the current work. However, the approaches and the settings have major differences. \citet{chen2020black} considers the restrictive setting of \textit{controllable} systems, yet with adversarial disturbances and general cost functions. They inject \textit{significantly} big inputs, \textit{exponential in system parameters}, with a pure exploration intent to guarantee the recovery of system parameters and stabilization. This negatively affects the practicality of the algorithm. On the other hand, in this work, we inject isotropic Gaussian perturbations to improve the exploration in the stochastic (sub-Gaussian process noise) \textit{stabilizable} \LQR while still aiming to control, \textit{i.e.} no pure exploration phase. This yields a practical \RL algorithm \alg that attains state-of-the-art performance. 

\section{CONCLUSION}
In this paper, we propose an \RL framework, \alg, that follows \OFU principle to balance between exploration and exploitation in interaction with \LQR{}s. We show that if an additional random exploration is enforced in the early stages of the agent's interaction with the environment, \alg has the guarantee to design a stabilizing controller sooner. We then show that while the agent enjoys the benefit of stable dynamics in further stages, the additional exploration does not alter the early performance of the agent considerably. Finally, we prove that the regret upper bound of \alg is $\OO(\sqrt{T})$ with polynomial dependence in the problem dimensions of the \LQR{}s in stabilizable systems. 

Our results highlight the benefit of early improved exploration to achieve improved regret at the expense of a slight increase in regret in the early stages. An important future direction is to study this phenomenon in more challenging online control problems in linear systems, e.g., under partially observability. Another interesting direction is to combine this mindset with the existing state-of-the-art model-based \RL approaches for the general systems and study their performance.

\bibliographystyle{unsrtnat}
\bibliography{main}

%%%%%%%%%%%%%%%%%%%%%%%%%%%%%%%%%%%
%%%%%% SUPPLEMENT (OPTIONAL) %%%%%%
%%%%%%%%%%%%%%%%%%%%%%%%%%%%%%%%%%%

\clearpage
\appendix

\thispagestyle{empty}

% For one-column format, uncomment the following:
\onecolumn \makesupplementtitle
% For two-column format, uncomment the following:
%\twocolumn[ \makesupplementtitle ]

In Appendix \ref{apx:neighborhood}, we first provide a discussion on how stabilizability and $(\kappa,\gamma)$-stabilizable systems are equivalent. Then we prove that for $(\kappa,\gamma)$-stabilizable systems the unique positive definite solution to the DARE given in \eqref{ARE} is bounded. Finally, we show that there exists a stabilizing neighborhood (ball) around the true model parameters in which the optimal controllers of the models within this ball stabilize the underlying system in Appendix \ref{apx:neighborhood}. In Appendix \ref{apx:smallest_eigen}, we show that due to improved exploration strategy, the regularized design matrix $V_t$ has its minimum eigenvalue scaling linearly over time, which guarantees the persistently exciting inputs for finding the stabilizing neighborhood and stabilizing controllers after the adaptive control with improved exploration phase. The exact definition of $\sigma_\star$ is also given in Lemma \ref{lem:smallest_eigen_apx} in Appendix \ref{apx:smallest_eigen}. We provide the system identification and confidence set constructions with their guarantees (both in terms of self-normalized and spectral norm) in Appendix \ref{apx:2norm}. In Appendix \ref{apx:boundedness}, we provide the boundedness guarantees for the system's state throughout the execution of \alg and provide the proof of Lemma \ref{lem:bounded_state}. The precise definition of $\Tw$, which was omitted in the main text, is also given in \eqref{exploration_duration} in Appendix \ref{apx:boundedness}. We provide the regret decomposition in Appendix \ref{apx:regret_decomp} and we analyze each term in this decomposition and give the proof of the main result of the paper in Appendix \ref{apx:regret_analysis}. In Appendix \ref{apx:comparecontrol}, we compare the results with \citet{abbasi2011lqr} and show it subsumes and improves the prior work. Appendix \ref{apx:technical} provides the technical theorems and lemmas that are utilized in the proofs. Finally, in Appendix \ref{apx:experiment_detail}, we provide the details on the experiments including the dynamics of the adaptive control tasks, parameter choices for the algorithms and additional experimental results.

\section{STABILIZABILITY OF THE UNDERLYING SYSTEM}\label{apx:neighborhood}
In this section, we first show that $(\kappa,\gamma)$-stabilizability is merely a quantification of  stabilizability. Then, we show that the given systems (both controllable and stabilizable) DARE has a unique positive definite solution. Finally, we show that combining two prior results, there exists a stabilizing neighborhood round the system parameters that any controller designed using parameters in that neighborhood stabilizes the system. 

\subsection[Stabilizability]{$(\kappa,\gamma)$-stabilizability}
Any stabilizable system is also $(\kappa,\gamma)$-stabilizable for some $\kappa$ and $\gamma$ and the conversely $(\kappa,\gamma)$-stabilizability implies stabilizability. In particular, for all stabilizable systems, by setting $1-\gamma = \rho(A_* + B_* K(\Theta_*))$ and $\kappa$ to be the condition number of $P(\Theta_*)^{1/2}$ where $P(\Theta_*)$ is the positive definite matrix that satisfies the following Lyapunov equation:
\begin{equation}
    (A_* + B_* K(\Theta_*))^\top P(\Theta_*) (A_* + B_* K(\Theta_*)) \preceq P(\Theta_*),
\end{equation}
one can show that $A_* + B_*K(\tts) = HLH^{-1}$, where $H = P(\Theta_*)^{-1/2}$ and $L = P(\Theta_*)^{1/2} (A_* + B_*K(\tts))P(\Theta_*)^{-1/2}$ with $\|H\| \|H^{-1} \| \leq \kappa $,  and  $\|L\| \leq 1-\gamma$ ~(Lemma B.1 of \citet{cohen2018online}).

\subsection[Bound on the Solution of DARE for Stabilizable Systems, Proof of Lemma \ref{bounded_P}]{Bound on the Solution of DARE for $(\kappa,\gamma)$-Stabilizable Systems, Proof of Lemma \ref{bounded_P}}
\textbf{Proof of Lemma \ref{bounded_P}:}
Recall the DARE given in \eqref{ARE}. The solution of this equation corresponds to recursively applying the following
\begin{align}
    \| P_*\| &= \| \sum\nolimits^{\infty}_{t=0} \left((A_* + B_* K(\Theta_*))^t\right)^\top \left(Q + K(\Theta_*)^\top R K(\Theta_*)\right) \left(A_* + B_* K(\Theta_*)\right)^t \| \nonumber \\
    &= \| \sum\nolimits^{\infty}_{t=0} \left(HL^tH^{-1}\right)^\top \left(Q + K(\Theta_*)^\top R K(\Theta_*)\right) \left(HL^tH^{-1}\right) \| \nonumber \\
    &\leq \overline{\alpha} (1+\|K(\Theta_*)\|^2) \|H\|^2\|H^{-1}\|^2  \sum\nolimits^{\infty}_{t=0} \|L\|^{2t} \label{step:asmm}\\
    &\leq \overline{\alpha} \gamma^{-1}\kappa^2(1+\kappa^2) \label{step:last}
\end{align}
where \eqref{step:asmm} follows from the upper bound on $\|Q\|, \|R\| \leq \overline{\alpha} $ and \eqref{step:last} follows from the definition of $(\kappa,\gamma)$-stabilizability. 
\null\hfill$\blacksquare$

\subsection{Stabilizing Neighborhood Around the System Parameters}

\begin{theorem}[Unique Positive Definite Solution to DARE, \citep{bertsekas1995dynamic}]
For $\tts = (A_*,B_*)$, If $(A_*,B_*)$ is stabilizable and $(C,A_*)$ is observable for $Q = C^\top C$, or $Q$ is positive definite, then there exists a unique, bounded solution, $P(\tts)$, to the DARE:
\begin{equation} \label{ARE_apx}
P(\tts) = A_*^\top P(\tts) A_* + Q -  A_*^\top P(\tts) B_* \left( R + B_*^\top P(\tts) B_* \right)^{-1} B_*^\top P(\tts) A_*.
\end{equation}
The controller $K(\tts) = -\left(R+B_*^{\top} P(\tts) B_*\right)^{-1} B_*^{\top} P(\tts) A_*$ produces stable closed-loop system, $\rho(A_*+B_*K(\tts)) < 1$.
\end{theorem}

This result shows that, for we get unique positive definite solution to DARE for stabilizable systems. Let $J_* \leq \mathcal{J}$.
The following lemma is introduced in \citet{simchowitz2020naive} and shows that if the estimation error on the system parameters is small enough, then the performance of the optimal controller synthesized by these model parameter estimates scales quadratically with the estimation error. 
\begin{lemma}[\citep{simchowitz2020naive}]
For constants $C_0 = 142\|P_*\|^8$ and $\epsilon=\frac{54}{\|P_*\|^5}$, such that, for any $0 \leq \varepsilon \leq \epsilon$
and for $\|\Theta' - \tts\|\leq \varepsilon$, the infinite horizon performance of the policy $K(\Theta')$ on $\tts$ obeys the following,
\begin{equation*}
    J(K(\Theta'),A_*,B_*,Q,R) - J_* \leq C_0 \varepsilon^2.
\end{equation*}
\end{lemma}
This result shows that there exists a $\epsilon$-neighborhood around the system parameters that stabilizes the system. This result further extended to quantify the stability in \citet{cassel2020logarithmic}. 

\begin{lemma}[Lemma 41 in \citet{cassel2020logarithmic}] \label{lemmaforstable}
Suppose $J(K(\Theta'),A_*,B_*,Q,R) \leq \mathcal{J}'$ for the \LQR under Assumption \ref{general_noise}, then $K(\Theta')$ produces $(\kappa',\gamma')$-stable closed-loop dynamics where $\kappa'=\sqrt{\frac{\mathcal{J}'}{\underline{\alpha}\Bar{\sigma}_w^2}}$ and $\gamma' = 1/2\kappa'^2 $.
\end{lemma}
Combining these results, we obtain the proof of Lemma \ref{stabilityofoptimisticcontroller}.

% \begin{lemma} [Strongly Stabilizable Neighborhood] \label{stabilityofoptimisticcontroller}
% Under Assumptions \ref{general_noise} \& \ref{parameterassump_stabilizability}, for any $\varepsilon \leq \min\{ \sqrt{\bar{\sigma}_w^2nD/C_0},\epsilon\}$, such that $\|\Theta' - \tts \| \leq \varepsilon$ for any $(\kappa,\gamma)$-stabilizable system $\tts$, $K(\Theta')$ produces $(\kappa',\gamma')$-stable closed-loop dynamics on $\tts$ where $\kappa' = \kappa \sqrt{2} $ and $\gamma' = \gamma/2$.
% \end{lemma}
\textbf{Proof of Lemma \ref{stabilityofoptimisticcontroller}:}
 Under Assumptions \ref{general_noise} \& \ref{parameterassump_stabilizability}, for $\varepsilon \leq \min\{ \sqrt{\mathcal{J}/C_0},\epsilon\}$, we obtain $J(K(\Theta'),A_*,B_*,Q,R) \leq 2\mathcal{J}$. Plugging this into Lemma \ref{lemmaforstable} gives the presented result. 

\null\hfill$\blacksquare$

\section[Smallest Singular Value of Regularized Design Matrix ]{SMALLEST SINGULAR VALUE OF REGULARIZED DESIGN MATRIX $V_t$} \label{apx:smallest_eigen}

In this section, we show that improved exploration of \alg provides persistently exciting inputs, which will be used to enable reaching a stabilizing neighborhood around the system parameters. In other words, we will lower bound the smallest eigenvalue of the regularized design matrix, $V_t$. The analysis generalizes the lower bound on smallest eigenvalue of the sample covariance matrix in Theorem 20 of \citep{cohen2019learning} for the general case of subgaussian noise. 

For the state $x_t$, and input $u_t$, we have:
\begin{align}\label{eq:equality_X_Z}
    x_t = A_* x_{t-1} + B_* u_{t-1} + w_{t-1}, ~~\textit{and} ~~~u_t = K(\ttt_{t-1})x_t + \nu_{t}
\end{align}
Let $\xi_t= z_t-\mathbb{E}\left[z_t| \fil_{t-1} \right]$. Using the equalities in (\ref{eq:equality_X_Z}), and the fact that $w_t$ and $\nu_{t}$ are $\fil_{t}$ measurable,  we write $\mathbb{E}\left[\xi_t\xi_t^\top | \fil_{t-1} \right]$ as follows.

\begin{align}
\mathbb{E}\left[\xi_{t} \xi_{t}^{\top} | \mathcal{F}_{t-1}\right] &=\left(\begin{array}{c}
I \\
K(\ttt_{t-1})
\end{array}\right) \mathbb{E}\left[w_{t} w_{t}^{\top} | \mathcal{F}_{t-1}\right]\left(\begin{array}{c}
I \\
K(\ttt_{t-1})
\end{array}\right)^{\top}+\left(\begin{array}{cc}
0 & 0 \\
0 & \mathbb{E}\left[\nu_{t} \nu_{t}^{\top} | \mathcal{F}_{t-1}\right]
\end{array}\right) \nonumber \\
%%%%%
&=\left(\begin{array}{c}
I \\
K(\ttt_{t-1})
\end{array}\right) (\bar{\sigma}_w^2I)\left(\begin{array}{c}
I \\
K(\ttt_{t-1})
\end{array}\right)^{\top}+\left(\begin{array}{cc}
0 & 0 \\
0 & \sigma_{\nu}^2 I
\end{array}\right) \\
%%%%%
% &=  \left(\begin{array}{cc}
% \bar{\sigma}_w^2 I & \bar{\sigma}_w^2 K(\ttt_{t-1})^{\top} \\
% \bar{\sigma}_w^2 K(\ttt_{t-1}) & \bar{\sigma}_w^2 K(\ttt_{t-1}) K(\ttt_{t-1})^{\top}+ \sigma_{\nu}^2 I
% \end{array}\right) \\
% %%%%%
& = \left(\begin{array}{cc}
\bar{\sigma}_w^2 I & \bar{\sigma}_w^2 K(\ttt_{t-1})^{\top} \\
\bar{\sigma}_w^2 K(\ttt_{t-1}) & \bar{\sigma}_w^2 K(\ttt_{t-1}) K(\ttt_{t-1})^{\top}+ 2 \kappa^2 \bar\sigma_w^2 I
\end{array}\right) \label{sigma_nu} \\
%%%%%
& \succeq  \bar{\sigma}_w^2\left(\begin{array}{cc}
I & K(\ttt_{t-1})^{\top} \\
 K(\ttt_{t-1}) &  2K(\ttt_{t-1}) K(\ttt_{t-1})^{\top}+ I/2
\end{array}\right) \label{kappaconstraint}\\
%%%%%
& =  \frac{\bar{\sigma}_w^2}{2}I+\bar{\sigma}_w^2\left(\begin{array}{cc}
\frac{1}{\sqrt{2}}I  \\
 \sqrt{2}K(\ttt_{t-1}) 
\end{array}\right) 
\left(\begin{array}{cc}
\frac{1}{\sqrt{2}}I  \\
 \sqrt{2}K(\ttt_{t-1}) 
\end{array}\right)^\top\\
%%%%%
& \succeq  \frac{\bar{\sigma}_w^2}{2}I
\label{expectedlowerbound}
\end{align}
where (\ref{sigma_nu}) follows from $\sigma_\nu^2 = 2 \kappa^2 \bar\sigma_w^2$ and  (\ref{kappaconstraint}) follows from the fact that $\kappa\geq 1$ and $\| K(\ttt_{t-1})  \| \leq \kappa $ for all $t$. Let $s_t = v^\top \xi_t$ for any unit vector $v \in \mathbb{R}^{n+d}$. (\ref{expectedlowerbound}) shows that that $\text{Var}\left[s_{t}| \mathcal{F}_{t-1}\right]\geq\frac{\bar{\sigma}_w^2}{2}$.

\begin{lemma} \label{lem:expected_stabil}
Suppose the system is stabilizable and we are in adaptive control with improved exploration phase of \alg. Denote $s_t = v^\top \xi_t$ where $v \in \mathbb{R}^{n+d}$ is any unit vector. Let $\bar\sigma_{\nu}:=  ((1+\kappa)^2 + 2\kappa^2)\sigma_w^2$. For a given positive $\temp$, let $E_t$ be an indicator random variable that equals $1$ if $s_t^2 >\temp$ and $0$ otherwise. Then for any positive $\temp$, and $\tempp$, such that $\temp\leq \tempp$, we have

\begin{align}\label{eq:expected}
   \mathbb{E}\left[E_t | \fil_{t-1} \right]\geq \frac{\frac{\bar{\sigma}_w^2}{2}-\temp-4\bar\sigma_{\nu}^2(1+\frac{\tempp}{2\bar\sigma_{\nu}^2} )\exp(\frac{-\tempp}{2\bar\sigma_{\nu}^2})}{\tempp}
\end{align}
\end{lemma}
Note that, for any $\bar\sigma_{\nu}\geq \bar{\sigma}_w$, there is a pair $(\temp,\tempp)$ such that the right hand side of (\ref{eq:expected}) is positive. 
%%%%%%%%%%%%%%%%%
\begin{proof}

Using the lower bound on the variance of $s_t$, we have,

\begin{align*}
    \frac{\bar{\sigma}_w^2}{2}&\leq \mathbb{E}\left[s_{t}^2\I(s_{t}^2< \temp ) | \mathcal{F}_{t-1}\right] + \mathbb{E}\left[s_{t}^2\I(s_{t}^2\geq \temp ) | \mathcal{F}_{t-1}\right]\\
    &\leq \temp + \mathbb{E}\left[s_{t}^2\I(s_{t}^2\geq \temp ) | \mathcal{F}_{t-1}\right]
\end{align*}

Now, deploying the fact that both $\nu_t$ and $w_t$, for any t, are sub-Gaussian given $\fil_{t-1}$, have that $\xi_{t}$ is also sub-Gaussian vector. Therefore, $s_t$ is a sub-Gaussian random variable with parameter $\bar\sigma_{\nu}$, where  $\bar\sigma_{\nu}:=  ((1+\kappa)^2 + 2\kappa^2)\sigma_w^2$.

% $s_t$ is a $\bar\sigma_{\nu}$-sub-Gaussian random variable, we have,

\begin{align}\label{eq:cdf1}
    \frac{\bar{\sigma}_w^2}{2}-\temp&\leq  \mathbb{E}\left[s_{t}^2\I(s_{t}^2\geq \temp ) | \mathcal{F}_{t-1}\right]\nonumber\\
    &=  \mathbb{E}\left[s_{t}^2\I(\tempp\geq s_{t}^2\geq \temp ) | \mathcal{F}_{t-1}\right]+\mathbb{E}\left[s_{t}^2\I(s_{t}^2\geq \tempp) | \mathcal{F}_{t-1}\right]
    % &=  \tempp\mathbb{E}\left[E_t | \fil_{t-1} \right]+\int_{|s_{t}|\geq \tempp}\Prob(|s_{t}|\geq s)ds\\
\end{align}
For the second term in the right hand side of the (\ref{eq:cdf1}), under the considerations of Fubini's and Radon–Nikodym theorems, we derive the following equality,
\begin{align*}
\int_{s^2\geq \tempp}\Prob(s_{t}^2\geq s^2| \mathcal{F}_{t-1})ds^2&=\int_{s^2\geq \tempp}\int_{s'^2\geq s^2}-\frac{d\Prob(s_{t}^2\geq s'^2| \mathcal{F}_{t-1})}{ds'^2}ds'^2ds^2\\
&=\int_{s'^2\geq \tempp}\int_{s'^2\geq s^2\geq \tempp}-\frac{d\Prob(s_{t}^2\geq s'^2| \mathcal{F}_{t-1})}{ds'^2}ds'^2ds^2\\
&=\int_{s'^2\geq \tempp}\int_{s'^2\geq s^2\geq \tempp}-\frac{d\Prob(s_{t}^2\geq s'^2| \mathcal{F}_{t-1})}{ds'^2}ds^2ds'^2\\
&=\int_{s'^2\geq \tempp}-\frac{d\Prob(s_{t}^2\geq s'^2| \mathcal{F}_{t-1})}{ds'^2}(s'^2-\tempp)ds'^2\\
&=\mathbb{E}\left[s_{t}^2\I(s_{t}^2\geq \tempp) | \mathcal{F}_{t-1}\right]-\tempp\int_{s'^2\geq \tempp}-\frac{d\Prob(s_{t}^2\geq s'^2| \mathcal{F}_{t-1})}{ds'^2} ds'^2\\
&=\mathbb{E}\left[s_{t}^2\I(s_{t}^2\geq \tempp) | \mathcal{F}_{t-1}\right]-\tempp~\Prob(s_{t}^2\geq \tempp| \mathcal{F}_{t-1}),
\end{align*}
resulting in the following equality,
\begin{align}\label{eq:derivsigma1}
\mathbb{E}\left[s_{t}^2\I(s_{t}^2\geq \tempp) | \mathcal{F}_{t-1}\right]=\int_{s^2\geq \tempp}\Prob(s_{t}^2\geq s^2| \mathcal{F}_{t-1})ds^2+\tempp~\Prob(s_{t}^2\geq \tempp| \mathcal{F}_{t-1}).
\end{align}
Using this equality, we extend the (\ref{eq:cdf1}) as follows,
\begin{align}\label{eq:derivsigma2}
    \frac{\bar{\sigma}_w^2}{2}-\temp&\leq  \mathbb{E}\left[s_{t}^2\I(\tempp\geq s_{t}^2\geq \temp ) | \mathcal{F}_{t-1}\right]+\int_{s^2\geq \tempp}\Prob(s_{t}^2\geq s^2| \mathcal{F}_{t-1})ds^2+\tempp~\Prob(s_{t}^2\geq \tempp| \mathcal{F}_{t-1})\nonumber\\
    &\leq  \tempp~\mathbb{E}\left[\I(\tempp\geq s_{t}^2\geq \temp ) | \mathcal{F}_{t-1}\right]+\int_{s^2\geq \tempp}\Prob(s_{t}^2\geq s^2| \mathcal{F}_{t-1})ds^2+\tempp~\Prob(s_{t}^2\geq \tempp| \mathcal{F}_{t-1})\nonumber\\
    &\leq  \tempp~\mathbb{E}\left[E_t | \fil_{t-1} \right]+\int_{s^2\geq \tempp}\Prob(s_{t}^2\geq s^2| \mathcal{F}_{t-1})ds^2+\tempp~\Prob(s_{t}^2\geq \tempp| \mathcal{F}_{t-1}).
\end{align}
Rearranging this inequality, we have,

\begin{align}\label{eq:cdf2}
   \mathbb{E}\left[E_t | \fil_{t-1} \right]&\geq \frac{\frac{\bar{\sigma}_w^2}{2}-\temp-\int_{s^2\geq \tempp}\Prob(s_{t}^2\geq s^2| \mathcal{F}_{t-1})ds^2-\tempp~\Prob(s_{t}^2\geq \tempp| \mathcal{F}_{t-1})}{\tempp}\nonumber\\
   &\geq \frac{\frac{\bar{\sigma}_w^2}{2}-\temp-2\int_{s^2\geq \tempp}\exp(\frac{-s^2}{2\bar\sigma_{\nu}^2})ds^2-2\tempp  \exp(\frac{-\tempp}{2\bar\sigma_{\nu}^2})}{\tempp}\nonumber\\
   &\geq \frac{\frac{\bar{\sigma}_w^2}{2}-\temp-4\bar\sigma_{\nu}^2\exp(\frac{-\tempp}{2\bar\sigma_{\nu}^2})-2\tempp  \exp(\frac{-\tempp}{2\bar\sigma_{\nu}^2})}{\tempp}\nonumber\\
   &= \frac{\frac{\bar{\sigma}_w^2}{2}-\temp-4\bar\sigma_{\nu}^2(1+\frac{\tempp}{2\bar\sigma_{\nu}^2} )\exp(\frac{-\tempp}{2\bar\sigma_{\nu}^2})}{\tempp}
\end{align}

The inequality in (\ref{eq:cdf2}) holds for any $\temp\leq\tempp$, therefore, the stated lower-bound on $\mathbb{E}\left[E_t | \fil_{t-1} \right]$ in the main statement holds.  

\end{proof}
For the choices of $\temp$ and $\tempp$ that makes right hand side of (\ref{eq:expected}) positive, let $c_p$ denote the right hand side of (\ref{eq:expected}), $c_p = \frac{\frac{\bar{\sigma}_w^2}{2}-\temp-4\bar\sigma_{\nu}^2(1+\frac{\tempp}{2\bar\sigma_{\nu}^2} )\exp(\frac{-\tempp}{2\bar\sigma_{\nu}^2})}{\tempp}$.

\begin{lemma}\label{lem:expected_stabil_bar}
Consider $\bar s_t = v^\top z_t$ where $v \in \mathbb{R}^{n+d}$ is any unit vector. Let $\bar E_t$ be an indicator random variable that equal $1$ if $\bar s_t^2 >\temp/4$ and $0$ otherwise. Then, there exist a positive pair $\temp$, and $\tempp$, and a constant $c_p>0$, such that  $\mathbb{E}\left[\bar E_t | \fil_{t-1} \right]\geq c_p' > 0 $.
\end{lemma}
\begin{proof}
Using the Lemma \ref{lem:expected_stabil}, we know that for $s_t = v^\top \xi_t$, we have $|s_t|\geq \tempsquare$ with a non-zero probability $c_p$. On the other hand, we have that, 
\begin{align*}
    \bar s_t = v^\top z_t = v^\top \xi_t + v^\top \mathbb{E}\left[z_t | \fil_{t-1} \right] = s_t + v^\top \mathbb{E}\left[z_t | \fil_{t-1} \right]
\end{align*}
Therefore, we have, $|\bar s_t| = \left|s_t+ v^\top \mathbb{E}\left[z_t | \fil_{t-1} \right]\right|$. Using this equality, if $\left|v^\top \mathbb{E}\left[z_t | \fil_{t-1} \right]\right|\leq \tempsquare/{2}$, since $|s_t|\geq \tempsquare$ with probability $c_p$, we have $|\bar s_t|\geq \tempsquare/{2}$ with probability $c_p$. 

In the following, we consider the case where  $\left|v^\top \mathbb{E}\left[z_t | \fil_{t-1} \right]\right|\geq \tempsquare/{2}$. For a constant $\tpppsquare$, using a similar derivation as in (\ref{eq:derivsigma1}) and (\ref{eq:derivsigma2}), we have

\begin{align*}
    \mathbb{E}\left[s_t^2| \fil_{t-1} \right] 
    % &= \mathbb{E}\left[s_t^2\I(s_t< 0)| \fil_{t-1} \right] + \mathbb{E}\left[s_t^2\I(s_t\geq 0)| \fil_{t-1} \right]\\
%%%%%
    &= \mathbb{E}\left[s_t^2\I(\tpppsquare<s_t< 0)| \fil_{t-1} \right] + \mathbb{E}\left[s_t^2\I(\tpppsquare>s_t> 0)| \fil_{t-1} \right]+\mathbb{E}\left[s_t^2\I(s_t^2\geq \tppp)| \fil_{t-1} \right]\\
%%%%%
    &= \mathbb{E}\left[s_t^2\I(\tpppsquare<s_t< 0)| \fil_{t-1} \right] + \mathbb{E}\left[s_t^2\I(\tpppsquare>s_t> 0)| \fil_{t-1} \right]+ 4\bar\sigma_{\nu}^2(1+\frac{\tempp}{2\bar\sigma_{\nu}^2} )\exp(\frac{-\tempp}{2\bar\sigma_{\nu}^2})
\end{align*}
Using the lower bound in the variance results in, 
\begin{align*}
        &\frac{\bar{\sigma}_w^2}{2}\leq  \mathbb{E}\left[s_t^2\I(\tpppsquare<s_t< 0)| \fil_{t-1} \right] + \mathbb{E}\left[s_t^2\I(\tpppsquare>s_t> 0)| \fil_{t-1} \right]+ 4\bar\sigma_{\nu}^2(1+\frac{\tppp}{2\bar\sigma_{\nu}^2} )\exp(\frac{-\tppp}{2\bar\sigma_{\nu}^2})\\
\end{align*}
Therefore, 
\begin{align}\label{eq:meanofs}
        &\frac{\bar{\sigma}_w^2}{2}-4\bar\sigma_{\nu}^2(1+\frac{\tppp}{2\bar\sigma_{\nu}^2} )\exp(\frac{-\tppp}{2\bar\sigma_{\nu}^2})\leq  \mathbb{E}\left[s_t^2\I(\tpppsquare<s_t< 0)| \fil_{t-1} \right] + \mathbb{E}\left[s_t^2\I(\tpppsquare>s_t> 0)| \fil_{t-1} \right]\nonumber\\
        &\quad\quad\quad\quad\quad\quad\quad\quad\quad\quad\quad\quad\quad= \tppp\left( \mathbb{E}\left[\frac{s_t^2}{\tppp}\I(-\tpppsquare<s_t< 0)| \fil_{t-1} \right] + \mathbb{E}\left[\frac{s_t^2}{\tppp}\I(\tpppsquare>s_t> 0)| \fil_{t-1} \right]\right)\nonumber\\
        &\quad\quad\quad\quad\quad\quad\quad\quad\quad\quad\quad\quad\quad\leq \tppp\left( \mathbb{E}\left[\frac{|s_t|}{\tpppsquare}\I(-\tpppsquare<s_t< 0)| \fil_{t-1} \right] + \mathbb{E}\left[\frac{s_t}{\tpppsquare}\I(\tpppsquare>s_t> 0)| \fil_{t-1} \right]\right)\nonumber\\
        % &\quad\quad\quad\quad\quad\quad\quad\quad\quad\quad\quad\quad\quad\leq \tppp\left( \mathbb{E}\left[\frac{|s_t|}{\tpppsquare}\I(s_t< 0)| \fil_{t-1} \right] + \mathbb{E}\left[\frac{s_t}{\tpppsquare}\I(s_t> 0)| \fil_{t-1} \right]\right)
\end{align}
Note the for a large enough $\tpppsquare$, the second term on the left hand side vanishes. Since we have $\mathbb{E}\left[s_t| \fil_{t-1} \right]= 0$, we write the following, to further analyze the right hand side of (\ref{eq:meanofs}),
\begin{align*}
 &   \mathbb{E}\left[s_t| \fil_{t-1} \right] = \mathbb{E}\left[s_t\I(s_t< 0)| \fil_{t-1} \right] + \mathbb{E}\left[s_t\I(s_t> 0)| \fil_{t-1} \right] = 0 \\
& \rightarrow    \mathbb{E}\left[|s_t|\I(s_t< 0)| \fil_{t-1} \right] = \mathbb{E}\left[s_t\I(s_t> 0)| \fil_{t-1} \right]
\end{align*}
Note that, since $s_t$ is sub-Gaussian variable, and has bounded away from zero variance, we have $\mathbb{E}\left[\I(s_t< 0)| \fil_{t-1} \right] + \mathbb{E}\left[\I(s_t> 0)| \fil_{t-1} \right]$ is bounded away from zero. We write this equality as follows:
\begin{align*}
  \mathbb{E}\left[|s_t|\I(-\tpppsquare<s_t< 0)| \fil_{t-1} \right]& + \mathbb{E}\left[|s_t|\I(s_t\leq-\tpppsquare)| \fil_{t-1} \right] \\
  &= \mathbb{E}\left[s_t\I(\tpppsquare>s_t> 0)| \fil_{t-1} \right]+ \mathbb{E}\left[s_t\I(s_t\geq \tpppsquare)| \fil_{t-1} \right]\\
\end{align*}
With rearranging this equality, and upper bounding the first term on the left hand side, we have
\begin{align}\label{eq:inequality1}
  \mathbb{E}\left[|s_t|\I(-\tpppsquare<s_t< 0)| \fil_{t-1} \right]&\leq \mathbb{E}\left[s_t\I(\tpppsquare>s_t> 0)| \fil_{t-1} \right]+ \mathbb{E}\left[s_t\I(s_t\geq \tpppsquare)| \fil_{t-1} \right]\nonumber\\
    &\leq \mathbb{E}\left[s_t\I(\tpppsquare>s_t> 0)| \fil_{t-1} \right]+ \bar\sigma_{\nu}^2\exp(\frac{-\tppp}{2\bar\sigma_{\nu}^2})
\end{align}
similarly we have
\begin{align}\label{eq:inequality2}
  \mathbb{E}\left[s_t\I(\tpppsquare>s_t> 0)| \fil_{t-1} \right]&\leq \mathbb{E}\left[|s_t|\I(-\tpppsquare<s_t< 0)| \fil_{t-1} \right]+ \bar\sigma_{\nu}^2\exp(\frac{-\tppp}{2\bar\sigma_{\nu}^2})
\end{align}

Using the inequality (\ref{eq:inequality1}) on the right hand side of (\ref{eq:meanofs}), we have

\begin{align*}
    \frac{\frac{\bar{\sigma}_w^2}{2}-4\bar\sigma_{\nu}^2(1+\frac{\tppp}{2\bar\sigma_{\nu}^2} )\exp(\frac{-\tppp}{2\bar\sigma_{\nu}^2})}{\tppp}&\leq \mathbb{E}\left[\frac{|s_t|}{\tpppsquare}\I(-\tpppsquare<s_t< 0)| \fil_{t-1} \right] + \mathbb{E}\left[\frac{s_t}{\tpppsquare}\I(\tpppsquare>s_t> 0)| \fil_{t-1} \right]\\
    & \leq 2\mathbb{E}\left[\frac{s_t}{\tpppsquare}\I(\tpppsquare>s_t> 0)| \fil_{t-1} \right] + \bar\sigma_{\nu}^2\exp(\frac{-\tppp}{2\bar\sigma_{\nu}^2})\\
    & \leq 2\mathbb{E}\left[\I(\tpppsquare>s_t> 0)| \fil_{t-1} \right] + \bar\sigma_{\nu}^2\exp(\frac{-\tppp}{2\bar\sigma_{\nu}^2})\\
    & \leq 2\mathbb{E}\left[\I(s_t> 0)| \fil_{t-1} \right] + \bar\sigma_{\nu}^2\exp(\frac{-\tppp}{2\bar\sigma_{\nu}^2})
\end{align*}
Similarly, using (\ref{eq:inequality1}) on the right hand side of (\ref{eq:meanofs}) we have 
\begin{align*}
    \frac{\frac{\bar{\sigma}_w^2}{2}-4\bar\sigma_{\nu}^2(1+\frac{\tppp}{2\bar\sigma_{\nu}^2} )\exp(\frac{-\tppp}{2\bar\sigma_{\nu}^2})}{\tppp}&\leq \mathbb{E}\left[\frac{|s_t|}{\tpppsquare}\I(-\tpppsquare<s_t< 0)| \fil_{t-1} \right] + \mathbb{E}\left[\frac{s_t}{\tpppsquare}\I(\tpppsquare>s_t> 0)| \fil_{t-1} \right]\\
    & \leq 2\mathbb{E}\left[\I(s_t< 0)| \fil_{t-1} \right] + \bar\sigma_{\nu}^2\exp(\frac{-\tppp}{2\bar\sigma_{\nu}^2})
\end{align*}

Therefore, it results in the two following lower bounds, 
\begin{align}\label{eq:symmetric}
&    \mathbb{E}\left[\I(s_t< 0)| \fil_{t-1} \right] \geq \frac{\frac{\bar{\sigma}_w^2}{2}-4\bar\sigma_{\nu}^2(1+\frac{\tppp}{2\bar\sigma_{\nu}^2} )\exp(\frac{-\tppp}{2\bar\sigma_{\nu}^2})}{2\tppp}-0.5\bar\sigma_{\nu}^2\exp(\frac{-\tppp}{2\bar\sigma_{\nu}^2})\nonumber\\
&    \mathbb{E}\left[\I(s_t> 0)| \fil_{t-1} \right] \geq \frac{\frac{\bar{\sigma}_w^2}{2}-4\bar\sigma_{\nu}^2(1+\frac{\tppp}{2\bar\sigma_{\nu}^2} )\exp(\frac{-\tppp}{2\bar\sigma_{\nu}^2})}{2\tppp}-0.5 \bar\sigma_{\nu}^2\exp(\frac{-\tppp}{2\bar\sigma_{\nu}^2})
\end{align}
Choosing $\tpppsquare$ sufficiently large results in the right hand sides in inequalities (\ref{eq:symmetric}) to be positive and bounded away form zero. Let $c_p''>0$ denote the right hand sides in the (\ref{eq:symmetric}). We use this fact to analyze $\bar s_t$ when $\left|v^\top \mathbb{E}\left[z_t | \fil_{t-1} \right]\right|\geq \tempsquare/{2}$. 

When $v^\top \mathbb{E}\left[z_t | \fil_{t-1} \right]\geq \tempsquare/{2}$, since probability $c_p''$, $s_t$ is  positive, therefore, $|\bar s_t|\geq\tempsquare/{2}$ with  probability $c_p''$. When $v^\top \mathbb{E}\left[z_t | \fil_{t-1} \right]\leq -\tempsquare/{2}$, since probability $c_p''$, $s_t$ is negative, therefore, $|\bar s_t|\geq\tempsquare/{2}$ with  probability $c_p''$. 

Therefore, overall, with probability $c_p':=\min \lbrace c_p,c_p''\rbrace$,  we have that $|\bar s_t|\geq\tempsquare/{2}$, resulting in the statement of the lemma. 

\end{proof}
%%%%

\begin{lemma}[Precise version of Lemma \ref{lem:smallest_eigen}, Persistence of Excitation During the Extra Exploration] \label{lem:smallest_eigen_apx}
If the duration of the adaptive control with improved exploration ${\Tw}\geq \frac{6n}{c_p'}\log(12/\delta)$, then with probability at least $1-\delta$, \alg has 
\begin{equation*}
    \lambda_{\min}(V_{\Tw}) \geq \sigma_{\star}^2\Tw, 
\end{equation*}
for $\sigma_{\star}^2=\frac{ c_p'\temp}{16}$.
\end{lemma}

\begin{proof}
Let $U_{t}= \bar E_{t}-\mathbb{E}_{t}\left[ \bar E_{t} | \mathcal{F}_{t-1}\right] .$ Then $U_{t}$ is a martingale difference sequence with $\left|U_{t}\right| \leq 1$. Applying Azuma's inequality, we have that with probability at least $1-\delta$
\[
\sum_{t=1}^{{\Tw}} U_{t} \geq-\sqrt{2 {\Tw} \log \frac{1}{\delta}}
\]

Using the Lemma \ref{lem:expected_stabil_bar}, we have
\begin{align*}
    \sum_t^{{\Tw}} \bar E_{t} &\geq  \sum_t^{{\Tw}}\mathbb{E}_{t}\left[\bar E_{t} | \mathcal{F}_{t-n}\right]-\sqrt{2 \Tw \log \frac{1}{\delta}}\\
    &\geq  c_p'{\Tw}-\sqrt{2 \Tw \log \frac{1}{\delta}}
\end{align*}
where for $\Tw \geq 8\log(1/\delta) / c_p'^2$, we have $\sum_t^{{\Tw}} \bar E_{t}\geq  \frac{c_p'}{2}{\Tw}$. Now, for any unit vector $v$, define $\bar s_t = v^\top z_t$, therefore from the definition of $\bar E_t$ we have, 
\begin{align*}
    v^\top V_{{\Tw}}v =  \sum_t^{{\Tw}} \bar s_t^2\geq \bar E_t\temp/4 \geq \frac{c_p'\temp}{8} {\Tw}
\end{align*}

This inequality hold for a given $v$. In the following we show a similar inequality for all $v$ together. Similar to the Theorem 20 in \citep{cohen2019learning}, consider a 1/4-net of $\mathbb{S}^{n+d-1}$, $\N(1/4)$ and set $M_{{\Tw}}:=\lbrace V_{{\Tw}}^{-1/2}v/\|V_{{\Tw}}^{-1/2}v\|:v\in\N(1/4)\rbrace$. These two sets have at most $12^{n+d-1}$ members. Using union bound over members of this set, when $\Tw \geq \frac{20}{c_p'^2}((n+d) + \log(1/\delta))$, we have that $v^\top V_{{\Tw}}v\geq \frac{c_p'\temp}{8} {\Tw}$ for all $v\in M_{\Tw}$ with a probability at least $1-\delta$. Using the definition of members in $M_{\Tw}$, for each $v\in\N(1/4)$, we have $v^\top V_{{\Tw}}^{-1}v\leq \frac{8}{\Tw c_p'\temp}$. Let $v_n$ denote the eigenvector of the largest eigenvalue of $V_{{\Tw}}^{-1}$, and a vector $v'\in \N(1/4)$ such that $\|v_n-v'\|\leq 1/4$. Then we have
\begin{align*}
    \|V_{{\Tw}}^{-1}\| &= v_n^\top V_{{\Tw}}^{-1}v_n=v'^\top V_{{\Tw}}^{-1}v'+ (v_n-v')^\top V_{{\Tw}}^{-1}(z_n+v')\\
    &\leq \frac{8}{\Tw c_p'\temp}+ \|v_n-v'\|\|V_{{\Tw}}^{-1}\|\|z_n+v'\|\leq \frac{8}{\Tw c_p'\temp}+ \|V_{{\Tw}}^{-1}\|/2
\end{align*}
Rearranging, we get that $\|V_{{\Tw}}^{-1}\|\leq \frac{16}{\Tw c_p'\temp}$. Therefore, the advertised bound holds for $\Tw \geq \frac{20}{c_p'^2}((n+d) + \log(1/\delta))$ with probability at least $1-\delta$.
\end{proof}

\section{SYSTEM IDENTIFICATION \& CONFIDENCE SET CONSTRUCTION } \label{apx:2norm}
To have completeness, for the proof of Lemma \ref{lem:2norm_bound} we first provide the proof for confidence set construction borrowed from \citet{abbasi2011lqr}, since Lemma \ref{lem:2norm_bound} builds upon this confidence set construction. First let 
\begin{equation} \label{kappa_e}
    \kappa_e = \left(\frac{\sigma_w}{\sigma_\star} \sqrt{ n  (n+d) \log \left(1 + \frac{cT (1+\kappa^2)(n+d)^{2(n+d)} }{\lambda(n+d)}\right) + 2n \log  \frac{1}{\delta}} + \sqrt{\lambda}S \right)
\end{equation}

\begin{proof}
Define $\tts^{\top}=\left[A, B\right]$ and $z_t = \left[x_t^\top u_t^\top \right]^\top$. The system in \eqref{output} can be characterized equivalently as 
\begin{equation*}
x_{t+1}=\tts^{\top} z_{t}+w_{t}
\end{equation*}

Given a single input-output trajectory $\{x_t, u_t \}^{T}_{t=1}$, one can rewrite the input-output relationship as, 
\begin{equation} \label{newform}
    X_T = Z_T \tts +  W_T
\end{equation}
for 
\begin{align}
  & X_T = \left[ \begin{array}{c}{x_1^{\top}} \\ {x_2^{\top}} \\ {\vdots} \\ {x_{T-1}^{\top}}  \\ {x_T^{\top}} \end{array}\right] \in \mathbb{R}^{T \times n}
  \quad Z_T = \left[ \begin{array}{c}{z_1^{\top}} \\ {z_2^{\top}} \\ {\vdots} \\ {z_{T-1}^{\top}}  \\ {z_T^{\top}} \end{array}\right] \in \mathbb{R}^{T \times (n+d)} \quad 
  W_T = \left[ \begin{array}{c}{w_1^{\top}} \\ {w_2^{\top}} \\ {\vdots} \\ {w_{T-1}^{\top}}  \\ {w_T^{\top}} \end{array}\right] \in \mathbb{R}^{T \times n}.
\end{align}
Then, we estimate $\tts$ by solving the following least square problem,
\begin{align*} 
    \tth_T &= \arg \min_X ||X_T - Z_T X||_F^2 + \lambda ||X||^2_F  \\
    &= (Z_T^{\top}Z_T + \lambda I)^{-1}Z_T^{\top}X_T \\
    &= (Z_T^{\top}Z_T + \lambda I)^{-1}Z_T^{\top}W_T + (Z_T^{\top}Z_T + \lambda I)^{-1}Z_T^{\top}Z_T\tts + \lambda (Z_T^{\top}Z_T + \lambda I)^{-1}\tts - \lambda (Z_T^{\top}Z_T + \lambda I)^{-1}\tts  \\ 
    &= (Z_T^{\top}Z_T + \lambda I)^{-1}Z_T^{\top}W_T + \tts - \lambda (Z_T^{\top}Z_T + \lambda I)^{-1}\tts 
\end{align*}

The confidence set is obtained using the expression for $\tth_T$ and subgaussianity of the $w_t$,  

\begin{align}
    |\Tr((\tth_T-\tts)^{\top}X)| &= |\Tr(W_T^{\top}Z_T (Z_T^{\top}Z_T + \lambda I)^{-1} X ) - \lambda \Tr( \tts^{\top}(Z_T^{\top}Z_T + \lambda I)^{-1} X)| \nonumber \\
    &\leq |\Tr(W_T^{\top}Z_T (Z_T^{\top}Z_T + \lambda I)^{-1} X )| + \lambda |\Tr( \tts^{\top}(Z_T^{\top}Z_T + \lambda I)^{-1} X)| \nonumber \\
    &\leq \sqrt{\Tr(X^{\top}(Z_T^{\top}Z_T + \lambda I)^{-1}X)\Tr(W_T^{\top}Z_T (Z_T^{\top}Z_T + \lambda I)^{-1} Z_T^{\top}W_T)} \nonumber \\
    &+ \lambda \sqrt{\Tr(X^{\top}(Z_T^{\top}Z_T \!+\! \lambda I)^{-1} X)\Tr(\tts^{\top}(Z_T^{\top}Z_T \!+\! \lambda I)^{-1} \tts)}, \label{CS_trace} \\
    &\hspace{-9em}= \sqrt{\Tr(X^{\top}(Z_T^{\top}Z_T \!+\! \lambda I)^{-1}X)} \bigg[ \sqrt{\Tr(W_T^{\top}Z_T (Z_T^{\top}Z_T \!+\! \lambda I)^{-1} Z_T^{\top}W_T)} \!+\! \lambda \sqrt{\Tr(\tts^{\top}(Z_T^{\top}Z_T \!+\! \lambda I)^{-1} \tts)} \bigg] \nonumber
\end{align}
where (\ref{CS_trace}) follows from $|\Tr(A^{\top}BC)| \leq \sqrt{\Tr(A^{\top}BA)\Tr(C^{\top}BC) }$ for square positive definite B due to Cauchy Schwarz (weighted inner-product).  For $X = (Z_T^{\top}Z_T + \lambda I)(\tth_T-\tts)$, we get
\begin{align*}
  \sqrt{\Tr((\tth_T\!-\!\tts)^\top(Z_T^{\top}Z_T \!+\! \lambda I)(\tth_T\!-\!\tts))}\! &\leq\! \sqrt{\Tr(W_T^{\top}Z_T (Z_T^{\top}Z_T \!+\! \lambda I)^{-1} Z_T^{\top}W_T)} \!+\! \sqrt{\lambda} \sqrt{\Tr(\tts^{\top}\tts)} 
\end{align*}

Let $\mathcal{S}_T = Z_T^{\top}W_T \in \mathbb{R}^{(n+d) \times n}$ and $s_i$ denote the columns of it. Also, let $V_T = (Z_T^{\top}Z_T + \lambda I)$. Thus, 

\begin{equation}
    \Tr(W_T^{\top}Z_T (Z_T^{\top}Z_T + \lambda I)^{-1} Z_T^{\top}W_T) = \Tr(\mathcal{S}_T^\top V_T^{-1} \mathcal{S}_T) = \sum_{i=1}^n s_i^\top V_T^{-1} s_i = \sum_{i=1}^n \| s_i \|^2_{V_T^{-1}}.
\end{equation}

Notice that $s_i = \sum_{j=1}^T w_{j,i} z_j $ where $w_{j,i}$ is the $i$'th element of $w_j$. From Assumption \ref{general_noise}, we have that $w_{j,i}$ is $\sigma_w$-subgaussian, thus we can use Theorem \ref{selfnormalized} to show that, 
\begin{equation}
    \Tr(W_T^{\top}Z_T (Z_T^{\top}Z_T + \lambda I)^{-1} Z_T^{\top}W_T) \leq 2 n \sigma_w^{2} \log \left(\frac{\operatorname{det}\left(V_{T}\right)^{1 / 2} \operatorname{det}(\lambda I)^{-1 / 2}}{\delta}\right).
\end{equation}
with probability $1-\delta$. From Assumption \ref{parameterassump_stabilizability}, we also have that $\sqrt{\Tr(\tts^{\top}\tts)} \leq S$. Combining these gives the self-normalized confidence set or the model estimate:
\begin{equation}
    \Tr((\tth_T-\tts)^\top V_{T}(\tth_T-\tts)) \leq \left(\sigma_w \sqrt{2 n \log \left(\frac{\operatorname{det}\left(V_{T}\right)^{1 / 2} \operatorname{det}(\lambda I)^{-1 / 2}}{\delta}\right)} + \sqrt{\lambda}S \right)^2 . 
\end{equation}

Notice that we have $\Tr((\tth_T-\tts)^\top V_{T}(\tth_T-\tts)) \geq \lambda_{\text{min}}(V_T) \|\tth_T-\tts \|_F^2 $. Therefore, 

\begin{equation}
    \|\tth_T-\tts\|_2 \leq \frac{1}{\sqrt{\lambda_{\text{min}}(V_T)}}\left(\sigma_w \sqrt{2 n \log \left(\frac{\operatorname{det}\left(V_{T}\right)^{1 / 2} \operatorname{det}(\lambda I)^{-1 / 2}}{\delta}\right)} + \sqrt{\lambda}S \right)
\end{equation}

To complete the proof, we need a lower bound on $\lambda_{\text{min}}(V_{\Tw})$. Using Lemma \ref{lem:smallest_eigen}, we obtain the following with probability at least $1-2\delta$: 

\begin{equation*}
    \|\tth_{\Tw}-\tts\|_2 \leq \frac{\beta_t(\delta)}{\sigma_\star \sqrt{\Tw}}.
\end{equation*}
From Lemma \ref{lem:bounded_state}, for $t\leq \Tw$, we have that $\|z_t \| \leq c (n+d)^{n+d}$ with probability at least $1-2\delta$, for some constant $c$. Combining this with Lemma \ref{upperboundlemma},
\begin{equation} \label{T_s_define}
    \|\tth_{\Tw}-\tts\|_2 \leq  \frac{\kappa_e}{\sqrt{\Tw}}.
\end{equation}

\end{proof}

\section{BOUNDEDNESS OF STATES}\label{apx:boundedness}

In this section, we will provide the proof of Lemma \ref{lem:bounded_state}, \textit{i.e.} bounds on states for the adaptive control with improved exploration and stabilizing adaptive control phases. First define the following. 

Let
\begin{equation} \label{exploration_duration}
    \Tw = \frac{\kappa_e^2}{\min\{ \bar{\sigma}_w^2nD/C_0,\epsilon^2\}}    
\end{equation}
such that for $T > \Tw$, we have $\|\tth_{T}-\tts\|_2 \leq \min\{ \sqrt{\bar{\sigma}_w^2nD/C_0},\epsilon\}$ with probability at least $1-2\delta$. Notice that due to Lemma \ref{stabilityofoptimisticcontroller}  and as shown in the following,
% for controllability
these guarantee the stability of the closed-loop dynamics for deploying optimistic controller for the remaining part of \alg. 

% \subsection{Stabilizable System}\label{apx:state_bound_stabilizable}

Choose an error probability, $\delta>0$. Consider the following events, in the probability space $\Omega$: 
\begin{itemize}
    \item The event that the confidence sets hold for $s=0, \ldots, T,$ \[\mathcal{E}_{t}=\left\{\omega \in \Omega: \forall s \leq T, \quad \Theta_{*} \in \mathcal{C}_{s}(\delta)\right\}\]
    \item The event that the state vector stays ``small'' for $s=0, \ldots, \Tw,$
    \[\mathcal{F}^{[s]}_{t}=\left\{\omega \in \Omega: \forall s \leq \Tw, \quad\left\|x_{s}\right\| \leq \bar{\alpha}_{t}\right\} \]
\end{itemize}
where
\begin{equation*}
\bar{\alpha}_{t} =\frac{18\kappa^3}{\gamma(8\kappa-1)} \bar{\eta}^{n+d}\left[G Z_{t}^{\frac{n+d}{n+d+1}} \beta_{t}(\delta)^{\frac{1}{2(n+d+1)}}  + (\|B_* \| \sigma_\nu + \sigma_w) \sqrt{2n \log \frac{nt}{\delta}} \right],
\end{equation*}
for
\[\begin{aligned} 
\bar{\eta} &\geq \sup _{\Theta \in \mathcal{S}}\left\|A_{*}+B_{*} K(\Theta)\right\|, \qquad \qquad Z_{T} =\max _{1 \leq t \leq T}\left\|z_{t}\right\| \\ 
G &=2\left(\frac{2 S(n+d)^{n+d+1 / 2}}{\sqrt{U}}\right)^{1 /(n+d+1)}, \quad U=\frac{U_{0}}{H}, \quad
U_{0} =\frac{1}{16^{n+d-2}\max\left(1, \enskip S^{2(n+d-2)}\right)} \end{aligned}\]
and $H$ is any number satisfying 
\begin{equation*}
H>\max\left(16, \enskip \frac{4 S^{2} M^{2}}{(n+d) U_{0}}\right), \enskip \text{where} \quad
M=\sup _{Y \geq 1} \frac{\left(\sigma_w \sqrt{n(n+d) \log \left(\frac{1+T Y / \lambda}{\delta}\right)}+\lambda^{1 / 2} S\right)}{Y}.
\end{equation*}

Notice that $\mathcal{E}_1 \supseteq \mathcal{E}_2 \supseteq \ldots \supseteq \mathcal{E}_T$ and $\mathcal{F}^{[s]}_1 \supseteq \mathcal{F}^{[s]}_2 \supseteq \ldots \supseteq \mathcal{F}^{[s]}_{T_s}$. This means considering the probability of last event is sufficient in lower bounding all event happening simultaneously. In \citet{abbasi2011lqr}, an argument regarding projection onto subspaces is constructed to show that the norm of the state is well-controlled except $n+d$ times at most in any horizon $T$. The set of time steps that is not well-controlled are denoted as $\mathcal{T}_{t}$. The given lemma shows how well controlled $\|(\tts - \tth_{t})^{\top} z_{t}\|$ is  besides $\mathcal{T}_{t}$.

\begin{lemma} [\citet{abbasi2011lqr}] \label{wellcontrol}
We have that for any $0 \leq t \leq T$,
\begin{equation*}
\max _{s \leq t, s \notin T_{t}}\left\|(\tts - \tth_{s})^{\top} z_{s}\right\| \leq G Z_{t}^{\frac{n+d}{n+d+1}} \beta_{t}(\delta / 4)^{\frac{1}{2(n+d+1)}}.
\end{equation*}
\end{lemma}

Notice that Lemma \ref{wellcontrol} does not depend on controllability or the stabilizability of the system. Thus, we will use Lemma \ref{wellcontrol} for $t \leq \Tw$ for the adaptive control with improved exploration phase of \alg. Then we consider the effect of stabilizing controllers for the remaining time steps. 

\subsection{State Bound for the Adaptive Control with Improved Exploration Phase}
One can write the state update as
\begin{equation*}
x_{t+1}=\Gamma_{t} x_{t}+r_{t}
\end{equation*}
where
\begin{equation} \label{state_update}
\Gamma_{t}=\left\{\begin{array}{ll}{\tilde{A}_{t-1}+\tilde{B}_{t-1} K(\tilde{\Theta}_{t-1})} & {t \notin \mathcal{T}_{T}} \\ {A_{*}+B_{*} K(\tilde{\Theta}_{t-1})} & {t \in \mathcal{T}_{T}}\end{array}\right. \quad \text{and} \quad r_{t}=\left\{\begin{array}{ll}{(\tts - \ttt_{t-1})^{\top} z_{t}+B_*\nu_t+w_{t}} & {t \notin \mathcal{T}_{T}} \\ {B_*\nu_t + w_{t}} & {t \in \mathcal{T}_{T}}\end{array}\right.
\end{equation} 

Thus, using the fact that $x_0 = 0$, we can obtain the following roll out for $x_t$,

\begin{align} 
x_{t} &=\Gamma_{t-1} x_{t-1}+r_{t-1}=\Gamma_{t-1}\left(\Gamma_{t-2} x_{t-2}+r_{t-2}\right)+r_{t} \nonumber \\ &=\Gamma_{t-1} \Gamma_{t-2} \Gamma_{t-3} x_{t-3}+\Gamma_{t-1} \Gamma_{t-2} r_{t-2}+\Gamma_{t-1} r_{t-1}+r_{t} \nonumber \\ &= \Gamma_{t-1} \Gamma_{t-2} \ldots \Gamma_{t-(t-1)} r_{1} + \cdots + \Gamma_{t-1} \Gamma_{t-2} r_{t-2} + \Gamma_{t-1} r_{t-1} + r_{t} \nonumber \\ &=\sum_{k=1}^{t}\left(\prod_{s=k}^{t-1} \Gamma_{s}\right) r_{k} \label{rollout} 
\end{align}

Recall that the controller is optimistically designed from set of parameters are $(\kappa, \gamma)$-strongly stabilizable by their optimal controllers. Therefore, we have
\begin{equation}
  1-\gamma \geq \max _{t \leq T} \rho\left(\tilde{A}_{t}+\tilde{B}_{t} K(\tilde{\Theta}_{t}) \right)  . \label{etabar}
\end{equation}

Therefore, multiplication of closed-loop system matrices, $\tilde{A}_{t}+\tilde{B}_{t} K(\tilde{\Theta}_{t})$, is not guaranteed to be contractive. In \citet{abbasi2011lqr}, the authors assume these matrices are contractive under controllability assumption. In order to bound the state similarly, we need to satisfy that the epochs that we use a particular optimistic controller is long enough that the state doesn't scale too badly during the exploration and produces bounded state. Thus, by choosing $H_0 = 2\gamma^{-1} \log(2\kappa\sqrt{2})$ and adopting Lemma 39 of \citet{cassel2020logarithmic}, we have that \begin{equation}
    \|x_t\| \leq \frac{18\kappa^3 \bar{\eta}^{n+d} }{\gamma(8\kappa-1)} \left( \max _{1 \leq k \leq t}\left\|r_{k}\right\| \right)
    % \left[G Z_{t}^{\frac{n+d}{n+d+1}} \beta_{t}(\delta)^{\frac{1}{2(n+d+1)}}  + (\|B_* \| \sigma_\nu + \sigma_w) \sqrt{2n \log \frac{nt}{\delta}} \right]
\end{equation}

Furthermore, we have that $\left\|r_{k}\right\| \leq\left\|(\tts - \ttt_{k-1})^{\top} z_{k}\right\|+\left\|B_* \nu_k + w_{k}\right\|$ when $k \notin \mathcal{T}_{T},$ and $\left\|r_{k}\right\|=\left\|B_*\nu_k + w_{k}\right\|,$ otherwise. Hence, 

\[
\max _{k\leq t}\left\|r_{k}\right\| \leq \max _{k\leq t, k \notin \mathcal{T}_{t}}\left\|(\tts - \ttt_{k-1})^{\top} z_{k}\right\|+\max _{k\leq t}\left\|B_*\nu_k + w_{k}\right\|
\]

The first term is bounded by the Lemma \ref{wellcontrol}. The second term involves summation of independent $\|B_* \| \sigma_\nu $ and $\sigma_w$ subgaussian vectors. Using Lemma \ref{subgauss lemma} with a union bound argument, for all $k\leq t$, $\left\|B_*\nu_k + w_{k}\right\| \leq (\|B_*\| \sigma_\nu + \sigma_w) \sqrt{2n \log \frac{nt}{\delta}}$ with probability at least $1-\delta$. Therefore, on the event of $\mathcal{E}$, 

\begin{equation}
      \|x_t\| \leq \frac{18\kappa^3 \bar{\eta}^{n+d} }{\gamma(8\kappa-1)} \left[G Z_{t}^{\frac{n+d}{n+d+1}} \beta_{t}(\delta)^{\frac{1}{2(n+d+1)}}  + (\|B_* \| \sigma_\nu + \sigma_w) \sqrt{2n \log \frac{nt}{\delta}} \right]
\end{equation}
for $t \leq \Tw$. Using union bound, we can deduce that $\mathcal{E}_T \cap \mathcal{F}^{[s]}_{T_s}$ holds with probability at least $1-2\delta$. Notice that this bound depends on $Z_t$ and $\beta_t(\delta)$ which in turn depends on $x_t$. Using Lemma 5 of \cite{abbasi2011lqr}, one can obtain the following bound 
\begin{equation}\label{stateboundexplore}
\|x_t \| \leq c' (n+d)^{n+d} .
\end{equation}
for some large enough constant $c'$. The adaptive control with improved exploration phase of \alg has this exponentially dimension dependent state bound for all $t\leq \Tw$. In the following section, we show that during the stabilizing adaptive control phase, the bound on state has a polynomial dependency on the dimensions.

\subsection{State Bound in Stabilizing Adaptive Control phase}
In the stabilizing adaptive control phase, \alg stops using the additive isotropic exploration component $\nu_t$, the state follows the dynamics of 
\begin{equation}
    x_{t+1} = (A_{*}+B_{*} K(\tilde{\Theta}_{t-1}) ) x_t + w_t
\end{equation}
Denote $\mathbf{M_t} = A_{*}+B_{*} K(\tilde{\Theta}_{t-1}) $ as the closed loop dynamics of the system. From the choice of $\Tw$ for the stabilizable systems, we have that $\mathbf{M_t} $ is $(\kappa \sqrt{2}, \gamma/2)$-strongly stable. Thus, we have $\rho(\mathbf{M_t}) \leq 1- \gamma/2$ for all $t > T_s$ and $\|H_t\| \|H_t^{-1}\|\leq \kappa \sqrt{2}$ for $H_t \succ 0$, such that $\|L_t\| \leq 1- \gamma/2$ for $\mathbf{M_t} = H_t L_t H_t^{-1}$. Then for $T>t>\Tw$, if the same policy, $\mathbf{M}$ is applied starting from state $x_{\Tw}$, we have
\begin{align}
\| x_t \| &= \bigg \|\prod_{i=\Tw+1}^{t} \!\!\!\mathbf{M} x_{\Tw} +  \sum_{i=\Tw+1}^t \left(\prod_{s=i}^{t-1} \mathbf{M}\right) w_i  \bigg\| \\
&\leq \kappa\sqrt{2} (1-\gamma/2)^{t-\Tw} \|x_{\Tw}\| + \max_{\Tw < i \leq T}\left\| w_i  \right \| \left( \sum_{i=\Tw+1}^t \kappa\sqrt{2} (1-\gamma/2)^{t-i+1}  \right) \\
&\leq  \kappa\sqrt{2} (1-\gamma/2)^{t-\Tw} \| x_{\Tw}\| + \frac{2\kappa\sigma_w\sqrt{2}}{\gamma}  \sqrt{2n\log(n(t-\Tw)/\delta)} \label{policy_change_stabil}
\end{align}
Note that $H_0 = 2\gamma^{-1} \log(2\kappa\sqrt{2})$. This gives that $\kappa\sqrt{2}(1-\gamma/2)^{H_0} \leq 1/2$. Therefore, at the end of each controller period the effect of previous state is halved. Using this fact, at the $i$th policy change after $\Tw$, we get 
\begin{align*}
    \|x_{t_i}\| &\leq 2^{-i} \|x_{\Tw}\| + \sum_{j=0}^{i-1} 2^{-j} \frac{2\kappa\sigma_w\sqrt{2}}{\gamma}  \sqrt{2n\log(n(t-\Tw)/\delta)} \\
    &\leq 2^{-i} \|x_{\Tw}\| + \frac{4\kappa\sigma_w\sqrt{2}}{\gamma}  \sqrt{2n\log(n(t-\Tw)/\delta)}
\end{align*}
For all $i> (n+d)\log(n+d) - \log(\frac{2\kappa\sigma_w\sqrt{2}}{\gamma}  \sqrt{2n\log(n(t-\Tw)/\delta)} )$, at policy change $i$, we get
\begin{align*}
    \|x_{t_i}\| \leq \frac{6\kappa\sigma_w\sqrt{2}}{\gamma}  \sqrt{2n\log(n(t-\Tw)/\delta)}.
\end{align*}
Moreover, due to stability of the synthesized controller, the worst possible controller update scheme is to update the controller every $H_0$ time-steps, \textit{i.e.,} invoking the condition of $t-\tau > H_0$ in the update rule. Notice that this update rule considers the worst effect of similarity transformation on the growth of the state, since otherwise applying the same controller for longer periods would have further reduction on the state due to the contraction that the stabilizing controller brings. Thus, from (\ref{policy_change_stabil}) we have that 
\begin{equation}
    \|x_t \| \leq \frac{(12\kappa^2+ 2\kappa\sqrt{2})\sigma_w}{\gamma}  \sqrt{2n\log(n(t-\Tw)/\delta)},
\end{equation}
for all $t>T_{r} \coloneqq \Tw + \Tbase$ where $\Tbase = \left((n+d)\log(n+d)\right)H_0$.

\section{REGRET DECOMPOSITION} \label{apx:regret_decomp}

The regret decomposition leverages the \OFU principle. Since during the adaptive control with improved exploration period \alg applies independent isotropic perturbations through the controller but still designs the optimistic controller, one can consider the external perturbation as a component of the underlying system. With this way, we consider the regret obtained by using the improved exploration separately. 

First noted that based on the definition of \OFU principle, \alg solves $J(\tilde{\Theta}_{t}) \leq \inf_{\Theta \in \mathcal{C}_{t}(\delta) \cap \mathcal{S}} J(\Theta)+1/\sqrt{t}$ to find the optimistic parameter. This search is done over only $\mathcal{C}_{t}(\delta)$ in the stabilizing adaptive control phase. Denote the system evolution noise at time $t$ as $\zeta_t$. For $t\leq \Tw$, system evolution noise can be considered as $\zeta_t =B_*\nu_t + w_t$ and for $t> \Tw$, $\zeta_t = w_t$. Denote the optimal average cost of system $\ttt$ under $\zeta_t$ as $J_*(\ttt, \zeta_t)$. The regret of the \alg can be decomposed as 
\begin{align}
 \sum_{t=0}^{T}\! x_t^\top Q x_t \!+\! u_t^\top R u_t \!+\! 2 \nu_t^\top R u_t \!+\! \nu_t^\top R \nu_t \!-\! J_*(\Theta_*, w_t)  \label{reg_decomp}
\end{align}
where $u_t$ is the optimal controller input for the optimistic system $\ttt_{t-1}$, $\nu_t$ is the noise injected and $x_t$ is the state of the system $\ttt_{t-1}$ with the system evolution noise of $\zeta_t$. From Bellman optimality equation for LQR, \citep{bertsekas1995dynamic}, we can write the following for the optimistic system, $\ttt_{t-1}$,

\begin{align*}
&J_*(\ttt_{t -1}, \zeta_t)  + x_t^\top  \tilde{P}_{t-1} x_t  = x_t^\top Q x_t + u_t^\top R u_t \\
&+ \mathbb{E}\big[(\ta_{t-1} x_{t}+\tb_{t-1} u_{t}+\zeta_{t})^{\top} \tilde{P}_{t-1} (\ta_{t-1} x_{t}+\tb_{t-1} u_{t}+\zeta_{t}) \big| \mathcal{F}_{t-1}\big],
\end{align*}
where $\tilde{P}_{t-1}$ is the solution of DARE for $\ttt_{t-1}$. Following the decomposition used in without additional exploration \citep{abbasi2011lqr}, we get,

\begin{align*}
	&J_*(\ttt_{t-1}, \zeta_t) + x_t^\top  \tilde{P}_{t-1} x_t - (x_t^\top Q x_t + u_t^\top R u_t ) \\
	 & =  (\ta_{t\!-\!1} x_{t}\!+\!\tb_{t\!-\!1} u_{t})^{\top}\! \tilde{P}_{t-1} \!(\ta_{t\!-\!1} x_{t}\!+\!\tb_{t\!-\!1} u_{t}) \\
	 &\qquad+ \mathbb{E}\big[x_{t+1}^\top \tilde{P}_{t-1} x_{t+1}\big| \mathcal{F}_{t-1}\big] \!-\! (A_* x_{t}\!+\!B_* u_{t})^{\top} \tilde{P}_{t\!-\!1} (A_* x_{t}\!+\!B_* u_{t}) 
\end{align*}
where we use the fact that $x_{t+1} = A_*x_t + B_*u_t + \zeta_t$, the martingale property of the noise and the conditioning on the filtration $\mathcal{F}_{t-1}$. Hence, summing up over time, we get 
\[ \sum\nolimits_{t=0}^{T} \left(x_t^\top Q x_t \!+\! u_t^\top R u_t\right) \!=\! \sum\nolimits_{t=0}^{T} J_*(\ttt_{t-1}, \zeta_t) \!+\! R_1^{\zeta} \!-\! R_2^{\zeta} \!-\! R_3^{\zeta}\]
for
\begin{align}
	R_1^{\zeta} &= \sum\nolimits_{t=0}^{T} \left\{x_t^\top  \tilde{P}_{t\!-\!1} x_t - \mathbb{E}\left[x_{t+1}^\top \tilde{P}_{t} x_{t+1}\big| \mathcal{F}_{t-1}\right] \right \} \label{R1_zeta} \\
	R_2^\zeta &= \sum\nolimits_{t=0}^{T} \mathbb{E}\left[x_{t+1}^\top  \left(\tilde{P}_{t-1} - \tilde{P}_{t}\right) x_{t+1} \big| \mathcal{F}_{t-1} \right ] \label{R2_zeta} \\
	R_3^\zeta &= \sum_{t=0}^{T}    \bar{x}_{t+1,\ttt_{t-1}}^{\top} \tilde{P}_{t-1} \bar{x}_{t+1,\ttt_{t-1}} \!- \bar{x}_{t+1,\tts}^{\top} \tilde{P}_{t-1} \bar{x}_{t+1,\tts}  \label{R3_zeta}
\end{align}
where $\bar{x}_{t+1,\ttt_{t-1}} = \ta_{t-1} x_{t}\!+\!\tb_{t-1} u_{t}$ and $\bar{x}_{t+1,\tts} = A_* x_{t}\!+\!B_* u_{t}$.

Therefore, when we jointly have that $\tts \in \mathcal{C}_{t}(\delta)$ for all time steps $t$ and the state is bounded as shown in Lemma \ref{lem:bounded_state},
\begin{align*}
\sum_{t=0}^{T}\! (x_t^\top Q x_t \!+\! u_t^\top R u_t)
&\!=\! \sum_{t=0}^{\Tw} \! \sigma_\nu^2\Tr(\tilde{P}_{t-1} B_*B_*^\top)  \!
+\! \sum_{t=0}^{T} \! \bar{\sigma}_w^2 \Tr(\tilde{P}_{t-1}) \!+\! R_1^\zeta \!-\! R_2^\zeta \!-\! R_3^\zeta 
\end{align*}
where the equality follows from the fact that, $J_*(\ttt_{t-1}, \zeta_t) = \Tr(\tilde{P}_{t-1} W)$ where $W = \mathbb{E}[\zeta_t\zeta_t^\top | \mathcal{F}_{t-1}]$ for a corresponding filtration $\mathcal{F}_t$. The optimistic choice of $\ttt_t$ provides that
\begin{equation*}
    \bar{\sigma}_w^2 \Tr(\tilde{P}_{t-1}  ) = J_*(\ttt_{t-1}, w_t)  \leq  J_*(\Theta_*, w_t) + 1/\sqrt{t}= \bar{\sigma}_w^2\Tr(P_*) + 1/\sqrt{t}.
\end{equation*}
% Thus we get, 
% \begin{align*}
% \sum_{t=0}^{T} (x_t^\top Q x_t &+ u_t^\top R u_t) \!-\! T J_*(\Theta_*,w_t) \\
% &\leq 
%  \Tw  \max_{0 \leq t \leq \Tw}\left\{\sigma_\nu^2 \Tr(\tilde{P}_{t-1}  B_* B^\top_*) \right \} \!+\! R_1^\zeta \!-\! R_2^\zeta \!-\! R_3^\zeta
% \end{align*}
Combining this with \eqref{reg_decomp} and Assumption \ref{parameterassump_stabilizability}, we obtain the following expression for the regret of \alg :
\begin{equation} \label{regret_decomposition}
\text{R}(T) \leq  \sigma_\nu^2 \Tw D \|B_* \|_F^2 + R_1^\zeta - R_2^\zeta - R_3^\zeta +  \sum_{t=0}^{\Tw} 2 \nu_t^\top R u_t + \nu_t^\top R \nu_t .
\end{equation}

% \paragraph{\alg without Additional Exploration}
% The regret of \alg using only \OFU yields the same regret decomposition in Section 4.2 of \citet{abbasi2011lqr}, since the underlying system dynamics is the same. Let $J_*(\tts, w_t)$ denote the optimal average expected cost of an \LQR, $\tts$, with $w_t$ disturbances obtained by its optimal controller. Therefore, under the event $\mathcal{E}_T \cap \mathcal{F}^{[s]}_T$ for \alg without additional exploration, we have  
% \begin{align*}
% \reg(T)= \sum_{t=0}^{T} \left(x_t^\top Q x_t + u_t^\top R u_t\right) - T J_*(\Theta_*,w_t) &\leq 
%  R_1 - R_2 - R_3 + 2\sqrt{T}
% \end{align*}
% where 
% \begin{equation}
% 	R_1 = \sum_{t=0}^{T} \left\{x_t^\top  P(\ttt_{t-1}) x_t - \mathbb{E}\left[x_{t+1}^\top P(\ttt_{t}) x_{t+1}\big| \mathcal{F}_{t-1}\right] \right \} \label{R1}
% \end{equation}
% \begin{equation} 
% 	R_2 = \sum_{t=0}^{T} \mathbb{E}\left[x_{t+1}^\top  \left(P(\ttt_{t-1}) - P(\ttt_{t})\right) x_{t+1} \big| \mathcal{F}_{t-1} \right ] \label{R2}
% \end{equation}
% \begin{equation}
% R_3 =  \sum_{t=0}^{T}   (\ta_{t-1} x_{t}\!+\!\tb_{t-1} u_{t})^{\top} \! P(\ttt_{t-1}) (\ta_{t-1} x_{t}\!+\!\tb_{t-1} u_{t}) \! -\! (A_* x_{t}\!+\!B_* u_{t})^{\top}\! P(\ttt_{t-1}) (A_* x_{t}\!+\!B_* u_{t}) . \label{R3}
% \end{equation}

\section{REGRET ANALYSIS} \label{apx:regret_analysis}

In this section, we provide the bounds on each term in the regret decomposition separately. We show that the regret suffered from the improved exploration is tolerable in the upcoming stages via the guaranteed stabilizing controller, yielding polynomial dimension dependency in regret.

\subsection[Direct Effect of Improved Exploration]{Direct Effect of Improved Exploration, Bounding 
$\sum_{t=0}^{\Tw} \left(2 \nu_t^\top R u_t + \nu_t^\top R \nu_t\right)$ in the event of $\mathcal{E}_T \cap \mathcal{F}^{[s]}_{\Tw}$}

The following gives an upper bound on the regret attained due to isotropic perturbations in the adaptive control with improved exploration phase of \alg. 
% The following result also holds for stabilizable setting since the state is upper bounded similar to the controllable \LQR{}s, \textit{i.e.}, $\|x_t \| \leq c' (n+d)^{n+d}$ for $t\leq \Tw$.

\begin{lemma}[Direct Effect of Improved Exploration on Regret] \label{lem:exp_effect}
If $\mathcal{E}_T \cap \mathcal{F}^{[s]}_{\Tw}$ holds then with probability at least $1-\delta$, 

\begin{equation}
    \sum_{t=0}^{\Tw} \left(2 \nu_t^\top R u_t + \nu_t^\top R \nu_t\right) \leq d \sigma_\nu \sqrt{B_\delta} + d\|R\|\sigma_\nu^2 \left(\Tw + \sqrt{\Tw} \log \frac{4d\Tw}{\delta}\sqrt{\log \frac{4}{\delta}} \right) 
\end{equation}
where 
\begin{equation*}
B_\delta =  8 \left(1 + \Tw \kappa^2 \|R\|^2 (n+d)^{2(n+d)} \right) \log \left( \frac{4d}{\delta} \left(1 + \Tw \kappa^2 \|R\|^2 (n+d)^{2(n+d)} \right)^{1/2} \right). 
\end{equation*}
\end{lemma}
\begin{proof}
 Let $q_{t}^\top =  u_{t}^\top R $. The first term can be written as 
\[
2\sum_{t=0}^{\Tw} \sum_{i=1}^{d} q_{t, i} \nu_{t, i} = 2\sum_{i=1}^{d}  \sum_{t=0}^{\Tw} q_{t, i} \nu_{t, i}
\]
Let $M_{t, i}=\sum_{k=0}^{t} q_{k, i} \nu_{k, i} .$ By Theorem \ref{selfnormalized} on some event $G_{\delta, i}$ that holds with probability at least $1-\delta /(2 d),$ for any $t \geq 0$,  

\begin{align*}
	M_{t, i}^2 &\leq 2 \sigma_\nu^2  \left(1 + \sum_{k=0}^{t} q_{k,i}^2\right) \log \left(\frac{2d}{\delta} \left(1 + \sum_{k=0}^{t} q_{k,i}^2\right)^{1/2} \right)   \\
\end{align*}

On $\mathcal{E}_T \cap \mathcal{F}^{[s]}_{\Tw}$ or $\mathcal{E}_T \cap \mathcal{F}^{[c]}_{T_c}$, $\| q_k \| \leq \kappa \|R\| (n+d)^{n+d} $, thus $q_{k,i} \leq \kappa \|R\| (n+d)^{n+d} $. Using union bound we get, for probability at least $1-\frac{\delta}{2}$, 

\begin{align} \label{explore_reg_1}
	&\sum_{t=0}^{\Tw} 2 \nu_t^\top R u_t  \leq  \nonumber \\
	&d \sqrt{ 8 \sigma_\nu^2  \left(1 + \Tw \kappa^2 \|R\|^2 (n+d)^{2(n+d)} \right) \log \left( \frac{4d}{\delta} \left(1 + \Tw \kappa^2 \|R\|^2 (n+d)^{2(n+d)} \right)^{1/2} \right) }
\end{align}

Let $W = \sigma_\nu  \sqrt{2d\log \frac{4d\Tw}{\delta}}$. Define $\Psi_{t}=\nu_{t}^{\top} R\nu_{t}-\mathbb{E}\left[\nu_{t}^{\top} R\nu_{t} | \mathcal{F}_{t-1}\right]$ and its truncated version $\tilde{\Psi}_{t}=\Psi_{t} \mathbb{I}_{\left\{\Psi_{t} \leq 2 D W^{2}\right\}}$. 

\begin{align*}
	\Pr \bigg( \sum_{t=1}^{\Tw} \Psi_t &>  2\|R\| W^2 \sqrt{2\Tw \log \frac{4}{\delta}}\bigg) \leq \\
	 &\Pr \left( \max_{1 \leq t \leq \Tw} \Psi_t >  2\|R\| W^2 \right) + \Pr \left( \sum_{t=1}^{\Tw} \tilde{\Psi}_t >  2\|R\| W^2 \sqrt{2\Tw \log \frac{4}{\delta}}\right) 
\end{align*}

Using Lemma \ref{subgauss lemma} with union bound and Theorem \ref{Azuma}, summation of terms on the right hand side is bounded by $\delta/2$. Thus, with probability at least $1-\delta/2$,

\begin{equation} \label{explore_reg_2}
	\sum_{t=0}^{\Tw} \nu_t^\top R \nu_t \leq d \Tw \sigma_\nu^2 \|R\| + 2\|R\| W^2 \sqrt{2\Tw \log \frac{4}{\delta}} .
\end{equation}

Combining \eqref{explore_reg_1} and \eqref{explore_reg_2} gives the statement of lemma for the regret of external exploration noise.   

\end{proof}

\subsection[Bound on R1 with Additional Exploration]{Bounding $R_1^\zeta$ in the event of $\mathcal{E}_T \cap \mathcal{F}^{[s]}_{\Tw}$ or $\mathcal{E}_T \cap \mathcal{F}^{[c]}_{T_c}$}
In this section, we state the bound on $R_1^\zeta$ given in \eqref{R1_zeta}. We first provide high probability bound on the system noise.

\begin{lemma}[Bounding sub-Gaussian vector] \label{noise_bound}
 With probability $1-\frac{\delta}{8}$, $\|\zeta_k \| \leq (\sigma_w+\|B_*\|\sigma_\nu) \sqrt{2n\log \frac{8nT}{\delta}}$ for $k \leq \Tw$ and $\|\zeta_k \| \leq \sigma_w \sqrt{2n\log \frac{8nT}{\delta}}$ for  $ \Tw < k \leq T$.  
\end{lemma}

\begin{proof}
 From the subgaussianity assumption, we have that for any index $1\leq i \leq n$ and any time $k$, $|w_{k,i}| \leq \sigma_w \sqrt{2\log \frac{8}{\delta}}$ and $|(B_*\nu_k)_i| < \| B_*\| \sigma_\nu \sqrt{2\log \frac{8}{\delta}}$ with probability $1-\frac{\delta}{8}$. Using the union bound, we get the statement of lemma. 
\end{proof}

Using this we state the bound on $R_1^\zeta$ for stabilizable systems.
\begin{lemma}[Bounding $R_1^\zeta$ for \alg] \label{R1_zeta_stabil}
Let $R_1^\zeta$ be as defined by (\ref{R1_zeta}). Under the event of $\mathcal{E}_T \cap \mathcal{F}^{[s]}_{\Tw}$, with probability at least $1-\delta/2$, using \alg for $t> T_{r}$, we have 
\begin{align*}
    R_1 &\leq k_{s,1} (n+d)^{n+d} (\sigma_w + \| B_*\|\sigma_\nu)n \sqrt{T_{r}} \log((n+d) T_{r}/\delta) \\
    &+ \frac{k_{s,2}(12\kappa^2+ 2\kappa\sqrt{2})}{\gamma}  \sigma_w^2 n\sqrt{n} \sqrt{T-\Tw} \log(n(t-\Tw)/\delta) \\
    &+  k_{s,3} n\sigma_w^2 \sqrt{T-\Tw}\log(nT/\delta) + k_{s,4} n(\sigma_w+\|B_*\|\sigma_\nu)^2\sqrt{\Tw}\log(nT/\delta),
\end{align*}
for some problem dependent coefficients $k_{s,1}, k_{s,2}, k_{s,3}, k_{s,4}$. 
\end{lemma}

\begin{proof}
 Assume that the event $\mathcal{E}_T \cap \mathcal{F}^{[s]}_{\Tw}$ holds. Let $f_{t} = A_*x_{t} + B_*u_{t}$. One can decompose $R_1$ as 
 \[
 R_1 = x_0^\top P(\ttt_0) x_0 - x_{T+1}^\top P(\ttt_{T+1}) x_{T+1} + \sum_{t=1}^{T} x_t^\top P(\ttt_{t}) x_t - \mathbb{E}\left[x_{t}^\top P(\ttt_{t}) x_{t}\big| \mathcal{F}_{t-2}\right] 
 \]
 Since $P(\ttt_0)$ is positive semidefinite and $x_0 = 0$, the first two terms are bounded above by zero. The second term is decomposed as follows \[
 \sum_{t=1}^{T} x_t^\top P(\ttt_{t}) x_t - \mathbb{E}\left[x_{t}^\top P(\ttt_{t}) x_{t}\big| \mathcal{F}_{t-2}\right]  = \sum_{t=1}^{T} f_{t-1}^\top P(\ttt_{t}) \zeta_{t-1} + \sum_{t=1}^{T} \left( \zeta_{t-1}^\top P(\ttt_{t}) \zeta_{t-1} - \mathbb{E}\left[\zeta_{t-1}^\top P(\ttt_{t}) \zeta_{t-1}\big| \mathcal{F}_{t-2}\right] \right)
 \]
 Let $R_{1,1} =  \sum_{t=1}^{T} f_{t-1}^\top P(\ttt_{t}) \zeta_{t-1} $ and $R_{1,2} = \sum_{t=1}^{T} \left( \zeta_{t-1}^\top P(\ttt_{t}) \zeta_{t-1} - \mathbb{E}\left[\zeta_{t-1}^\top \tp_t \zeta_{t-1}\big| \mathcal{F}_{t-2}\right] \right)$. Let $v_{t-1}^\top =  f_{t-1}^\top P(\ttt_{t}) $. $R_{1,1}$ can be written as 
 \[
 R_{1,1} = \sum_{t=1}^{T} \sum_{i=1}^{n} v_{t-1, i} \zeta_{t-1, i} = \sum_{i=1}^{n}  \sum_{t=1}^{T} v_{t-1, i} \zeta_{t-1, i}.
 \]
 Let $M_{t, i}=\sum_{k=1}^{t} v_{k-1, i} \zeta_{k-1, i} .$ By Theorem \ref{selfnormalized} on some event $G_{\delta, i}$ that holds with probability at least $1-\delta /(4 n),$ for any $t \geq 0$,  
\begin{align*}
M_{t, i}^2 &\leq 2 (\sigma_w^2 + \|B_*\|^2 \sigma_\nu^2 ) \left(1 + \sum_{k=1}^{T_{r}} v_{k-1,i}^2\right) \log \left(\frac{4n}{\delta} \left(1 + \sum_{k=1}^{T_{r}} v_{k-1,i}^2\right)^{1/2} \right) \\
&\qquad + 
2 \sigma_w^2 \left(1 + \sum_{k=T_{r}+1}^{t} v_{k-1,i}^2\right) \log \left(\frac{4n}{\delta} \left(1 + \sum_{k=T_{r}+1}^{t} v_{k-1,i}^2\right)^{1/2} \right) \quad \text{for } t > T_{r}.
\end{align*}
Notice that \alg stops additional isotropic perturbation after $t=\Tw$, and the state starts decaying until $t=T_{r}$. For simplicity of presentation we treat the time between $\Tw$ and $T_{r}$ as exploration sacrificing the tightness of the result. On $\mathcal{E}_T \cap \mathcal{F}^{[s]}_{\Tw}$, $\| v_k \| \leq DS (n+d)^{n+d} \sqrt{1+\kappa^2} $ for $k \leq T_{r}$ and $\| v_k \| \leq
\frac{(12\kappa^2+ 2\kappa\sqrt{2})DS\sigma_w\sqrt{1+\kappa^2}}{\gamma}  \sqrt{2n\log(n(t-\Tw)/\delta)}$ for $k > T_{r}$. Thus, $v_{k,i} \leq DS (n+d)^{n+d} \sqrt{1+\kappa^2}$ and $v_{k,i} \leq \frac{(12\kappa^2+ 2\kappa\sqrt{2})DS\sigma_w\sqrt{1+\kappa^2}}{\gamma}  \sqrt{2n\log(n(t-\Tw)/\delta)}$ respectively for $k\leq T_{r}$ and $k>T_{r}$ . Using union bound we get, for probability at least $1-\frac{\delta}{4}$,  for $t > T_{r}$,
\begin{align*}
    R_{1,1} &\leq n \sqrt{ 2 (\sigma_w^2 + \|B_*\|^2 \sigma_\nu^2 )  \left(1 + T_{r} D^2 S^2 (n+d)^{2(n+d)}  (1+\kappa^2) \right) } \times \\
    &\qquad \qquad \qquad \sqrt{\log \left( \frac{4n}{\delta} \left(1 + T_{r} D^2 S^2 (n+d)^{2(n+d)}  (1+\kappa^2) \right)^{1/2}  \right) }   \\
&+n \sqrt{ 2 \sigma_w^2  \left(1 + \frac{2(t-T_{r})(12\kappa^2+ 2\kappa\sqrt{2})^2D^2S^2n\sigma_w^2(1+\kappa^2)}{\gamma^2} \log(n(T-\Tw)/\delta) \right) } \times \\
&\qquad \qquad \qquad \sqrt{ \log \left( \frac{4n}{\delta} \left(1 + \frac{2(t-T_{r})(12\kappa^2+ 2\kappa\sqrt{2})^2D^2S^2n\sigma_w^2(1+\kappa^2)}{\gamma^2} \log(n(T-\Tw)/\delta) \right)  \right) }.
\end{align*}
Let $\mathcal{W}_{exp} = (\sigma_w+\|B_*\|\sigma_\nu) \sqrt{2n\log \frac{8nT}{\delta}}$ and $\mathcal{W}_{noexp} =  \sigma_w \sqrt{2n\log \frac{8nT}{\delta}}$. Define $\Psi_{t}=\zeta_{t-1}^{\top} P(\ttt_{t}) \zeta_{t-1}-\mathbb{E}\left[\zeta_{t-1}^{\top} P(\ttt_{t}) \zeta_{t-1} | \mathcal{F}_{t-2}\right]$ and its truncated version $\tilde{\Psi}_{t}=\Psi_{t} \mathbb{I}_{\left\{\Psi_{t} \leq 2 D W_{exp}^{2}\right\}}$ for $t \leq \Tw$ and $\tilde{\Psi}_{t}=\Psi_{t} \mathbb{I}_{\left\{\Psi_{t} \leq 2 D W_{noexp}^{2}\right\}}$ for $t > \Tw$ . Notice that $R_{1,2} =  \sum_{t=1}^{T} \Psi_{t} $. 
\begin{align*}
	&\Pr \left( \sum_{t=1}^{\Tw} \Psi_t > 2DW_{exp}^2 \sqrt{2\Tw \log \frac{8}{\delta}}\right) + \Pr \left( \sum_{t=\Tw+1}^{T} \Psi_t > 2DW_{noexp}^2 \sqrt{2(T-\Tw) \log \frac{8}{\delta}}\right) \\
	&\leq \Pr \left( \max_{1 \leq t \leq \Tw} \Psi_t > 2DW_{exp}^2 \right) + \Pr \left( \max_{\Tw + 1 \leq t \leq T} \Psi_t > 2DW_{noexp}^2 \right) \\
	&+\Pr \left( \sum_{t=1}^{\Tw} \tilde{\Psi}_t > 2DW_{exp}^2 \sqrt{2\Tw \log \frac{8}{\delta}}\right) +  \Pr \left( \sum_{t=\Tw+1}^{T} \tilde{\Psi}_t > 2DW_{noexp}^2 \sqrt{2(T-\Tw) \log \frac{8}{\delta}}\right)
\end{align*}
By Lemma \ref{subgauss lemma} with union bound and Theorem \ref{Azuma}, summation of terms on the right hand side is bounded by $\delta/4$. Thus, with probability at least $1-\delta/4$, for $t > \Tw$,
\[
 R_{1,2} \leq  4nD\sigma_w^2   \sqrt{2(t-\Tw) \log \frac{8}{\delta}}\log \frac{8nT}{\delta} + 4nD(\sigma_w+\|B_*\|\sigma_\nu)^2\sqrt{2\Tw \log \frac{8}{\delta}}\log \frac{8nT}{\delta}.
 \]
 Combining $R_{1,1}$ and $R_{1,2}$ gives the statement. 
\end{proof}

% \begin{lemma}[Bounding $R_1^\zeta$ for stabilizable LQR] \label{R1_zeta_stabil}
% Let $R_1^\zeta$ be as defined by (\ref{R1_zeta}). Under the event of $\mathcal{E}_T \cap \mathcal{F}^{[s]}_{\Tw}$, with probability at least $1-\delta/2$, for stabilizable \LQR, for $t> T_{r}$, we have 
% \begin{align*}
%     R_1 &\leq k_{s,1} (n+d)^{n+d} (\sigma_w + \| B_*\|\sigma_\nu)n \sqrt{T_{r}} \log((n+d) T_{s,c}/\delta) \\
%     &+ \frac{k_{s,2}(12\kappa^2+ 2\kappa\sqrt{2})}{\gamma}  \sigma_w^2 n\sqrt{n}  \log(n(t-\Tw)/\delta) \\
%     &+  k_{s,3} n\sigma_w^2 \sqrt{T-\Tw}\log(nT/\delta) + k_{s,4} n(\sigma_w+\|B_*\|\sigma_\nu)^2\sqrt{\Tw}\log(nT/\delta),
% \end{align*}
% for some problem dependent coefficients $k_{s,1}, k_{s,2}, k_{s,3}, k_{s,4}$. 
% \end{lemma}

\subsection[Bound on R2 with Additional Exploration]{Bounding $|R_2^\zeta|$ on the event of $\mathcal{E}_T \cap \mathcal{F}^{[s]}_{\Tw}$}
In this section, we will bound $|R_2^\zeta|$ given in \eqref{R2_zeta}. We first provide a bound on the maximum number of policy changes. 

\begin{lemma}[Number of Policy Changes for \alg]\label{number_of_change_stabil_exp}
On the event of $\mathcal{E}_T \cap \mathcal{F}^{[c]}_{\Tw}$, \alg changes the policy at most 
\begin{equation}
 \min \left \{T/H_0 , (n+d) \log_2\left(1 + \frac{\lambda + T_r (n+d)^{2(n+d)} (1+\kappa^2) + (T-T_r)(1+\kappa^2)X_s^2}{\lambda}\right) \right\},
\end{equation}
where $X_s = \frac{(12\kappa^2+ 2\kappa\sqrt{2})\sigma_w}{\gamma}  \sqrt{2n\log(n(T-\Tw)/\delta)}$.
\end{lemma}
\begin{proof}
 Changing policy $K$ times up to time $\Tw$ requires $\det(V_{T}) \geq \lambda^{n+d} 2^K$. We also have that 
 \[
 \lambda_{\text{max}}(V_{T}) \leq \lambda + \sum_{t=0}^{T} \|z_t\|^2 \leq \lambda + T_r (n+d)^{2(n+d)} (1+\kappa^2) + (T-T_r)(1+\kappa^2)X_s^2
 \]
 
 Thus, $\lambda^{n+d} 2^K \leq \left( \lambda + T_r (n+d)^{2(n+d)} (1+\kappa^2) + (T-T_r)(1+\kappa^2)X_s^2 \right)^{n+d}$. Solving for K gives 
 
 \[
 K \leq (n+d) \log_2\left(1 + \frac{ T_r (n+d)^{2(n+d)} (1+\kappa^2) + (T-T_r)(1+\kappa^2)X_s^2}{\lambda}\right).
 \]
 
  Moreover, the number of policy changes is also controlled by the lower bound $H_0$ on the duration of each controller. This policy update method would give at most $T/H_0$ policy changes. Since for the policy update of \alg requires both conditions to be met, the upper bound on the number of policy changes is minimum of these. 
\end{proof}
Notice that besides the policy change instances, all the terms in $R_2^\zeta$ are 0. Therefore, we have the following results for stabilizable systems. 

% \begin{lemma}[Bounding $R_2^\zeta$ for controllable LQR] \label{R2_zeta_control}
% 	Let $R_2^\zeta$ be as defined by \eqref{R2_zeta}. Under the event of $\mathcal{E}_T \cap \mathcal{F}^{[c]}_{T_c}$, for controllable \LQR, we have
% 	\begin{align*}
% 	  |R_2| &\leq 2D (n+d)^{2(n+d)+1} \log_2\left(1 + \frac{T_{r} (n+d)^{2(n+d)} (1+\kappa^2)  }{\lambda}\right) \\
% 	&+ 2D  \frac{32n\sigma_w^2(1+\kappa^2)}{(1-\Upsilon)^2}  \log \frac{n(T-T_c)}{\delta} (n+d) \times \\ 
% 	&\quad  \log_2\left(1 + \frac{T_{r} 
% 	(n+d)^{2(n+d)} (1+\kappa^2) + (T-T_{r})\frac{32n\sigma_w^2(1+\kappa^2)}{(1-\Upsilon)^2}   \log \frac{n(T-T_c)}{\delta}}{\lambda}\right)  
% 	\end{align*}
% \end{lemma}
% \begin{proof}
% On the event $\mathcal{E}_T \cap \mathcal{F}^{[c]}_{T_c}$, we know the maximum number of policy changes up to $T_{r}$ and $T$ using Lemma \ref{number_of_change_control_exp}. Using the fact that $\|x_t\| \leq (n+d)^{n+d}$ for $t \leq T_{r}$ and $\| x_t\| \leq \frac{4\sigma_w}{1-\Upsilon}   \sqrt{2n\log \frac{n(T-T_c)}{\delta}}$, we obtain the statement of the lemma. 
% \end{proof}
\begin{lemma}[Bounding $R_2^\zeta$ for \alg]\label{R2_zeta_stabil}
Let $R_2^\zeta$ be as defined by \eqref{R2_zeta}. Under the event of $\mathcal{E}_T \cap \mathcal{F}^{[s]}_{\Tw}$, using \alg, we have
	\begin{align*}
	  |R_2^{\zeta}| &\leq 2D (n+d)^{2(n+d)+1} \log_2\left(1 + \frac{T_{r} (n+d)^{2(n+d)} (1+\kappa^2)  }{\lambda}\right) \\
	&+ 2D X_{s}^2 (n+d)\log_2\left(1 + \frac{ T_r (n+d)^{2(n+d)} (1+\kappa^2) + (T-T_r)(1+\kappa^2)X_s^2}{\lambda}\right)
	\end{align*}
where $X_s = \frac{(12\kappa^2+ 2\kappa\sqrt{2})\sigma_w}{\gamma}  \sqrt{2n\log(n(T-\Tw)/\delta)}$
\end{lemma}
\begin{proof}
On the event $\mathcal{E}_T \cap \mathcal{F}^{[s]}_{\Tw}$, we know the maximum number of policy changes up to $T_{r}$ and $T$ using Lemma \ref{number_of_change_stabil_exp}. Using the fact that $\|x_t\| \leq (n+d)^{n+d}$ for $t \leq T_{r}$ and $\| x_t\| \leq \frac{(12\kappa^2+ 2\kappa\sqrt{2})\sigma_w}{\gamma}  \sqrt{2n\log(n(t-\Tw)/\delta)},$ we obtain the statement of the lemma. 
\end{proof}
\subsection[Bound on R3 with Additional Exploration]{Bounding $|R_3^\zeta|$ on the event of $\mathcal{E}_T \cap \mathcal{F}^{[s]}_{\Tw}$ }
Before bounding $R_3^\zeta$, first consider the following for stabilizable \LQR{}s.

\begin{lemma} \label{problematic_lemma}
On the event of $\mathcal{E}_T \cap \mathcal{F}^{[s]}_{\Tw}$, using \alg in a stabilizable \LQR, the following holds,
\begin{align*}
    \sum_{t=0}^T \|(\tts - \ttt_t)^\top \!z_t \|^2 \!&\leq \!
    \frac{8(1+\kappa^2)\beta_T^2(\delta)}{\lambda} \times \\
    &\Bigg( (n+d)^{2(n+d)}  \max\left\{2, \left(1+ \frac{(1+\kappa^2)(n+d)^{2(n+d)}}{\lambda}\right)^{H_0} \right\} \log \frac{\det(V_{T_{r}})}{\det(\lambda I)}\\
    &\qquad+ X_s^2 \max\left\{2, \left(1+ \frac{(1+\kappa^2)X_s^2}{\lambda}\right)^{H_0} \right\} \log \frac{\det(V_{T})}{\det(V_{T_{r}})} \Bigg)  
\end{align*}
where $X_s = \frac{(12\kappa^2+ 2\kappa\sqrt{2})\sigma_w}{\gamma}  \sqrt{2n\log(n(t-\Tw)/\delta)}$.
\end{lemma}
\begin{proof}
 Let $s_t = (\tts - \ttt_t)^\top z_t$ and $\tau\leq t$ be the time step that the last policy change happened. We have the following using triangle inequality,
 \begin{equation*}
     \| s_t \| \leq \|(\tts - \tth_t)^\top z_t \| + \|(\tth_t - \ttt_t)^\top z_t \|.
 \end{equation*}
For all $\Theta \in \mathcal{C}_{\tau}(\delta)$, for $\tau \leq T_{r}$, we have
\begin{align}
    \|(\Theta - \tth_t)^\top z_t \| &\leq \| V_t^{1/2} (\Theta - \tth_t)\| \|z_t \|_{V_t^{-1}}  \\
    &\leq \|V_\tau^{1/2} (\Theta - \tth_t) \| \sqrt{\frac{\det(V_t)}{\det (V_\tau)}} \|z_t\|_{V_t^{-1}}  \\
    &\leq \max\left\{\sqrt{2}, \sqrt{\left(1+ \frac{(1+\kappa^2)(n+d)^{2(n+d)}}{\lambda}\right)^{H_0}} \right\} \|V_\tau^{1/2} (\Theta - \tth_t) \| \|z_t\|_{V_t^{-1}} \\ 
    &\leq  \max\left\{\sqrt{2}, \sqrt{\left(1+ \frac{(1+\kappa^2)(n+d)^{2(n+d)}}{\lambda}\right)^{H_0}} \right\} \beta_\tau(\delta) \|z_t\|_{V_t^{-1}}. 
\end{align}
Similarly, for  for all $\Theta \in \mathcal{C}_{\tau}(\delta)$, for $\tau > T_{r}$, we have \[
\|(\Theta - \tth_t)^\top z_t \| \leq \max\left\{\sqrt{2}, \sqrt{\left(1+ \frac{(1+\kappa^2)X_s^2}{\lambda}\right)^{H_0}} \right\} \beta_\tau(\delta) \|z_t\|_{V_t^{-1}}
\]
Using these results, we obtain, 
\begin{align*}
    \sum_{t=0}^T &\|(\tts - \ttt_t)^\top z_t \|^2 \\
    &\leq 8 \max\left\{2, \left(1+ \frac{(1+\kappa^2)(n+d)^{2(n+d)}}{\lambda}\right)^{H_0} \right\} \frac{\beta_T^2(\delta)(1+\kappa^2)(n+d)^{2(n+d)}}{\lambda} \log \left(\frac{\det (V_{T_{r}})}{\det (\lambda I)} \right)  \\
    &+ 8 \max\left\{2, \left(1+ \frac{(1+\kappa^2)X_s^2}{\lambda}\right)^{H_0} \right\} \frac{\beta_T^2(\delta)(1+\kappa^2)X_s^2}{\lambda} \log \left(\frac{\det (V_{T})}{\det (V_{T_{r}})} \right)
\end{align*}
where we use Lemma \ref{upperboundlemma}.
\end{proof}

% \begin{lemma} \label{nonproblematic_lemma}
% On the event of $\mathcal{E}_T \cap \mathcal{F}^{[s]}_{\Tw}$, the following holds,
% \begin{align*}
%     \sum_{t=T_{r}}^T \|(\tts - \ttt_t)^\top \!z_t \|^2 \!&\leq \kappa_e^2 (1+\kappa^2)X_s^2\log(T)
% \end{align*}
% for $X_s = \frac{(12\kappa^2+ 2\kappa\sqrt{2})\sigma_w}{\gamma}  \sqrt{2n\log(n(T-\Tw)/\delta)}$ and $\kappa_e$ as defined in \eqref{kappa_e}.
% \end{lemma}
% \begin{proof}
% Let $s_t = (\tts - \ttt_t)^\top z_t$ and $\tau\leq t$ be the time step that the last policy change happened. We have the following using triangle inequality,
%  \begin{equation*}
%      \| s_t \| \leq \|(\tts - \tth_t)^\top z_t \| + \|(\tth_t - \ttt_t)^\top z_t \|.
%  \end{equation*}
% For all $\Theta \in \mathcal{C}_{\tau}(\delta)$, for $\tau > T_{r}$, we have
% \begin{align*}
%     \|(\Theta - \tth_t)^\top z_t \| &\leq \|  (\Theta - \tth_t)\| \| z_t\|   \\
%  &\leq \frac{\kappa_{e}}{\sqrt{\tau}} \sqrt{1+\kappa^2} X_s
% \end{align*}
% Let $\tau_t$ denote the time step that last policy change occured before time $t$. Using the new update rule we obtain, 
% \begin{align*}
%     \sum_{t=T_{r}}^T \|(\tts - \ttt_t)^\top z_t \|^2
%     &\leq \sum_{t=T_{r}}^T 4 \frac{\kappa_e^2 (1+\kappa^2)X_s^2}{\tau_t}\\
%     &\leq \Tbase 4 \frac{\kappa_e^2 (1+\kappa^2)X_s^2}{\Tw + \Tbase} + 2\Tbase 4 \frac{\kappa_e^2 (1+\kappa^2)X_s^2}{\Tw + 2\Tbase} + ...\\
%     &\leq \kappa_e^2 (1+\kappa^2)X_s^2\log((T-T_r)/\Tbase),
% \end{align*}
% where the last line follows from the fact that there can be at most $\log((T-T_r)/\Tbase)$ updates in this update scheme.
% \end{proof} 
 
 Using Lemma \ref{problematic_lemma}, we bound $R_3^\zeta$ as follows.

\begin{lemma}[Bounding $R_3^\zeta$ for \alg]\label{R3_zeta_stabil}
Let $R_3^\zeta$ be as defined by \eqref{R3_zeta}. Under the event of $\mathcal{E}_T \cap \mathcal{F}^{[s]}_{\Tw}$, using \alg with the choice of $\lambda = (1+\kappa^2)X_s^2$, we have
	\begin{align*}
	  |R_3^\zeta| &= \OO \left((n+d)^{(H_0+2)(n+d) + 2} \sqrt{n} \sqrt{T_{r}} + (n+d) n \sqrt{T-T_{r}} \right)
	\end{align*}
\end{lemma}
\begin{proof}
Let $Y_1 = \frac{8(1+\kappa^2)\beta_T^2(\delta)}{\lambda} (n+d)^{2(n+d)}  \max\left\{2, \left(1+ \frac{(1+\kappa^2)(n+d)^{2(n+d)}}{\lambda}\right)^{H_0} \right\} \log \frac{\det(V_{T_{r}})}{\det(\lambda I)}$ and $Y_2 = \frac{8(1+\kappa^2)\beta_T^2(\delta)}{\lambda} X_s^2 \max\left\{2, \left(1+ \frac{(1+\kappa^2)X_s^2}{\lambda}\right)^{H_0} \right\} \log \frac{\det(V_{T})}{\det(V_{T_{r}})} $ for $X_s = \frac{(12\kappa^2+ 2\kappa\sqrt{2})\sigma_w}{\gamma}  \sqrt{2n\log(n(t-\Tw)/\delta)}$. The following uses triangle inequality and Cauchy Schwarz inequality and again triangle inequality to give:

 \begin{align*} 
 \left|R_{3}^\zeta\right| 
 &\leq \sum_{t=0}^{T}\left|\left\|P(\ttt_t)^{1 / 2} \ttt_{t}^{\top} z_{t}\right\|^{2}-\left\|P(\ttt_t)^{1 / 2} \tts^{\top} z_{t}\right\|^{2}\right| \\
 =& \sum_{t=0}^{T_{r}}\left|\left\|P(\ttt_t)^{1 / 2} \ttt_{t}^{\top} z_{t}\right\|^{2}-\left\|P(\ttt_t)^{1 / 2} \tts^{\top} z_{t}\right\|^{2}\right| + \sum_{t=T_{r}}^{T}\left|\left\|P(\ttt_t)^{1 / 2} \ttt_{t}^{\top} z_{t}\right\|^{2}-\left\|P(\ttt_t)^{1 / 2} \tts^{\top} z_{t}\right\|^{2}\right| \\ 
\leq& \left(\sum_{t=0}^{T_{r}}\left(\left\|P(\ttt_t)^{1 / 2} \ttt_{t}^{\top} z_{t}\right\|\!-\!\left\|P(\ttt_t)^{1 / 2} \tts^{\top} z_{t}\right\|\right)^{2}\right)^{1 / 2} \!\!\! \left(\sum_{t=0}^{T_{r}}\left(\left\|P(\ttt_t)^{1 / 2} \ttt_{t}^{\top} z_{t}\right\|\!+\!\left\|P(\ttt_t)^{1 / 2} \tts^{\top} z_{t}\right\|\right)^{2}\right)^{1 / 2} \\
+&  \left(\sum_{t=T_{r}}^{T}\left(\left\|P(\ttt_t)^{1 / 2} \ttt_{t}^{\top} z_{t}\right\|\!-\!\left\|P(\ttt_t)^{1 / 2} \tts^{\top} z_{t}\right\|\right)^{2}\right)^{1 / 2} \!\!\! \left(\sum_{t=T_{r}}^{T}\left(\left\|P(\ttt_t)^{1 / 2} \ttt_{t}^{\top} z_{t}\right\|\!+\!\left\|P(\ttt_t)^{1 / 2} \tts^{\top} z_{t}\right\|\right)^{2}\right)^{1 / 2} \\
\leq& \left(\sum_{t=0}^{T_{r}}\left\|P(\ttt_t)^{1 / 2}\left(\ttt_{t}-\tts\right)^{\top} z_{t}\right\|^{2}\right)^{1 / 2} \!\!\! \left(\sum_{t=0}^{T_{r}}\left(\left\|P(\ttt_t)^{1 / 2} \ttt_{t}^{\top} z_{t}\right\|+\left\|P(\ttt_t)^{1 / 2} \tts^{\top} z_{t}\right\|\right)^{2}\right)^{1 / 2} \\
&+ \left(\sum_{t=T_{r}}^{T}\left\|P(\ttt_t)^{1 / 2}\left(\ttt_{t}-\tts\right)^{\top} z_{t}\right\|^{2}\right)^{1 / 2} \!\!\! \left(\sum_{t=T_{r}}^{T}\left(\left\|P(\ttt_t)^{1 / 2} \ttt_{t}^{\top} z_{t}\right\|+\left\|P(\ttt_t)^{1 / 2} \tts^{\top} z_{t}\right\|\right)^{2}\right)^{1 / 2} \\
\leq& \sqrt{Y_1} \sqrt{4T_{r} D (1+\kappa^2) S^2 (n+d)^{2(n+d)}} + \sqrt{ Y_2} \sqrt{4(T-T_{r}) D (1+\kappa^2) S^2 X_s^2}\\
\leq& \frac{\max\left\{8, 4\sqrt{2}\left(1+ \frac{(1+\kappa^2)(n+d)^{2(n+d)}}{\lambda}\right)^{H_0/2} \right\}DS(1+\kappa^2)\beta_T(\delta)(n+d)^{2(n+d)}}{\sqrt{\lambda}} \times \\
&\qquad \qquad \sqrt{T_{r} (n+d) \log \left( 1 + \frac{T_{r}(1+\kappa^2)(n+d)^{2(n+d)}}{\lambda(n+d)} \right)} \\
+& \frac{\max\left\{8, 4\sqrt{2}\left(1+ \frac{(1+\kappa^2)X_s^2}{\lambda}\right)^{H_0/2} \right\}DS(1+\kappa^2)\beta_T(\delta)}{\sqrt{\lambda}} X_s^2 \times \\
&\qquad \qquad \sqrt{(T-T_{r}) (n+d)\log \left( 1 + \frac{T_{r}(1+\kappa^2)(n+d)^{2(n+d)} + (T-T_{r})X_s^2}{\lambda(n+d)} \right) }
\end{align*}
Examining the first term, it has the dimension dependency of $(n+d)^{(n+d)H_0} \times \sqrt{n(n+d)} \times (n+d)^{2(n+d)} \times \sqrt{n+d}$ where $\sqrt{n(n+d)}$ is due to $\beta_T(\delta)$. For the second term, with the choice of $\lambda = (1+\kappa^2)X_s^2$, the exponential dependency on the dimension with $H_0$ can be converted to a scalar multiplier, \textit{i.e.}, $\left(1+ \frac{(1+\kappa^2)X_s^2}{\lambda}\right)^{H_0/2} = \sqrt{2}^{H_0}$ and $(1+\kappa^2)X_s^2/\sqrt{\lambda} = \sqrt{(1+\kappa^2)}X_s$. Therefore, for the second term, we have the dimension dependency of $ \sqrt{n(n+d)} \times \sqrt{n} \times \sqrt{n+d} $ which gives the advertised bound.

\end{proof}

\subsection{Combining Terms for Final Regret Upper Bound}
% \textbf{Proof of Theorem \ref{reg:exp_c}:}
% Recall that 
% \begin{equation*}
% \text{REGRET}(T)  \leq \sigma_\nu^2 \Tw D \|B_* \|_F^2 +  \sum_{t=0}^{T_{c}} \left(2 \nu_t^\top R u_t + \nu_t^\top R \nu_t\right) + R_1^\zeta - R_2^\zeta - R_3^\zeta.
% \end{equation*}
% Combining Lemma \ref{lem:exp_effect} for $\sum_{t=0}^{T_{c}} \left(2 \nu_t^\top R u_t + \nu_t^\top R \nu_t\right) $, Lemma \ref{R1_zeta_control} for $R_1^\zeta$, Lemma \ref{R2_zeta_control} for $|R_2^\zeta|$ and Lemma \ref{R3_zeta_control} for $|R_3^\zeta|$, we get the advertised regret bound. \\
% \null\hfill$\blacksquare$ 

\noindent\textbf{Proof of Theorem \ref{reg:exp_s}:}
Recall that 
\begin{equation*}
\text{REGRET}(T)  \leq \sigma_\nu^2 \Tw D \|B_* \|_F^2 +  \sum_{t=0}^{\Tw} \left(2 \nu_t^\top R u_t + \nu_t^\top R \nu_t\right) + R_1^\zeta - R_2^\zeta - R_3^\zeta.
\end{equation*}
Combining Lemma \ref{lem:exp_effect} for $\sum_{t=0}^{\Tw} \left(2 \nu_t^\top R u_t + \nu_t^\top R \nu_t\right) $, Lemma \ref{R1_zeta_stabil} for $R_1^\zeta$, Lemma \ref{R2_zeta_stabil} for $|R_2^\zeta|$ and Lemma \ref{R3_zeta_stabil} for $|R_3^\zeta|$, we get the advertised regret bound.
\null\hfill$\blacksquare$

\section{CONTROLLABILITY ASSUMPTION IN ABBASI-YADKORI AND SZEPESVARI (2011)} \label{apx:comparecontrol}
In \citet{abbasi2011lqr}, the authors derive their results for the following setting:
\begin{assumption} [Controllable Linear Dynamical System] \label{apx:parameterassump_controllability}
The unknown parameter $\tts$ is a member of a set $\mathcal{S}_c$ such that
\begin{equation*}
    \mathcal{S}_c \subseteq \left\{ \Theta'=[A', B'] \in \mathbb{R}^{n \times(n+d)} ~\big|~  \Theta' \text{ is controllable,}~ \|A'+B'K(\Theta')\| \leq \Upsilon < 1,~ \|\Theta'\|_F \leq S \right\}
\end{equation*}
Following the controllability and the boundedness of $\mathcal{S}_c$, we have finite numbers $D$ and $\kappa\geq 1$ s.t., $ \sup \{\|P(\Theta')\| ~|~ \Theta' \in \mathcal{S}_c \}\leq D$ and $\sup \{\|K(\Theta')\| ~|~ \Theta' \in \mathcal{S}_c \}\leq \kappa $.
\end{assumption}
Our results are strict generalizations of these since stabilizable systems subsume controllable systems and all closed-loop contractible systems considered with Assumption \ref{apx:parameterassump_controllability} is a subset of general stable closed-loop systems considered in this work. For the setting in Assumption \ref{apx:parameterassump_controllability}, we can bound the state following similar steps to stabilizable case but since the closed-loop system is contractible we do not need minimum length on epoch of an optimistic controller since the state would always shrink. Adopting the proofs provided in this work to Assumption \ref{apx:parameterassump_controllability}, one can obtain the similar polynomial dimension dependency via additional exploration of \alg. This shows that with additional exploration the result of \citet{abbasi2011lqr} could be directly improved.

\section{TECHNICAL THEOREMS AND LEMMAS}
\label{apx:technical}
\begin{theorem}[Self-normalized bound for vector-valued martingales~\citep{abbasi2011improved}]
\label{selfnormalized}
Let $\left(\mathcal{F}_{t} ; k \geq\right.$
$0)$ be a filtration, $\left(m_{k} ; k \geq 0\right)$ be an $\mathbb{R}^{d}$-valued stochastic process adapted to $\left(\mathcal{F}_{k}\right),\left(\eta_{k} ; k \geq 1\right)$
be a real-valued martingale difference process adapted to $\left(\mathcal{F}_{k}\right) .$ Assume that $\eta_{k}$ is conditionally sub-Gaussian with constant $R$. Consider the martingale
\begin{equation*}
S_{t}=\sum\nolimits_{k=1}^{t} \eta_{k} m_{k-1}
\end{equation*}
and the matrix-valued processes
\begin{equation*}
V_{t}=\sum\nolimits_{k=1}^{t} m_{k-1} m_{k-1}^{\top}, \quad \overline{V}_{t}=V+V_{t}, \quad t \geq 0
\end{equation*}
Then for any $0<\delta<1$, with probability $1-\delta$
\begin{equation*}
\forall t \geq 0, \quad\left\|S_{t}\right\|^2_{\overline{V}_{t}^{-1}} \leq 2 R^{2} \log \left(\frac{\operatorname{det}\left(\overline{V}_{t}\right)^{1 / 2} \operatorname{det}(V)^{-1 / 2}}{\delta}\right)
\end{equation*}

\end{theorem}

\begin{theorem}[Azuma's inequality]  \label{Azuma}
Assume that $\left(X_{s} ; s \geq 0\right)$ is a supermartingale and $\left|X_{s}-X_{s-1}\right| \leq c_{s}$ almost surely. Then for all $t>0$ and all $\epsilon>0$,
\begin{equation*}
P\left(\left|X_{t}-X_{0}\right| \geq \epsilon\right) \leq 2 \exp \left(\frac{-\epsilon^{2}}{2 \sum_{s=1}^{t} c_{s}^{2}}\right)
\end{equation*}
\end{theorem}

\begin{lemma}[Bound on Logarithm of the Determinant of Sample Covariance Matrix~\citep{abbasi2011improved}]\label{upperboundlemma}
The following holds for any $t \geq 1$ :
\begin{equation*}
\sum_{k=0}^{t-1}\left(\left\|z_{k}\right\|_{V_{k}^{-1}}^{2} \wedge 1\right) \leq 2 \log \frac{\operatorname{det}\left(V_{t}\right)}{\operatorname{det}(\lambda I)}
\end{equation*}
Further, when the covariates satisfy $\left\|z_{t}\right\| \leq c_{m}, t \geq 0$ with some $c_{m}>0$ w.p. 1 then
\begin{equation*}
\log \frac{\operatorname{det}\left(V_{t}\right)}{\operatorname{det}(\lambda I)} \leq(n+d) \log \left(\frac{\lambda(n+d)+t c_{m}^{2}}{\lambda(n+d)}\right)
\end{equation*}
\end{lemma}

\begin{lemma}[Norm of Subgaussian vector]\label{subgauss lemma}
Let $v\in \mathbb{R}^d$ be a entry-wise $R$-subgaussian random variable. Then with probability $1-\delta$, $\|v\| \leq R\sqrt{2d\log(d/\delta)}$.
\end{lemma}

\section{IMPLEMENTATION DETAILS OF EXPERIMENTS AND ADDITIONAL RESULTS} \label{apx:experiment_detail}

In this section, we provide the simulation setups, with the parameter settings for each algorithm and the details of the adaptive control tasks. The implementations and further system considerations are available at \href{https://github.com/SahinLale/StabL}{https://github.com/SahinLale/StabL}. In the experiments, we use four adaptive control tasks: 
\begin{enumerate}[label=\textbf{(\arabic*)}]
    \item A marginally unstable Laplacian system~\citep{dean2018regret} 
    \item The longitudinal flight control of Boeing 747 with linearized dynamics~\citep{747model} 
    \item Unmanned aerial vehicle (UAV) that operates in a 2-D plane~\citep{zhao2021infinite} 
    \item A stabilizable but not controllable linear dynamical system.
\end{enumerate}

For each setting we deploy 4 different algorithms: 
\begin{enumerate}[label=\textbf{(\roman*)}]
    \item Our algorithm \alg, 
    \item \OFULQ of \citet{abbasi2011lqr}, 
    \item Certainty equivalent controller (CEC) with fixed isotropic perturbations, 
    \item CEC with decaying isotropic perturbations.
\end{enumerate}

For each algorithm there are different varying parameters. For each adaptive control task, we tune each parameter in terms of regret performance and present the performance of the best performing parameter choices since the regret analysis for each algorithm considers the worst case scenario. In each setting, we will specify these parameters choices for each algorithm. We use the actual errors $\|\tth_t-\tts\|_2$ rather than bounds or bootstrap estimates for each algorithm, since we observe that the overall effect is negligible as mentioned in \citet{dean2018regret}. The following gives the implementation details of each algorithm.

\paragraph{(i) \alg} We have $\sigma_\nu$, $H_0$ and $\Tw$ as the varying parameters. In the implementation of optimistic parameter search we deploy projected gradient descent (PGD), which works efficiently for the small dimensional problems. The implementation follows Section G.1 of \citep{dean2018regret}. Note that this approach, hence the optimistic parameter choice, can be computationally challenging for higher dimensional systems. We pick the regularizer $\lambda = 0.05$ for all adaptive control tasks.

\paragraph{(ii) \OFULQ} We deploy a slight modification on the implementation of \OFULQ given by \citep{abbasi2011lqr}. Similar to \alg, we add an additional minimum policy duration constraint to the general switching constraints of \OFULQ, \textit{i.e.}, the standard determinant doubling of $V_t$. This prevents too frequent changes in the beginning of the algorithm and dramatically improves the regret performance. This minimum duration $H_0^{OFU}$ is the only varying parameter for \OFULQ. For the optimistic parameter search we also implement PGD. We pick the regularizer $\lambda = 0.001$ for all adaptive control tasks.

\paragraph{(iii) CEC w/t fixed perturbations} This algorithm is the standard baseline in control theory. In the implementation, the optimal infinite-horizon \LQR controller for the estimated system is deployed and fixed isotropic perturbations $\mathcal{N}(0,\sigma_{\exp}^2I)$ are injected throughout the implementation. The isotropic perturbations are injected since it is well-known that certainty equivalent controllers can result with drastically incorrect parameter
estimates \citep{lai1982least,becker1985adaptive,kumar1990convergence} due to lack of exploration. The policy changes happen in epochs with linear scaling, \textit{i.e.}, each epoch $i$ is of $iH_{ep}$ length. This growth is observed to be preferable over the standard exponentially increasing epoch lengths adopted in theoretical analyses of the worst case regret guarantees. Thus, the varying parameters for CEC w/t fixed perturbations are $\sigma_{exp}$ and $H_{ep}$. We pick the regularizer $\lambda = 0.5$ for all adaptive control tasks.

\paragraph{(iv) CEC w/t decaying perturbations} The implementation of this algorithm is similar to \textbf{(iii)}. The difference is that the injected perturbations have decaying variance over epochs. We adopt the decay of $1/\sqrt{i}$ for each epoch $i$, \textit{i.e.} $\sigma_{i,exp} =\sigma_{exp}^{dec} / \sqrt{i} $ for some initial $\sigma_{exp}^{dec} $ such that isotropic perturbations are injected in each epoch. Based on the extensive experimental study, we deduced that this decay performs better than the decay of $i^{-1/3}$ as given in \citet{dean2018regret} or $2^{-i/2}$ as given in \citet{simchowitz2020naive}.
The varying parameters for this algorithm are $\sigma_{exp}^{dec}$ and $H_{ep}^{dec}$ in which the latter defines the first epoch length in the linear scaling of epochs. We pick the regularizer $\lambda = 0.05$ for all adaptive control tasks.

% \paragraph{(v) TS} The implementation of this algorithm follows from \citet{abeille2018improved}. Even though the worst case theoretical regret guarantees hold only for scalar setting, in practice TS provides a good alternative to \OFULQ. In the TS, a parameter is sampled from the set of plausible parameters uniformly, instead of selecting tge optimistic parameter as in \OFULQ. Similarly to \alg and \OFULQ, TS implementation also has minimum policy duration parameter $H_0^{TS}$ besides the determinant doubling condition for the policy changes. 

In each experiment, the system starts from $x_0 = 0$ to reduce variance over runs. For each setting, we run $200$ independent runs with the duration of $200$ time steps. Note that we do not compare \alg with the adaptive control algorithms provided in \citep{simchowitz2020naive,simchowitz2020improper,dean2018regret,cohen2019learning,faradonbeh2018input,faradonbeh2020adaptive} which all require a given initial stabilizing policy or stable open-loop dynamics and \citep{abeille2017thompson} which is tailored for scalar systems. Moreover, \citet{chen2020black} deals with adversarial \LQR setting and uses ``significantly'' large inputs to identify the model dynamics which causes orders of magnitude worse regret.

\subsection{Marginally Unstable Laplacian System} \label{apx:I1}
The \LQR problem is given as 
\begin{equation}
    A_* = \left[\begin{array}{ccc}
1.01 & 0.01 & 0 \\
0.01 & 1.01 & 0.01 \\
0 & 0.01 & 1.01
\end{array}\right], \enskip B_* = I_{3\times3}, \enskip Q = 10 I, \enskip R = I, \enskip w \sim \mathcal{N}(0,I).
\end{equation}
This system dynamics have been studied in \citep{dean2018regret,dean2019sample,abbasi2018regret,tu2018least} and it corresponds a Laplacian system with weakly connected adjacent nodes. Notice that the inputs have less cost weight than the states. This skewed cost combined with the unstable dynamics severely hinders the design of effective strategies for \OFU-based methods.

\paragraph{Algorithmic Setups:} For \alg, we set $H_0 = 15$, $\Tw = 35$ and $\sigma_\nu = 1.5$. For CEC with decaying perturbation, we set $H_{ep}^{dec} = 20$, and $\sigma_{exp}^{dec} = 2$. For CEC with fixed perturbation, we set $\sigma_{exp} = 1.3$ and $H_{ep} = 15$. For \OFULQ, we set $H_0^{OFU} = 6$. 

\paragraph{Regret After 200 Time Steps:} 
In Table \ref{table:2_apx}, we provide the regret performance of the algorithms after 200 time steps of adaptive control in the Laplacian system. As expected the regret performance of \OFULQ suffers the most regret due to unstable dynamics and skewed cost, which makes it difficult to design effective policies for the \OFU-based algorithms. Even though \alg uses \OFU principle, it overcomes the difficulty to design effective policies via the improved exploration in the early stages and achieves the best regret performance.

\begin{table*}[h]
% \hspace{-2.2em}
\centering
\captionsetup{justification=centering}
\caption{Regret After 200 Time Steps in Marginally Unstable Laplacian System}
\label{table:2_apx}
 \begin{tabular}{c c c c c c}
 \textbf{Algorithm} &  \textbf{Average Regret} &  \textbf{Best 95\% } &   \textbf{Best 90\%} &   \textbf{Best 75\%} &   \textbf{Best 50\%}   \\
  \hline
 \alg & $\bf{1.55 \times 10^4}$   & $\bf{1.42 \times 10^4}$  & $\bf{1.32 \times 10^4}$ & $\bf{1.12 \times 10^4 }$ & $\bf{8.89  \times 10^3} $ \\

OFULQ  & $6.17 \times 10^{10}$ & $4.57 \times 10^7$  &  $4.01 \times 10^6$ & $3.49 \times 10^5$ & $4.70  \times 10^4 $ \\

CEC w/t Fixed & $3.72 \times 10^{10}$ & $2.23 \times 10^{5}$  & $2.14 \times 10^{4}$  & $1.95 \times 10^{4}$  &   $1.73 \times 10^{4}$  \\

  CEC w/t Decay & $4.63 \times 10^{4}$  &  $4.27 \times 10^{4}$ & $4.03 \times 10^{4}$ & $3.51 \times 10^4$ & $2.84 \times 10^{4}$ 
\end{tabular}
% \vspace{-1em}
\end{table*}

Figure \ref{fig:regret_laplace} gives the regret comparison between \alg and CEC with decaying isotropic perturbations which performs the second best in the given Laplacian system. Note that we did not include \OFULQ and CEC w/t fixed perturbations  in the figure since they perform orders of magnitude worse that \alg and CEC w/t decaying perturbations. 

\begin{figure}[h]
\centering
  \includegraphics[width=0.6\linewidth]{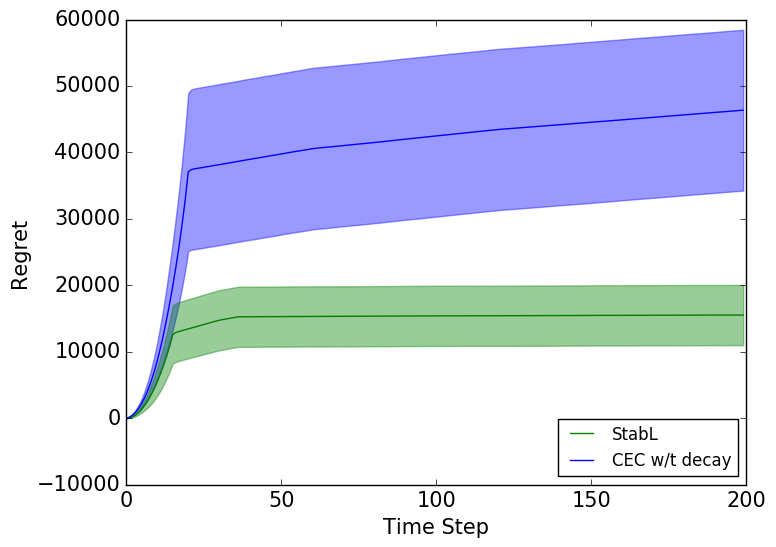}
  \caption{Regret of \alg vs CEC with decaying isotropic perturbations. The solid lines are the average regrets for $200$ independent runs and the shaded regions are the half standard deviations.}
  \label{fig:regret_laplace}
\end{figure}

\paragraph{Maximum State Norm:}
In Table \ref{table:max_laplace}, we display the stabilization capabilities of the algorithms by providing the averages of the maximum $\ell_2$ norms of the states in $200$ independent runs. We also include the worst case state magnitudes which demonstrates how controlled the states are during the entire adaptive control task. The results show that \alg maintains the smallest magnitude of the state and thus, the most stable dynamics. We also verify that after the first policy change which happens after $15$ time steps, the spectral radius of the closed-loop system formed via \alg is always stable, \textit{i.e.} $\rho(A_* + B_* K(\ttt_t)) < 1$ for $t>15$. 

\begin{table}[h]
% \hspace{-2.2em}
\centering
\captionsetup{justification=centering}
\caption{Maximum State Norm in Marginally Unstable Laplacian System}
% \label{table:4}
\begin{tabular}{c c c c c}
 \textbf{Algorithm} &  \textbf{Average $\max\|x\|_2$} &  \textbf{Worst 5\% } &   \textbf{Worst 10\%} &   \textbf{Worst 25\%} \\
  \hline
 \alg &  $\bf{1.35 \times 10^1}$   & $\bf{2.24 \times 10^1}$  & $\bf{2.15 \times 10^1}$ & $\bf{1.95 \times 10^1}$  \\

 OFULQ  & $9.59 \times 10^{3}$ & $1.83 \times 10^5$  &  $9.04 \times 10^4$ & $3.81 \times 10^4$  \\

CEC w/t Fixed & $3.33 \times 10^{3}$ & $6.64 \times 10^{4}$   & $3.32 \times 10^{4}$  &   $1.33 \times 10^{4}$  \\

  CEC w/t Decay & $2.04 \times 10^{1}$  & $3.46 \times 10^{1}$ & $3.27 \times 10^1$ & $2.87 \times 10^{1}$ 
\end{tabular}
\label{table:max_laplace}
\end{table}

\paragraph{Persistence of Excitation via \alg:}
In order to further highlight the benefit of improved exploration strategy, we empirically study the smallest eigenvalue of the regularized design matrix $V_t$ for \alg and \OFULQ. The evolution of the $\lambda_{\min}(V_t)$ is shown for both algorithms in Figure \ref{fig:compare}. From the figure, one can see that improved exploration strategy of \alg achieves linear scaling of $\lambda_{\min}(V_t)$, \textit{i.e.}, persistence of excitation. Thus, it finds the stabilizing neighborhood after the first epoch. On the other hand, the control inputs of \OFULQ fail to excite the system uniformly, thus it cannot quickly find a stabilizing policy. This results in unstable dynamics and significantly more regret on average (Table \ref{table:2_apx}). 

\begin{figure}[h]
\centering
  \includegraphics[width=0.6\linewidth]{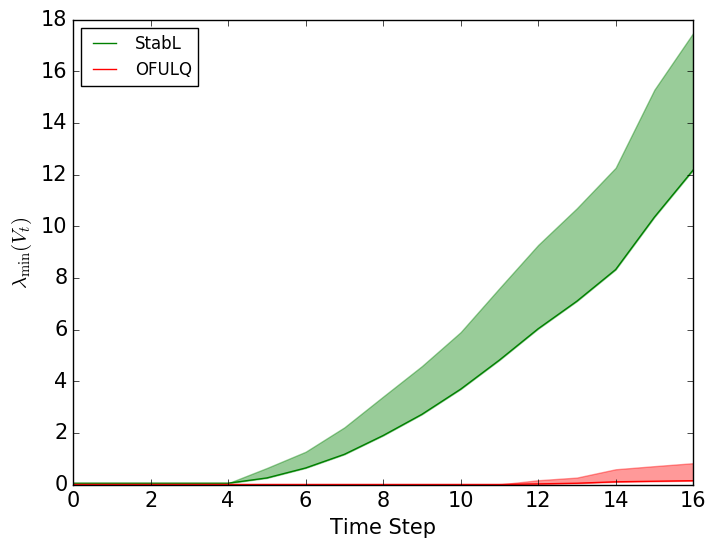}
  \caption{Scaling of the smallest eigenvalue of the design matrix for \alg and \OFULQ. The solid line denotes the mean and the shaded region denotes one standard deviation. The early improved exploration strategy helps \alg achieve linear scaling in $\lambda_{\min}(V_t)$, thus persistence of excitation. The only \OFU-based controllers of \OFULQ fail to achieve persistence of excitation.}
  \label{fig:compare}
\end{figure}

\subsection{Longitudinal Flight Control of Boeing 747} \label{apx:I2}

The \LQR problem is given as
\begin{equation}
    A_* = \begin{bmatrix}
    0.99 & 0.03 & -0.02 & -0.32 \\
    0.01 & 0.47 & 4.7 & 0 \\
    0.02 & -0.06 & 0.4 & 0 \\
    0.01 & -0.04 & 0.72 & 0.99 
\end{bmatrix}, ~ 
B_* = \begin{bmatrix}
  0.01 & 0.99 \\
  -3.44 & 1.66 \\
  -0.83 & 0.44 \\
  -0.47 & 0.25
\end{bmatrix}, ~ Q = I, ~ R = I, ~ w \sim \mathcal{N}(0,I).
\end{equation}

This problem is the longitudinal flight control of Boeing 747 with linearized dynamics and introduced in \citep{747model}. The given linear dynamical system corresponds to the dynamics for level flight of Boeing 747 at the altitude of 40000ft with the speed of 774ft/sec, for a discretization of 1 second. The first state element is the velocity of aircraft along body axis, the second is the velocity of aircraft perpendicular to body axis, the third is the angle between body axis and horizontal and the final element is the angular velocity of aircraft. The first input element is the elevator angle and the second one is the thrust. The process noise corresponds to the external wind conditions.  

Notice that the dynamics are linearized around a certain point and it is important to guarantee that the linearization is valid. To this end, an \RL policy should stabilize the system and keep the state small in order to not lead the system to the unmodeled nonlinear dynamics. 

\paragraph{Algorithmic Setups:} For \alg, we set $H_0 = 10$, $\Tw = 35$ and $\sigma_\nu = 2$. For CEC with decaying perturbation, we set $H_{ep}^{dec} = 30$, and $\sigma_{exp}^{dec} = 2$. For CEC with fixed perturbation, we set $\sigma_{exp} = 2.5$ and $H_{ep} = 25$. For \OFULQ, we set $H_0^{OFU} = 7$. 

\paragraph{Regret After 200 Time Steps:} In Table \ref{table:regret_boeing}, we give the regret performance of the algorithms after 200 time steps in Boeing 747 flight control. In terms of average regret, \alg attains half of the regret of CEC with decay and performs orders of magnitude better than \OFULQ. Also, consider Figure \ref{fig:regret_boeing}. Notice that until the third policy update, \OFULQ is still working towards further exploration and is not designing effective controllers to regulate the system dynamics. This is due to the higher dimensions of the Boeing 747 control system which prevents quick and effective exploration. This results in unstable system dynamics in the early stages and poorly scaling of the regret. On the other hand, the early improved exploration strategy helps \alg to maintain stable dynamics with the expense of an additional regret in the early stages compared to \OFULQ. However, as it can be seen from Figure \ref{fig:regret_boeing}, this improved exploration strategy yields significantly lower regret in the later stages. 

\begin{table}[h]
% \hspace{-2.2em}
\centering
\captionsetup{justification=centering}
\caption{Regret After 200 Time Steps in Boeing 747 Flight Control}
\label{table:regret_boeing}
  \begin{tabular}{c c c c c c}
 \textbf{Algorithm} &  \textbf{Average Regret} &  \textbf{Top 95\% } &   \textbf{Top 90\%} &   \textbf{Top 75\%} &   \textbf{Top 50\%}   \\
  \hline
 \alg & $\bf{1.34 \times 10^4}$   & $\bf{1.05 \times 10^3}$  & $\bf{9.60 \times 10^3}$ & $\bf{7.58 \times 10^3 }$ & $\bf{5.28  \times 10^3} $ \\

 OFULQ  & $1.47 \times 10^{8}$ & $4.19 \times 10^6$  &  $9.89 \times 10^5$ & $5.60 \times 10^4$ & $8.91  \times 10^3 $ \\

CEC w/t Fixed  & $4.79 \times 10^{4}$ & $4.62 \times 10^{4}$  & $4.51 \times 10^{4}$  & $4.25 \times 10^{4}$  &   $3.88 \times 10^{4}$  \\

  CEC w/t Decay & $2.93 \times 10^{4}$  &  $2.61 \times 10^{4}$ & $2.48 \times 10^{4}$ & $2.22 \times 10^4$ & $1.86 \times 10^{4}$ 
\end{tabular}
% \vspace{-1em}
\end{table}

\begin{figure}[h]
\centering
  \includegraphics[width=0.6\linewidth]{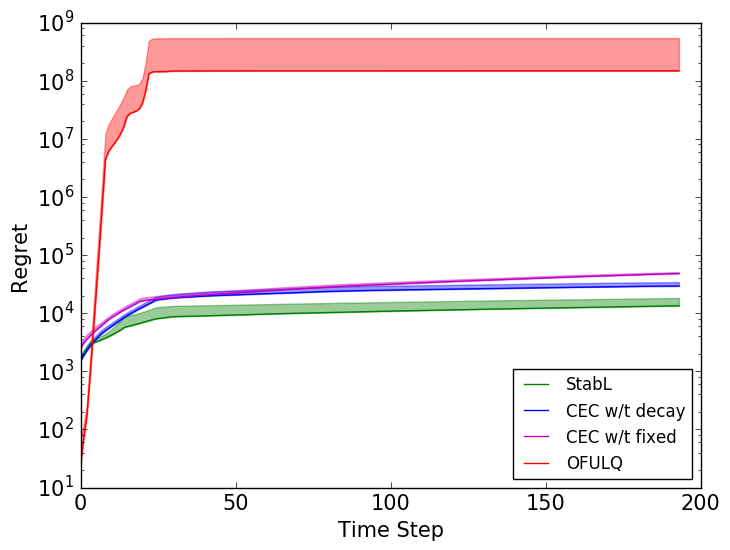}
  \caption{Regret Comparison of all algorithms in Boeing 747 flight control. The solid lines are the average regrets for $200$ independent runs and the shaded regions are the quarter standard deviations.}
  \label{fig:regret_boeing}
\end{figure}

\textbf{Maximum State Norm:}
Similar to Laplacian system, \alg controls the state well and provides the lowest average maximum norm.  

\begin{table}[h]
% \hspace{-2.2em}
\centering
\captionsetup{justification=centering}
\caption{Maximum State Norm in Boeing 747 Control}
% \label{table:4}
\begin{tabular}{c c c c c}
 \textbf{Algorithm} &  \textbf{Average $\max\|x\|_2$} &  \textbf{Worst 5\% } &   \textbf{Worst 10\%} &   \textbf{Worst 25\%} \\
  \hline
 \alg &  $\bf{3.38 \times 10^1}$   & $\bf{8.02 \times 10^1}$  & $\bf{7.01 \times 10^1}$ & $\bf{5.23 \times 10^1 }$  \\

 OFULQ  & $1.62 \times 10^{3}$ & $2.25 \times 10^4$  &  $1.37 \times 10^4$ & $6.26 \times 10^3$  \\

CEC w/t Fixed & $4.97 \times 10^{1}$ & $7.78 \times 10^{1}$   & $7.31 \times 10^{1}$  &   $6.48 \times 10^{1}$  \\

  CEC w/t Decay & $4.60 \times 10^{1}$  & $7.96 \times 10^{1}$ & $7.25 \times 10^1$ & $6.31 \times 10^{1}$ 
\end{tabular}
\label{table:max_state_boeing}
\end{table}

\subsection{Unmanned Aerial Vehicle (UAV) in 2-D plane} \label{apx:I3}
The \LQR problem is given as 
\begin{equation}
    A_* = \left[\begin{array}{cccc}
1 & 0.5 & 0 & 0 \\
0 & 1 & 0 & 0 \\
0 & 0 & 1 & 0.5 \\
0 & 0 & 0 & 1
\end{array}\right], \enskip B_* = \left[\begin{array}{cc}
0.125 & 0  \\
0.5 & 0  \\
0 & 0.125  \\
0 & 0.5 
\end{array}\right], Q = diag(1,0.1,2,0.2), \enskip R = I, \enskip w \sim \mathcal{N}(0,I)
\end{equation}
This problem is the linearized model of a UAV which operates in a 2-D plane \citep{zhao2021infinite}. Notice that it corresponds to the model of double integrator. The first and third state elements correspond to the position, whereas the second and fourth state elements are velocity components. The inputs are the acceleration. The process noise corresponds to the external wind conditions. Similar to Boeing 747, the dynamics are linearized and keeping the state vector small is critical in order to maintain the validity of the linearization.

\paragraph{Algorithmic Setups:} For \alg, we set $H_0 = 20$, $\Tw = 55$ and $\sigma_\nu = 4$. For CEC with decaying perturbation, we set $H_{ep}^{dec} = 30$, and $\sigma_{exp}^{dec} = 3.5$. For CEC with fixed perturbation, we set $\sigma_{exp} = 3$ and $H_{ep} = 35$. For \OFULQ, we set $H_0^{OFU} = 7$. 

\paragraph{Regret After 200 Time Steps:}

\begin{table}[h]
% \hspace{-2.2em}
\centering
\captionsetup{justification=centering}
\caption{Regret After 200 Time Steps in UAV Control}
\label{table:regret_uav}
  \begin{tabular}{c c c c c c}
 \textbf{Algorithm} &  \textbf{Average Regret} &  \textbf{Top 95\% } &   \textbf{Top 90\%} &   \textbf{Top 75\%} &   \textbf{Top 50\%}   \\
  \hline
\alg & $\bf{1.53 \times 10^5}$   & $\bf{1.05 \times 10^5}$  & $\bf{9.23 \times 10^4}$ & $\bf{6.85 \times 10^4}$ & $\bf{4.47  \times 10^4} $ \\

OFULQ  & $5.06 \times 10^{7}$ & $1.75 \times 10^6$  &  $1.03 \times 10^6$ & $2.46 \times 10^5$ & $5.82  \times 10^4 $ \\

CEC w/t Fixed  & $4.52 \times 10^{5}$ & $3.80 \times 10^{5}$  & $3.35 \times 10^{5}$  & $2.50 \times 10^{5}$  &   $1.64 \times 10^{5}$  \\

  CEC w/t Decay & $3.24 \times 10^{5}$  &  $2.70 \times 10^{5}$ & $2.37 \times 10^{5}$ & $1.75 \times 10^5$ & $1.03 \times 10^{5}$ 
\end{tabular}
% \vspace{-1em}
\end{table}

In Table \ref{table:regret_uav}, we give the regret performance of the algorithms after 200 time steps in UAV control control task. Once more, \alg performs significantly better than other \RL methods. The evolution of the average regret is also given in Figure \ref{fig:regret_uav}. As suggested by the theory, by paying a linear regret cost for a short period of time in the early stages, \alg guarantees stabilizing the underlying system and achieves the best regret performance. 

\vspace{4em}
\textbf{Maximum State Norm:}

\begin{table}[h]
% \hspace{-2.2em}
\centering
\captionsetup{justification=centering}
\caption{Maximum State Norm in UAV Control}
% \label{table:4}
\begin{tabular}{c c c c c}
 \textbf{Algorithm} &  \textbf{Average $\max\|x\|_2$} &  \textbf{Worst 5\% } &   \textbf{Worst 10\%} &   \textbf{Worst 25\%} \\
  \hline
\alg &  $\bf{8.46 \times 10^1}$   & $\bf{2.51 \times 10^2}$  & $\bf{2.00 \times 10^2}$ & $\bf{1.50 \times 10^2 }$  \\

OFULQ  & $5.61 \times 10^{2}$ & $6.35 \times 10^3$  &  $3.78 \times 10^3$ & $1.90 \times 10^3$  \\

CEC w/t Fixed & $1.45 \times 10^{2}$ & $3.12 \times 10^{2}$   & $2.91 \times 10^{2}$  &   $2.42 \times 10^{2}$  \\

  CEC w/t Decay & $1.26 \times 10^{2}$  & $2.71 \times 10^{2}$ & $2.48 \times 10^2$ & $2.12 \times 10^{2}$ 
\end{tabular}
\label{table:max_state_uav}
\end{table}

\begin{figure}[h]
\centering
  \includegraphics[width=0.6\linewidth]{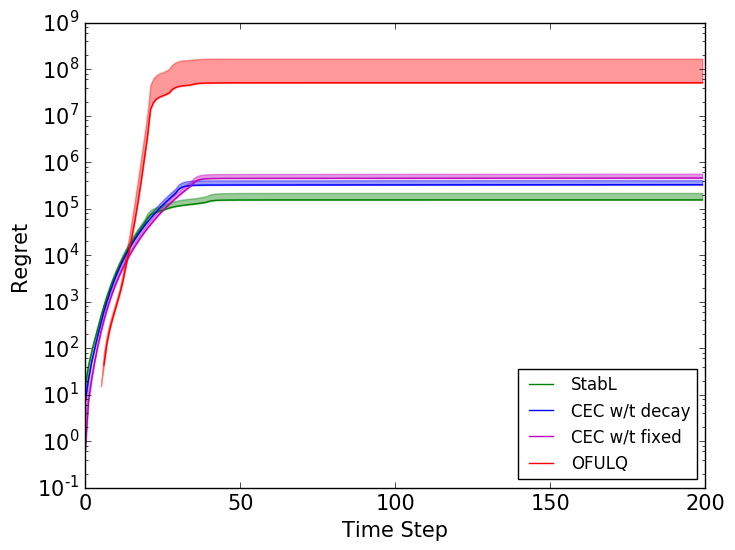}
  \caption{Regret Comparison of all algorithms in UAV control task. The solid lines are the average regrets for $200$ independent runs and the shaded regions are the quarter standard deviations.}
  \label{fig:regret_uav}
\end{figure}

\subsection{Stabilizable but Not Controllable System} \label{apx:I4}
The \LQR problem is given as 
\begin{equation} \label{sys_stabil}
    A_* = \left[\begin{array}{ccc}
-2 & 0 & 1.1 \\
1.5 & 0.9 & 1.3 \\
0 & 0 & 0.5
\end{array}\right], \enskip B_* = \left[\begin{array}{cc}
1 & 0  \\
0 & 1  \\
0 & 0  
\end{array}\right], Q = I, \enskip R = I, \enskip w \sim \mathcal{N}(0,I)
\end{equation}
This problem is particularly challenging in terms of system identification and controller design since the system is not controllable but stabilizable. As expected besides \alg which is tailored for the general stabilizable setting, other algorithms perform poorly. In fact, CEC with fixed noise significantly blows up due to significantly unstable dynamics for the controllable part of the system. Therefore, we only present the remaining three algorithms. 

\paragraph{Algorithmic Setups:} For \alg, we set $H_0 = 8$, $\Tw = 20$ and $\sigma_\nu = 2.5$. For CEC with decaying perturbation, we set $H_{ep}^{dec} = 30$, and $\sigma_{exp}^{dec} = 3$. For \OFULQ, we set $H_0^{OFU} = 6$. 

\paragraph{Regret After 200 Time Steps:}
Table \ref{table:regret_stabil} provides the regret of the algorithms after 200 time steps. This setting is where \OFULQ fails dramatically due to not being tailored for the stabilizable systems. Compared to CEC with decaying perturbation, \alg also provides an order of magnitude improvement (Figure \ref{fig:regret_stabil})

\begin{table}[h]
% \hspace{-2.2em}
\centering
\captionsetup{justification=centering}
\caption{Regret After 200 Time Steps in Stabilizable but Not Controllable System \eqref{sys_stabil}}
\label{table:regret_stabil}
  \begin{tabular}{c c c c c c}
 \textbf{Algorithm} &  \textbf{Average Regret} &  \textbf{Top 95\% } &   \textbf{Top 90\%} &   \textbf{Top 75\%} &   \textbf{Top 50\%}   \\
  \hline
\alg & $\bf{1.68 \times 10^6}$   & $\bf{9.56 \times 10^5}$  & $\bf{7.21 \times 10^5}$ & $\bf{3.72 \times 10^5}$ & $\bf{1.29  \times 10^5} $ \\

 OFULQ  & $5.20 \times 10^{12}$ & $1.74 \times 10^{12}$  &  $8.27 \times 10^{11}$ & $2.13 \times 10^{11}$ & $4.51  \times 10^{10} $ \\

 CEC w/t Decay & $1.56 \times 10^{7}$  &  $1.17 \times 10^{7}$ & $9.75 \times 10^{6}$ & $5.96 \times 10^6$ & $2.33 \times 10^{6}$ 
\end{tabular}
% \vspace{-1em}
\end{table}

\begin{figure}[h]
\centering
  \includegraphics[width=0.6\linewidth]{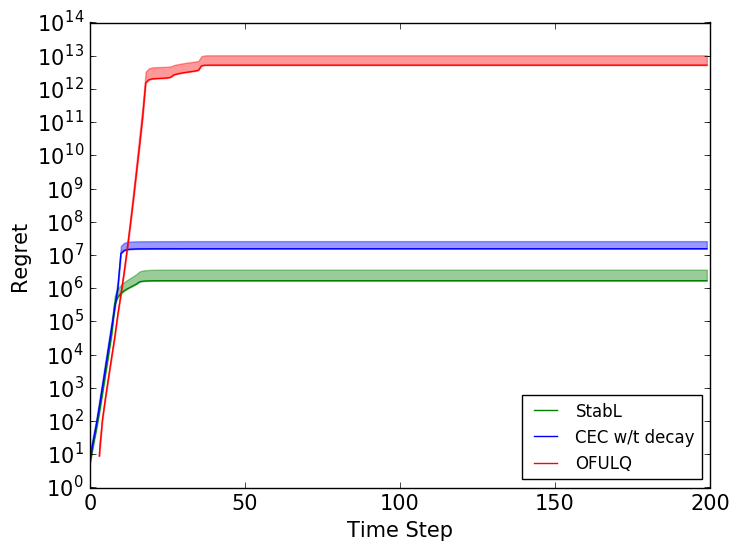}
  \caption{Regret Comparison of three algorithms in controlling \eqref{sys_stabil}. The solid lines are the average regrets for $200$ independent runs and the shaded regions are the quarter standard deviations.}
  \label{fig:regret_stabil}
\end{figure}

\paragraph{Maximum State Norm:}

\vspace{4em}

\begin{table}[b]
% \hspace{-2.2em}
\centering
\captionsetup{justification=centering}
\caption{Maximum State Norm in the Control of Stabilizable but Not Controllable System \eqref{sys_stabil}}
% \label{table:4}
\begin{tabular}{c c c c c}
 \textbf{Algorithm} &  \textbf{Average $\max\|x\|_2$} &  \textbf{Worst 5\% } &   \textbf{Worst 10\%} &   \textbf{Worst 25\%} \\
  \hline
\alg &  $\bf{3.02 \times 10^2}$   & $\bf{1.04 \times 10^3}$  & $\bf{8.88 \times 10^2}$ & $\bf{6.68 \times 10^2 }$  \\

OFULQ  & $4.39 \times 10^{5}$ & $3.10 \times 10^6$  &  $2.40 \times 10^6$ & $1.39 \times 10^6$  \\

  CEC w/t Decay & $1.37 \times 10^{3}$  & $4.07 \times 10^{3}$ & $3.54 \times 10^3$ & $2.78 \times 10^{3}$ 
\end{tabular}
\label{table:max_state_stabil}
\end{table}

\end{document}